%% file: main.tex
\documentclass{article} %
\usepackage{iclr2026_conference,times}

\input{math_commands.tex}
\input{tables}

\definecolor{mydarkblue}{rgb}{0,0.08,0.45} 

\usepackage[utf8]{inputenc} %
\usepackage[T1]{fontenc}    %
\usepackage{url}            %
\usepackage{booktabs}       %
\usepackage{amsfonts}       %
\usepackage{nicefrac}       %
\usepackage{microtype}      %

\usepackage{url}
\usepackage[colorlinks=true, allcolors=mydarkblue]{hyperref}
\usepackage[nameinlink, capitalize]{cleveref}
\usepackage{comment}
\usepackage[colorinlistoftodos]{todonotes}
\usepackage{multirow}
\usepackage{amsthm}
\usepackage{amsmath}
\usepackage{amssymb}
\usepackage{tikz,pgfplots}
\DeclareUnicodeCharacter{2212}{−}
\usepgfplotslibrary{groupplots,dateplot}
\usetikzlibrary{patterns,shapes.arrows}
\pgfplotsset{compat=newest}
\usepackage{wrapfig}
\usepackage{algorithm}
\usepackage{algpseudocode}
\usepackage{subcaption}
\usepackage{graphicx}
\usepackage{makecell}

\newtheorem{prop}{Proposition}

\usepackage{xspace}
\newcommand{\ie}{\textit{i.e.}\@\xspace}
\newcommand{\eg}{\textit{e.g.}\@\xspace}

\newlength\figureheight
\newlength\figurewidth

\renewcommand{\paragraph}[1]{\textbf{#1}~~}

\title{PriorGuide: Test-Time Prior Adaptation for \\ Simulation-Based Inference}

\author{\textbf{Yang Yang$^{1}$, Severi Rissanen$^{2,3}$, Paul E.\ Chang$^{1,4}$, Nasrulloh Loka$^{1}$, Daolang Huang$^{2,3}$, }\\
\textbf{Arno Solin$^{2,3}$, Markus Heinonen$^{3}$, Luigi Acerbi$^{1}$}\\[0.3em]
$^{1}$Department of Computer Science, University of Helsinki, Finland \quad $^{2}$ELLIS Institute Finland \\
$^{3}$Department of Computer Science, Aalto University, Finland \quad $^{4}$Verda \\
{\small \texttt{\{yang.yang,paul.chang,nasrulloh.satrio,luigi.acerbi\}@helsinki.fi}} \\
{\small \texttt{\{severi.rissanen,daolang.huang,arno.solin,markus.o.heinonen\}@aalto.fi}}
}

\iclrfinalcopy %
\begin{document}

\maketitle

\begin{abstract}
Amortized simulator-based inference offers a powerful framework for tackling Bayesian inference in computational fields such as engineering or neuroscience, increasingly leveraging modern generative methods like diffusion models to map observed data to model parameters or future predictions. These approaches yield posterior or posterior-predictive samples for new datasets without requiring further simulator calls after training on simulated parameter-data pairs. 
However, their applicability is often limited by the prior distribution(s) used to generate model parameters during this training phase. To overcome this constraint, we introduce \emph{PriorGuide}, a technique specifically designed for diffusion-based amortized inference methods. PriorGuide leverages a novel guidance approximation that enables flexible adaptation of the trained diffusion model to new priors at test time, crucially without costly retraining. This allows users to readily incorporate updated information or expert knowledge post-training, enhancing the versatility of pre-trained inference models. \looseness-2
\end{abstract}

\section{Introduction}
Simulation-based inference (SBI) has become a key tool across scientific disciplines, enabling Bayesian inference for complex systems where the likelihood function $p(\vx \mid \vtheta)$ for data $\vx$ and model parameters $\vtheta$ is intractable, but one can easily simulate from the forward model $\vx \sim p(\vx \mid\vtheta)$~\citep{cranmer2020frontier}. Within the Bayesian paradigm, prior beliefs about parameters $\vtheta$ are updated with observed data $\vx$ to form a posterior distribution $p(\vtheta\mid\vx)$ \citep{gelman2013bayesian}. While traditional methods like Markov Chain Monte Carlo (MCMC;~\citealp{robert2007bayesian}) are effective when likelihoods are available, recent \emph{amortized inference} methods can directly learn the inverse mapping $\vx \mapsto p(\vtheta\mid \vx)$ from observations to posteriors using neural networks \citep{lueckmann2017flexible, radev2020bayesflow}. These amortized approaches, once trained, can rapidly infer posterior (parameters) or posterior-predictive (data) distributions given new observations, significantly speeding up the inference process.

Modern generative models, including diffusion models \citep{sohl2015deep, ho2020denoising, song2021score}, transformers \citep{vaswani2017attention}, and flow-matching techniques \citep{lipman2023flow}, have demonstrated state-of-the-art performance in tackling these inverse modeling tasks for amortized SBI \citep{muller2022transformers,wildberger2024flow,schmitt_consistency_2024, gloeckler2024all, chang2025amortized,whittle2025distribution,mittal2025amortized,hollmann2025accurate}. These models are trained on vast numbers of simulated parameter-data pairs $(\vtheta, \vx)$, often drawing parameters $\vtheta$ from a broad, uniform prior distribution to ensure comprehensive coverage of the parameter space.

However, this reliance on a fixed \emph{training prior} introduces significant limitations. Practitioners often possess domain knowledge that, if incorporated as a more specific prior, could substantially improve inference accuracy. Moreover, \emph{prior sensitivity analysis}---a crucial step for assessing the robustness of scientific conclusions to modeling assumptions---becomes cumbersome.
This practice is vital across many disciplines, \eg,~economists must validate the policy implications of macroeconomic models against different theoretical priors \citep{del2008forming}, climate scientists need to validate the climate sensitivity estimates over multiple sets of assumptions \citep{sherwood2020assessment}, and epidemiologists need to assess the sensitivity of pandemic forecasts to assumptions about disease transmission \citep{flaxman2020estimating}.
Changing priors is computationally challenging: non-amortized methods require costly re-simulation for each new prior, while amortized methods require retraining~\citep{elsemuller2024sensitivity}. Though practitioners often resort to approximations like importance sampling to avoid these costs, such methods fail when priors differ substantially. This limitation becomes increasingly problematic as the field trends towards \emph{foundation models} for SBI \citep{hollmann2025accurate,vetter2025effortless}.
While some recent amortized methods offer inference-time prior adaptation, they are often restricted to specific families of priors (\eg, factorized histograms or Gaussian mixtures pre-defined at training) or simple constraints \citep{elsemuller2024sensitivity,chang2025amortized, whittle2025distribution, gloeckler2024all}. 
The broader issue is the impracticality of pre-training over all potential tasks, such as all prior distributions a user might wish to employ. Recent successes of the \emph{test-time compute} paradigm~\citep{snell2025scaling} suggest that rather than attempting exhaustive amortization for all scenarios, models could be designed to flexibly incorporate specific requirements, such as a user-defined prior, through dedicated computations at inference time---an ability which was unattained so far.
For a comprehensive discussion of the related work, we refer the reader to \cref{app:related_work}.

\begin{figure}[t]
  \centering\scriptsize
  \setlength{\figurewidth}{.275\textwidth}
  \setlength{\figureheight}{.66\figurewidth}
  \pgfplotsset{
      width=\figurewidth,
      height=\figureheight,
      scale only axis,
      tick align=inside,
      ylabel near ticks,
      xlabel near ticks,
      x grid style={gray,dotted},
      xtick style={color=black!20},
      y grid style={gray,dotted},
      ytick style={color=black!20},
      ylabel={\strut},
      xlabel={\strut},
      ylabel style={align=center},
      yticklabels={},
      xticklabels={},
      ytick={\empty},
      xtick={\empty},
      title style={font=\bf,yshift=-6pt},
      major tick length=0pt,
      axis line style={rounded corners=2pt,draw=black!20},
	  legend cell align={left},every axis/.append style={legend style={draw=none,inner xsep=1pt, inner ysep=0.5pt, nodes={inner sep=1pt, text depth=2pt, scale=0.75,transform shape},fill=white,fill opacity=0.8}}
  }
  \begingroup
  \pgfplotsset{xlabel={$\theta_1$},ylabel={$\theta_2$},
    ytick={-1.6,-.8,0,.8,1.6},xtick={.2,.6,1,1.4,1.8},
    xmajorgrids,ymajorgrids,
    xlabel style={yshift=3pt},ylabel style={yshift=-3pt}}
  \begin{subfigure}{0.35\textwidth}
    \raggedleft
    \pgfplotsset{title={Priors\strut},ylabel={\textbf{Parameters}, $\bm{\theta}$ \\ $\theta_2$}}
    \input{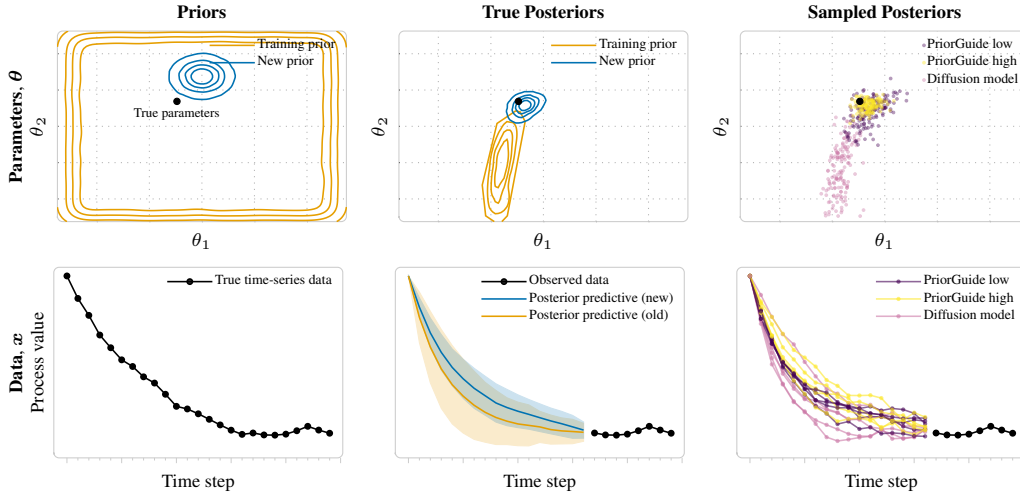}
  \end{subfigure}
  \hfill
  \begin{subfigure}{0.32\textwidth}
    \raggedleft
    \pgfplotsset{title={True Posteriors\strut}}
    \input{fig/fig1/2.theta_posterior_topmid.tex}
  \end{subfigure}
  \hfill
  \begin{subfigure}{0.32\textwidth}
    \raggedleft
    \pgfplotsset{title={Sampled Posteriors\strut}}
    \input{fig/fig1/3.theta_amortized_topright.tex}
  \end{subfigure}\\[.5em]
  \endgroup
  \begingroup
  \pgfplotsset{xlabel={Time step}, ytick={-10},ylabel={\strut},
    ylabel style={yshift=-3pt},xlabel style={yshift=3pt},
    xtick={0,5,10,15,20,25},minor tick num=4,
    major tick length=2pt,minor tick length=1pt}
  \begin{subfigure}{0.35\textwidth}
    \raggedleft
    \pgfplotsset{ylabel={\textbf{Data}, $\bm{x}$\\ Process value\strut}}
    \input{fig/fig1/4.x_observed_bottomleft.tex}
  \end{subfigure}
  \hfill
  \begin{subfigure}{0.32\textwidth}
    \raggedleft
    \input{fig/fig1/5.x_groundtruth_bottommid.tex}
  \end{subfigure}
  \hfill
  \begin{subfigure}{0.32\textwidth}
    \raggedleft
    \input{fig/fig1/6.x_amortized_bottomright.tex}
  \end{subfigure}
  \endgroup\\[-.5em]
  \definecolor{darkcyan0114178}{RGB}{0,114,178}
  \definecolor{darkgrey176}{RGB}{176,176,176}
  \definecolor{orange2301590}{RGB}{230,159,0}
  \caption{PriorGuide adapts a diffusion model to new prior information at test time for simulator-based inference---here for a time-series model.
\emph{Left:} Original broad \textcolor{orange2301590}{\bf training prior} and a new, more specific \textcolor{darkcyan0114178}{\bf target prior}. \emph{Middle:} Corresponding true posterior distributions over parameters ($\uparrow$) and predictive data ($\downarrow$) for each prior, given observed data. \emph{Right:} Posterior ($\uparrow$) and posterior-predictive ($\downarrow$) samples from the diffusion model trained on the old prior vs.\ those from PriorGuide at \emph{low} or \emph{high} test-time cost, illustrating PriorGuide's ability to match the \emph{new} posterior from the middle column.
}
  \label{fig:overall}
\vspace{-2.5em}
\end{figure}

In this paper, we introduce PriorGuide, a novel method that empowers diffusion-based amortized inference models with the ability to adapt to new prior distributions $q(\vtheta)$ at inference time, without the need for retraining the original score model trained under prior $\pt(\vtheta)$.\footnote{The intuitive requirement, quantified later, is that $q(\vtheta)$ should reside in regions of non-negligible mass under $\pt(\vtheta)$, to avoid out-of-distribution regions where the original score model would be poorly trained.}
PriorGuide leverages the guidance mechanisms inherent in diffusion models to steer the generative process according to the new target prior, 
seamlessly integrating new user-provided information during the sampling process. Crucially, this approach allows for a trade-off: users can invest more computational resources at inference time---such as more diffusion steps or adding refinement techniques like Langevin dynamics---to achieve higher inference fidelity. 
See \cref{fig:overall} for a conceptual overview.

\paragraph{Contributions}
Our main contribution is a principled framework for flexibly incorporating new prior distributions at inference time into pre-trained diffusion models for SBI.
We leverage a novel Gaussian mixture model approximation for effectively turning the target prior $q(\vtheta)$ into a tractable guidance term for the diffusion sampling process. We demonstrate empirically the effectiveness of PriorGuide on a range of SBI problems, showing its ability to accurately recover posterior and posterior-predictive distributions under various inference-time prior specifications. We show how sampling can be refined with additional Langevin dynamics steps, and study the tradeoff between test-time compute and inference accuracy within our framework.

\vspace{-0.5em}
\section{Background}
\label{sec:background}
\vspace{-0.8em}

The primary goal in many scientific applications is to infer model parameters $\vtheta$ given observed data $\vx$, or to predict future data $\vx^\star$. Bayesian inference provides a framework for computing the \emph{posterior} distribution over parameters or \emph{posterior predictive} distribution for new data:
\begin{equation} \label{eq:inference}
\begin{alignedat}{2}
    & p(\vtheta \mid \vx) \propto p(\vx \mid \vtheta) \, p(\vtheta), &\quad& \text{(posterior)} \\
    & p(\vx^\star \mid \vx) = \int p(\vx^\star \mid \vtheta, \vx) \, p(\vtheta \mid \vx) \diff\vtheta, &\quad& \text{(posterior predictive)}
\end{alignedat}
\end{equation}
where $p(\vtheta)$ is the prior and $p(\vx \mid \vtheta)$ is the likelihood.

The \emph{prior} $p(\vtheta)$ encodes the practitioner's beliefs about plausible parameter values \emph{before} observing data---beliefs informed by physical constraints, previous experiments, or theoretical considerations~\citep{gelman2013bayesian}. In practice, priors are typically simple---Gaussians, uniforms, smooth distributions---reflecting broad domain knowledge rather than precise specifications. Yet even among simple priors, the specific choice varies by application: different analyses, constraints, or domain knowledge call for different specifications, and modern statistical practice recommends validation of results under multiple prior choices~\citep{gelman2020bayesian}.

The \emph{likelihood} $p(\vx \mid \vtheta)$ encodes the statistical or mechanistic description of how data are generated from the model. In many real-world scenarios from finance to physics, \emph{evaluating} the likelihood is intractable, but \emph{generating} samples $\vx \sim p(\vx \mid \vtheta)$ from a simulator is feasible, leading to the field of Simulation-Based Inference (SBI; \citealp{cranmer2020frontier}).

A powerful paradigm within SBI is \emph{amortized inference}. Instead of performing inference from scratch for each $\vx$, amortized methods train a neural network $q_\phi$ once on a large dataset of simulated parameter-data pairs, $\mathcal{D}_{\text{sim}} = \{(\vtheta_i, \vx_i)\}_{i=1}^N$. Parameters $\vtheta_i$ are drawn from a \emph{training prior}, $\pt(\vtheta)$, and data $\vx_i \sim p(\vx \mid \vtheta_i)$. Once trained, $q_\phi$ provides rapid inference for new observations, amortizing the upfront computational cost.
The network $q_\phi$ is usually trained to approximate the posterior $p(\vtheta \mid \vx)$~\citep{lueckmann2017flexible,greenberg2019automatic,radev2020bayesflow}, the likelihood $p(\vx \mid \vtheta)$~\citep{papamakarios2019sequential}, or the posterior predictive distribution $p(\vx^\star \mid \vx)$~\citep{garnelo2018conditional,muller2022transformers}. Recently proposed architectures can flexibly perform \emph{all} of these tasks, using transformers~\citep{chang2025amortized} or diffusion models~\citep{gloeckler2024all}.

\subsection{Diffusion models}

Diffusion models are a powerful framework for generative modeling that transforms samples from arbitrary to simple distributions and vice versa through a gradual noising and denoising process \citep{sohl2015deep}. In the forward diffusion process, starting from a distribution $p(\diffvar_0)$ we can draw samples from (\eg, joint samples from the training prior and simulator), Gaussian noise is progressively added to the samples until, at the end of the process ($t = 1$), the distribution converges to a simple terminal distribution, typically Gaussian. In the Variance Exploding (VE) formulation~\citep{song2021score, karras2022elucidating},  the forward diffusion process can be described as:
\begin{equation}
\label{eq:forward_process_theta}
p(\diffvar_t) = \int p(\diffvar_t \mid \diffvar_0) \, p(\diffvar_0) \diff \diffvar_0 = \int \mathcal{N}(\diffvar_t \mid \diffvar_0, \sigma(t)^2 \mathbf{I}) \, p(\diffvar_0) \diff \diffvar_0, 
\end{equation}
where $p(\diffvar_t \mid \diffvar_0)$ is the transition kernel (here Gaussian), $\sigma(t)^2$ defines the noise variance schedule as a function of time (typically increasing with $t$), and $\diffvar_t$ represents the noisy samples at time $t$. The corresponding reverse process reconstructs the original sample distribution from noise, and can be formulated as either a stochastic differential equation (SDE) or an ordinary differential equation (ODE). 
The reverse SDE process takes the form~\citep{song2021score, karras2022elucidating}:
\begin{equation}
\diff\diffvar_t = -2 \dot{\sigma}(t) \sigma(t) \nabla_{\diffvar} \log p(\diffvar_t) \diff t + \sqrt{2 \dot{\sigma}(t) \sigma(t)} \diff\omega_t, \label{eq:reverse_sde_theta}
\end{equation}
where $\nabla_{\diffvar} \log p(\diffvar_t)$ is the \emph{score function} (gradient of the log-density), $\!\diff\omega_t$ is a Wiener process representing Brownian motion (noise), and $\dot{\sigma}(t)$ is the time derivative of the noise schedule.

\paragraph{Learning the score function}
The score function $\nabla_{\diffvar} \log p(\diffvar_t)$ can be approximated using a neural network $s(\diffvar_t, t)$, trained to minimize the denoising score matching loss \citep{hyvarinen2005, vincent2011connection, song2021score}:
\begin{equation}
\mathcal{L}_{\text{DSM}} = \mathbb{E}_{t \sim p(t)} \mathbb{E}_{\diffvar_0 \sim p(\diffvar_0)} \mathbb{E}_{\diffvar_t \sim \mathcal{N}(\diffvar_t \mid \diffvar_0, \sigma(t)^2 \mathbf{I})} \left[ \omega(t) \left\| s(\diffvar_t, t) - \nabla_{\diffvar_t} \log p(\diffvar_t \mid \diffvar_0) \right\|_2^2 \right].
\end{equation}
Here $p(t)$ is the distribution of noise levels sampled during training and $\omega(t)$ weights different noise levels in the loss. Once trained, the network $s(\diffvar_t, t)$ approximates the gradient of the log-probability density of noised distributions and affords sampling through the reverse SDE; \cref{eq:reverse_sde_theta}. Starting from a sample $\diffvar_t \sim \mathcal{N}(\diffvar_t \mid \diffvar_0, \sigma_{\text{max}}^2 \mathbf{I})$ for $t=1$ with sufficiently large $\sigma_{\text{max}}$, integrating the reverse process backward in time approximately reconstructs the original distribution $p(\diffvar_0)$.

The diffusion framework's flexibility stems largely from its ability to incorporate \emph{guidance mechanisms}, which afford steering the sampling process toward desired outcomes by including additional information or constraints. Notable examples include classifier guidance~\citep{dhariwal2021diffusion} and classifier-free guidance~\citep{ho2022classifier}, which enable variable-strength conditioned generation. For inverse problems, guidance methods exist to incorporate information about observations~\citep{chung2023diffusion, song2023pseudoinverse}.

\vspace{-0.5em}
\subsection{Diffusion-based amortized SBI}

Modern amortized SBI methods leverage highly expressive generative models for multiple inference tasks. Early applications of diffusion models to SBI include \cite{geffner2023compositional}, who demonstrated that learning conditional score functions via denoising score matching offers superior stability and sample quality compared to traditional flow-based baselines. Furthermore, Simformer \citep{gloeckler2024all} trains a diffusion model on samples from the joint distribution $p(\vtheta, \vx) = \pt(\vtheta) \, p(\x \mid \vtheta)$. Simformer employs a transformer architecture to model the score function $s_\phi(\vxi_t, t, \text{mask})$ of the noised joint variable $\vxi_t = (\vtheta_t, \vx_t)$ at diffusion time $t$. The mask specifies which components of $\vxi$ are conditioned upon and which are to be generated. By setting the mask appropriately (\eg, conditioning on $\vx$ to generate $\vtheta$), Simformer can sample from various conditionals, including the posterior $p(\vtheta \mid \vx)$ and posterior predictive $p(\vx^\star \mid \vx)$.
Crucially, these learned conditionals are implicitly tied to the training prior $\pt(\vtheta)$ used to generate the training data---applying a different prior $q(\vtheta)$ would require retraining the diffusion model with the new prior.

\subsection{Prior adaptation in amortized SBI} \label{sec:prior_adaptation}
Standard amortized SBI methods tie the learned posterior $q_\phi(\vtheta \mid \vx)$ to the fixed prior $p_\text{train}(\vtheta)$ used during training. Changing this prior traditionally requires retraining the model from scratch, which is computationally prohibitive for expensive simulators. A few recent amortized methods achieve prior flexibility by training over a \emph{meta-prior}---a distribution or predefined set of possible prior specifications---to learn how to incorporate different prior information at runtime. For instance, the Amortized Conditioning Engine (ACE;~\citealp{chang2025amortized}) allows users to specify factorized priors at runtime by encoding each one-dimensional prior density as a normalized histogram over a predefined grid; its transformer architecture is trained to process these specific histogram-based prior encodings alongside observed data. Similarly, the Distribution Transformer (DT;~\citealp{whittle2025distribution}) learns a direct mapping from a prior, itself represented as a Gaussian mixture model (GMM), to a GMM posterior, using attention mechanisms to transform the prior components based on the data. Sensitivity-aware SBI~\citep{elsemuller2024sensitivity} focuses on providing efficient sensitivity analysis to various modeling choices, including different priors, represented by a discrete set of possible alternative prior specifications. All of these methods enable prior adaptation by relying on their pre-training across a chosen meta-prior. Our approach, described next, sidesteps this requirement by preforming purely inference-time adaptation to flexibly handle new target priors.

\section{PriorGuide} \label{sec:methods}

PriorGuide offers a solution to take a diffusion-based amortized SBI model such as Simformer~\citep{gloeckler2024all}, trained on a training prior $\pt(\vtheta)$, and adapt it to a new target prior $q(\vtheta)$ at inference time, \emph{without retraining}. The objective is to perform standard SBI tasks such as sampling from the posterior or posterior predictive distribution \emph{under the new prior}, that is $q(\vtheta \mid \vx) \propto p(\vx \mid \vtheta)q(\vtheta)$ or $q(\vx^* \mid \vx)$, respectively.
The method achieves this by adjusting the score guidance during the diffusion sampling process. The key relationship for this adaptation is as follows:
\begin{prop} 
\label{eq:prop1}
Let the posterior under the original prior be $p(\vtheta \mid \x) \propto \pt(\vtheta) p(\x \mid \vtheta)$, and let the target posterior---the posterior under the new prior---be $q(\vtheta \mid \x) \propto q(\vtheta) p(\x \mid \vtheta)$. Then, sampling from $q(\vtheta \mid \x)$ is equivalent to sampling from $\ratio(\vtheta) p(\vtheta \mid \x)$ with $\ratio(\vtheta) \equiv \frac{q(\vtheta)}{\pt(\vtheta)}$ the prior ratio.
\vspace{-1.2em}
\end{prop}

\begin{proof}
We can rewrite the target posterior \(q(\vtheta\mid\x)\) as
\[
q(\vtheta \mid \x) \propto q(\vtheta)p(\x \mid \vtheta) = \frac{q(\vtheta)}{\pt(\vtheta)}\, \pt(\vtheta)p(\x \mid \vtheta) \propto \frac{q(\vtheta)}{\pt(\vtheta)}\, p(\vtheta \mid \x) = \ratio(\vtheta) p(\vtheta\mid\x),
\]
where the prior ratio $\ratio(\vtheta) \equiv \frac{q(\vtheta)}{\pt(\vtheta)}$ takes the role of an importance weighing function, analogous to the correction applied in multi-round neural posterior estimation \citep{lueckmann2017flexible}.
\vspace{-1em}
\end{proof}
Next, we focus on the task of sampling from the target posterior. First, we show in \cref{sec:new_prior_guidance} how introducing the new prior amounts to adding a guidance term to the diffusion process for posterior sampling. In \cref{sec:guidance_approx}, we develop analytical approximations to make this tractable. \cref{sec:prioguide_langevin} shows how we can provide guarantees using corrective Langevin steps. Finally, in \cref{sec:priorguide_post_pred} we show how our results for posterior sampling readily extend to the posterior predictive case.

\subsection{Target prior as guidance}
\label{sec:new_prior_guidance}

Assume we have a diffusion model trained under $\pt(\vtheta)$ to sample from $p(\vtheta\mid\x)$ with learnt score model $s(\vtheta_t,t,\x)$. PriorGuide leverages the fact that we can relate the score of the target posterior $q(\vtheta\mid \x)$ to the original score.
The marginal pdf at time $t$ of the diffusion process for $q(\vtheta\mid \x)$ is:
\begin{align}
\label{eq:modified_posterior}
q(\vtheta_t\mid\x) &\propto \int \ratio(\vtheta_0) p(\vtheta_0\mid\x)p(\vtheta_t\mid\vtheta_0) \diff\vtheta_0,
\end{align}
which is written as an integral over \(\vtheta_0\) by noting that \(q(\vtheta_0\mid\x) \propto \ratio(\vtheta_0) p(\vtheta_0\mid\x)\) and then we propagated this to time \(t\) by convolution with the transition kernel \(p(\vtheta_t\mid\vtheta_0)\). 
Thus, the score is:
\begin{align}
\label{eq:modified_score1}
\nabla_{\vtheta_t} \log q(\vtheta_t\mid\x) &= \nabla_{\vtheta_t} \log \int \ratio(\vtheta_0) p(\vtheta_0\mid\x)p(\vtheta_t\mid\vtheta_0,\x) \diff\vtheta_0 \\
\label{eq:modified_score2}
&= \nabla_{\vtheta_t} \log \int \ratio(\vtheta_0) p(\vtheta_0\mid\vtheta_t, \x)p(\vtheta_t\mid\x) \diff\vtheta_0 \\
\label{eq:modified_score3}
&= \nabla_{\vtheta_t} \log p(\vtheta_t\mid\x) + \nabla_{\vtheta_t} \log \int \ratio(\vtheta_0) p(\vtheta_0\mid\vtheta_t, \x) \diff\vtheta_0 \\
&= \underbrace{s(\vtheta_t, t, \x)}_{\text{original score}} + \nabla_{\vtheta_t} \log \underbrace{\mathbb{E}_{p(\vtheta_0\mid\vtheta_t, \x)}}_{\text{reverse kernel}}\Big[ \underbrace{\ratio(\vtheta_0)}_{\text{prior ratio}}\Big], \vphantom{\int} \label{eq:posterior_guidance_decomposition}
\end{align}
where in \cref{eq:modified_score2} we re-express the joint probability $p(\vtheta_0\mid\x)p(\vtheta_t\mid\vtheta_0, \x) = p(\vtheta_0, \vtheta_t\mid\x)$ as $p(\vtheta_0\mid\vtheta_t,\x)p(\vtheta_t\mid\x)$, which allows us to separate the contribution of the new prior guidance from the original score model \(s(\vtheta_t,t,\x)\). In multiple steps we exploit the fact that multiplicative constants inside the integral disappear under the score. In conclusion, \cref{eq:posterior_guidance_decomposition} expresses the score of the target (new) posterior as the old score, which we have, plus a guidance term which we can estimate.

\paragraph{Guided diffusion}
We can draw samples from the posterior distribution via the reverse diffusion process using the modified score in \cref{eq:posterior_guidance_decomposition}. The first term is the trained score model and the second term estimates how the new prior's influence propagates to time $t$ (guidance term).
This is a common way to implement a guidance function \citep{chung2023diffusion,song2023pseudoinverse,song2023loss,rissanen2024hunch}, which now depends on the prior ratio. 
The core challenge lies in evaluating the expectation over $\vtheta_0$, which is intractable and requires simulating the reverse SDE. To make this tractable, we develop analytical approximations in the following section. 

\paragraph{Ensuring prior coverage}
A crucial consideration for stable guidance is ensuring the new prior $q(\vtheta)$ remains within regions adequately covered by the training prior $\pt(\vtheta)$.\footnote{Within this coverage, $q(\vtheta)$ can differ substantially from $\pt(\vtheta)$---for instance, being more concentrated, multimodal, or shifted---enabling meaningful prior adaptation.}
If $q(\vtheta)$ assigns significant mass to regions where $\pt(\vtheta)$ has negligible support (\ie, is \emph{out-of-distribution} or OOD;~\citealp{lee2018simple,nalisnick2019deep}), two related issues can arise: (a) in these regions, the learned score $s(\vtheta_t, t, \x)$ is likely a poor representation of the true score, as the training set contained few examples in these regions; and (b) the prior ratio $\ratio(\vtheta) = q(\vtheta)/\pt(\vtheta)$ can become arbitrarily large or ill-defined, destabilizing the guidance mechanism. Lack of coverage can be quantified using OOD metrics~\citep{lee2018simple,nalisnick2019deep,schmitt2023detecting,huang2024learning}.
Notably, the requirement that $q(\vtheta)$ should be covered by $\pt(\vtheta)$ is typically not restrictive since the common practice is to train amortized models on broad training priors.
This diagnostic check is detailed in \cref{app:prior_diagnostic}.

\subsection{Approximating the guidance function}\label{sec:guidance_approx}

To approximate the guidance term in \cref{eq:posterior_guidance_decomposition} efficiently while maintaining flexible test-time priors, we introduce two approximations. Following recent work \citep{song2023pseudoinverse,peng2024improving,rissanen2024hunch}, we model the reverse transition kernel as a Gaussian. We then introduce a novel approach that represents $\ratio(\vtheta)$ as a Gaussian mixture model. 
This yields an analytical solution for the guidance, circumventing the issue of estimating the score of an expectation via Monte Carlo, which would suffer from both bias and variance.

\paragraph{Reverse transition kernel approximation}
We first approximate the reverse transition kernel \(p(\vtheta_0\mid\vtheta_t, \x)\) as a multivariate Gaussian distribution: %
\begin{equation} \label{eq:reverse_approx}
p(\vtheta_0\mid\vtheta_t, \x) \approx \mathcal{N}\left(\vtheta_0 \mid  \vmu_{0|t}(\vtheta_t, \x), \vSigma_{0|t} \right)
\end{equation}
whose mean is obtained from the score function via Tweedie's formula~\citep{song2019generative}:
\begin{equation}
\label{eq:tweedie}
\vmu_{0|t}(\vtheta_t, \vx) = \vtheta_t + \sigma(t)^2 \nabla_{\vtheta_t} \log p(\vtheta_t \mid \vx).
\end{equation}
This approximation is common in the guidance literature \citep{chung2023diffusion,song2023pseudoinverse, boystweedie, peng2024improving, rissanen2024hunch, finzi2023user, bao2022analytic}. 
For the covariance matrix $\vSigma_{0|t}$, we adopt a simple yet effective approximation inspired by \citep{song2023pseudoinverse, ho2022video}:
\begin{equation} \label{eq:simple_cov}
\vSigma_{0|t} = \frac{\sigma(t)^2}{1 + \sigma(t)^2} \mathbf{I}.
\end{equation}
This approximation acts as a time-dependent scaling factor that naturally aligns with the diffusion process---starting at the identity matrix when $t=1$ and approaching zero as $t \to 0$, effectively increasing the precision of our prior guidance at smaller timesteps. This approximation becomes exact for all $t$ if the posterior under the original target distribution is $p(\vtheta_0\mid \x)=\mathcal{N}(\vtheta_0\mid \bm{0}, \mathbf{I})$.\footnote{More advanced covariance approximations are explored in the literature \citep{boystweedie, baoanalytic, peng2024improving, finzi2023user, manorposterior, rozet2024learning, rissanen2024hunch}, but introduce added computational costs or  implementation complexity. For this work, we focus on the simple and effective \cref{eq:simple_cov}, which already yields strong results  especially when combined with the Langevin refinement in \cref{sec:langevin}. 
}

\paragraph{Prior ratio approximation}
With the goal of obtaining a closed-form solution for the guidance, we then approximate the prior ratio function $\ratio(\vtheta) = \frac{q(\vtheta)}{p(\vtheta)}$ as a generalized mixture of Gaussians:
\begin{equation} \label{eq:gmm_ratio}
\ratio(\vtheta) \approx \sum_{i=1}^K w_i \mathcal{N} (\vtheta \mid \vmu_i, \vSigma_i), \qquad \ratio(\vtheta) \ge 0,
\end{equation}
where \(\{w_i, \vmu_i, \vSigma_i\}_{i=1}^{K}\) represent the weights, means and covariance matrices of the mixture, with $K$ a hyperparameter denoting the number of mixture components.\footnote{We empirically analyze the sensitivity to $K$ in \cref{app:ablation_gmm_components}, finding that performance is robust once $K$ provides sufficient expressivity (\eg, $K = 20$ in this paper), and increasing $K$ has negligible computational cost.} 
Since this represents a ratio rather than a distribution, the mixture weights need not be positive nor sum to one, as long as the ratio remains non-negative, potentially enabling more expressive approximations such as subtractive mixtures \citep{loconte2024subtractive}.
Notably, when \(p(\vtheta)\) is uniform, $r(\vtheta) \propto q(\vtheta)$. Since the guidance term is a gradient of a log-expectation, the constant factor vanishes, allowing us to directly specify $q(\vtheta)$ as a Gaussian mixture.
For the more general case of non-uniform training distributions, obtaining the Gaussian mixture approximation for the ratio is a standard function approximation task~\citep{SORENSON}. Crucially, since the densities of both $p_{\text{train}}$ and $q$ are analytically known, fitting the ratio avoids the instability and high variance inherent to statistical density-ratio estimation from finite samples. We provide a straightforward gradient-based fitting procedure in~\cref{app:inference_algorithm}.

\paragraph{Guidance term}
Plugging in \cref{eq:reverse_approx} and \cref{eq:gmm_ratio}, the guidance from \cref{eq:posterior_guidance_decomposition} becomes:
\begin{equation}
\label{eq:prior_approx}
\nabla_{\vtheta_t}\log\mathbb{E}_{p(\vtheta_0\mid\vtheta_t, \x)}\left[\ratio(\vtheta_0)\right] \approx \nabla_{\vtheta_t} \log \int \sum_{i=1}^K w_i \mathcal{N} (\vtheta_0 \mid \vmu_i, \vSigma_i) \mathcal{N} (\vtheta_0 \mid \vmu_{0|t}(\vtheta_t, \x), \vSigma_{0|t}) \diff\vtheta_0.
\end{equation}
This integral can be solved analytically (full derivation in \cref{app:proofs}), yielding:
\begin{align} %
\nabla_{\vtheta_t} \log \mathbb{E}_{p(\vtheta_0\mid\vtheta_t,\x)}[\ratio(\vtheta_0)]  &\approx  %
\sum_{i=1}^{K} \tilde w_i (\vmu_i - \vmu_{0|t}(\vtheta_t))^\top \widetilde{\vSigma}_i^{-1} \nabla_{\vtheta_t} \vmu_{0|t}(\vtheta_t), \label{eq:priorw}
\end{align}
where $\widetilde{\vSigma}_i = \vSigma_i + \vSigma_{0|t}$ and $\tilde w_i = w_i \mathcal{N}(\vmu_i \mid \vmu_{0|t}(\vtheta_t, \x), \widetilde{\vSigma}_i)/{\sum_{j=1}^K w_j \mathcal{N}(\vmu_j \mid \vmu_{0|t}(\vtheta_t, \x), \widetilde{\vSigma}_j)}$.

Finally, the PriorGuide update to the mean of the reverse kernel can be expressed concisely using Tweedie's formula, \cref{eq:tweedie}, and our derived guidance term, \cref{eq:priorw}:
\begin{equation}
\mu_{0|t}^\text{new}(\vtheta_t, \x)  = \mu_{0|t}(\vtheta_t, \x) + \sigma(t)^2 \sum_i^{K} \tilde w_i (\vmu_i - \vmu_{0|t}(\vtheta_t, \x))^\top \widetilde{\vSigma}_i^{-1} \nabla_{\vtheta_t} \vmu_{0|t}(\vtheta_t, \x).
\end{equation}
This update intuitively combines the original prediction $\mu_{0|t}(\vtheta_t)$ based on the training prior with a weighted sum of correction terms from our new prior. The correction magnitude is controlled by both the noise schedule $\sigma(t)^2$ and the distance between the mixture components and current prediction.

\subsection{Asymptotically correct sampling with Langevin dynamics}
\label{sec:prioguide_langevin}

The diffusion sampling process detailed in \cref{sec:guidance_approx} is approximate due to the Gaussian approximation of the reverse transition kernel $p(\vtheta_0\mid\vtheta_t, \x)$. However, as stated in the following proposition, this Gaussian approximation becomes correct on low noise levels, rendering our guidance term accurate.\looseness-1
\begin{prop}
    \label{eq:prop2}
    As $t,\sigma(t) \to 0$, the approximation $\mathcal{N}(\vtheta_0\mid \vtheta_t + \sigma(t)^2\nabla_{\vtheta_t} \log p(\vtheta_t), \frac{\sigma(t)^2}{1 + \sigma(t)^2} \mathbf{I})$ converges to the true $p(\vtheta_0\mid \vtheta_t)$, under mild regularity conditions on $p(\vtheta_0)$.\looseness-3%
\end{prop}
We provide the exact statement and a proof in \cref{app:proofs}. Close-by statements are well-known in the diffusion literature \citep{finzi2023user, ho2022video}. Thus, the guidance approximation is correct at low noise levels (up to the GMM prior ratio fit accuracy), and we can run accurate Langevin dynamics MCMC sampling ~\citep{sarkka2019applied}. To incorporate this with the regular diffusion process, we run $N_L$ Langevin steps after each regular diffusion step, effectively transforming the sampling into an annealed MCMC process that resembles methods used in compositional generation \citep{geffner2023compositional, du2023reduce} and early unconditional diffusion models \citep{song2021score}. This refinement increases the total inference cost by a factor of $N_L + 1$.

Thus, PriorGuide has two main hyperparameters: the number of diffusion steps $N > 0$, as per any diffusion model, and the number of interleaved Langevin steps $N_L \ge 0$. For test-time sampling, the total number of function evaluations (NFE), or forward passes of the trained network, is $N \times (N_L + 1)$.

\subsection{PriorGuide posterior predictive sampling}
\label{sec:priorguide_post_pred}

PriorGuide is readily applied to compute posterior predictive distributions as well under a new prior.
Starting from a diffusion model trained to generate samples from the joint posterior predictive distribution, $p(\x^\star,\vtheta\mid \x)$, we can marginalize over $\vtheta$ to get the posterior predictive $p(\x^\star \mid \x)$. The joint posterior predictive under the new prior becomes:
\[
q(\x^\star,\vtheta\mid\x) = q(\x^\star\mid\vtheta,\x)q(\vtheta\mid\x) \propto q(\x^\star\mid\vtheta,\x)\ratio(\vtheta) p(\vtheta\mid\x) = \ratio(\vtheta) p(\x^\star,\vtheta\mid \x),
\]
which results in a posterior predictive version of \cref{eq:posterior_guidance_decomposition}:
\begin{align}
    \nabla_{\x^\star_t, \vtheta_t} \log q(\x^\star_t, \vtheta_t\mid\x) %
    &= s(\x^\star_t, \vtheta_t, t, \x) + \nabla_{\x^\star_t,\vtheta_t} \log \mathbb{E}_{p(\vtheta_0\mid\x^\star_t,\vtheta_t, \x)}\left[ \ratio(\vtheta_0)\right].\label{eq:posterior_predictive_guidance_decomposition}
\end{align}
The posterior predictive and posterior guidance terms differ only in the conditioning information for the score and reverse transition kernel. Thus, everything presented earlier in this section applies to this scenario, and the posterior predictive is obtained from the previous formulas with substitutions $p(\vtheta_0\mid\vtheta_t, \x)\rightarrow p(\vtheta_0\mid\vxit_t, \x)$, $\vmu_{0|t}(\vtheta_t, \x)\rightarrow \vmu_{0|t}(\vxit_t, \x)$ and $\nabla_{\vtheta}\rightarrow\nabla_{\vxit}$, where $\vxit_t \equiv (\vx_t^\star, \vtheta_t)$.

\section{Experiments}
\label{sec:experiments}

We empirically evaluate PriorGuide across a range of SBI problems, focusing on its ability to adapt to new priors at test time for posterior and posterior predictive inference.
First, \cref{sec:2dexamples} provides an intuitive demonstration on a 2D problem. In \cref{sec:posterior}, we evaluate posterior inference on several SBI problems, comparing PriorGuide to existing methods that support test-time prior adaptation.
\cref{sec:post_pred} examines PriorGuide's performance on challenging posterior predictive tasks.
Finally, \cref{sec:langevin} studies the trade-off between computational cost and inference accuracy. Full experimental details can be found in \cref{app:experimental_details}, with \cref{app:additional_baselines} reporting additional baseline results. We also conduct two ablation studies, including an analysis of the sensitivity of PriorGuide to the distance between training and test-time priors (\cref{app:sensitity_prior_distance}) and a study of the impact of the number of GMM components used to approximate the prior ratio (\cref{app:ablation_gmm_components}).
The code is available at \url{https://github.com/acerbilab/prior-guide}.

\subsection{Illustrative example of test-time prior adaptation} \label{sec:2dexamples}

We illustrate PriorGuide's capabilities on Two Moons, a two-dimensional SBI model with a bimodal posterior~\citep{greenberg2019automatic}.
We train the diffusion model under a uniform prior $\pt(\vtheta)$ over $[-1,1]^2$, and 
test how PriorGuide handles a new prior $q(\vtheta)$ at test time.
\cref{fig:exp1} shows that PriorGuide incorporates the new prior, matching well the true Bayesian posterior under the new prior.

\subsection{Test-time prior adaptation for posterior inference} \label{sec:posterior}

We evaluate PriorGuide's posterior inference capabilities on six SBI problems (see \cref{tab:benchmarks_overview_posterior}), ranging from established SBI benchmarks to real models from engineering and neuroscience: Two Moons~\citep{lueckmann2021benchmarking}; the Ornstein-Uhlenbeck Process (OUP;~\citealp{uhlenbeck1930theory}); the Turin model of radio propagation~\citep{turin1972statistical}; the Gaussian Linear model~\citep{lueckmann2021benchmarking} and its high-dimensional variant; and the Bayesian Causal Inference model of multisensory perception (BCI;~\citealp{kording2007causal}). Training priors $\pt(\vtheta)$ for the base diffusion model (Simformer;~\citealp{gloeckler2024all}) were uniform or Gaussian; details in \cref{app:experimental_details}.

For baselines, we consider the base Simformer (no prior adaptation) and the Amortized Conditioning Engine (ACE;~\citealp{chang2025amortized}), one of several approaches~\citep{elsemuller2024sensitivity,whittle2025distribution} that amortizes test-time prior adaptation for posterior inference by pre-training on a variety of possible (factorized) priors. PriorGuide is more flexible than ACE by not needing pretraining on specific priors---instead, it modifies a diffusion-based amortized inference model at runtime---, and can represent correlated and non-factorized priors.
Detailed comparisons against additional non-amortized methods---including classic algorithms (rejection sampling and sampling-importance-resampling) and neural likelihood estimation (NLE;~\citealp{papamakarios2019sequential}) with MCMC---are provided in \cref{app:additional_baselines}. While often computationally expensive, these methods serve as fundamental benchmarks for posterior sampling.

We consider three different families of target priors: \emph{mild}, \emph{strong} and \emph{mixture}. 
Mild and strong priors are defined as multivariate Gaussian distributions with means drawn from a uniform box and diagonal covariance matrices, where the strong priors have smaller standard deviations; these represent scenarios with varying degrees of available information. Mixture priors are defined as a mixture distribution with two multivariate Gaussian components with the same setup as the strong priors; this can represent situations with distinct, competing hypotheses about the parameter values.
For each prior family, we randomly generate ten possible prior parameterizations $q^{(i)}(\vtheta)$, sample ten parameter vectors $\vtheta_{i,j} \sim q^{(i)}(\vtheta)$ from that prior, and simulate one observed dataset per $\vtheta_{i,j}$, $\vx_{i,j} \sim p(\x\mid\vtheta_{i,j})$ to evaluate the methods.
See \cref{app:test-time-prior-generations} for the full procedure.

We measure each method's performance using: 1) the root mean squared error (RMSE) between the true parameter and the samples from the estimated posterior; 2) the classifier two-sample test (C2ST) 
between the estimated posterior samples and \emph{ground-truth} posterior samples; 3) the mean marginal total variation distance (MMTV) between the estimated vs. \emph{ground-truth} posterior samples. For RMSE and MMTV, lower is better, while for C2ST, closer to 0.5 is better.

Results in \cref{tab:uniform} show that PriorGuide largely improves inference accuracy over the base Simformer model in all scenarios, making use of the prior information provided at test time, and achieves leading performance in most cases, especially when stronger prior beliefs (\emph{strong} and \emph{mixture}) are presented.

\tableposterior

\vspace{-0.5em}
\subsection{Test-time prior adaptation for data prediction} \label{sec:post_pred}
\vspace{-0.5em}

We next evaluate PriorGuide's ability to perform posterior predictive inference under new target priors, focusing on forecasting or retrocasting scenarios, as shown in \cref{fig:overall}.
We use the OUP and Turin models, both of which generate time series trajectories (\cref{fig:exp_post_pred}).
We employ the same procedure and test-time prior setup (mild, strong, and mixture) from \cref{sec:posterior}.
For each target prior, we condition the model on partial trajectories: in half of the cases, the first $30\%$ of the trajectory, and for the other half, the last $30\%$. The task is always to predict the unobserved $70\%$ of the trajectory.\footnote{This task formulation prevents simple data interpolation as induced by sampling time indices randomly.}

We evaluate the performance of all methods using RMSE and the maximum mean discrepancy (MMD) with an exponentiated quadratic kernel between the ground-truth trajectory $\vx_o$ and generated posterior predictive samples.
\cref{fig:exp_post_pred} shows how example posterior predictive distributions from PriorGuide closely match the true data. Results in
\cref{tab:postpred} show that PriorGuide can generate reliable posterior predictive samples and achieve performance on par with or better than other methods.

\begin{figure}
  \centering
\includegraphics[width=\textwidth]{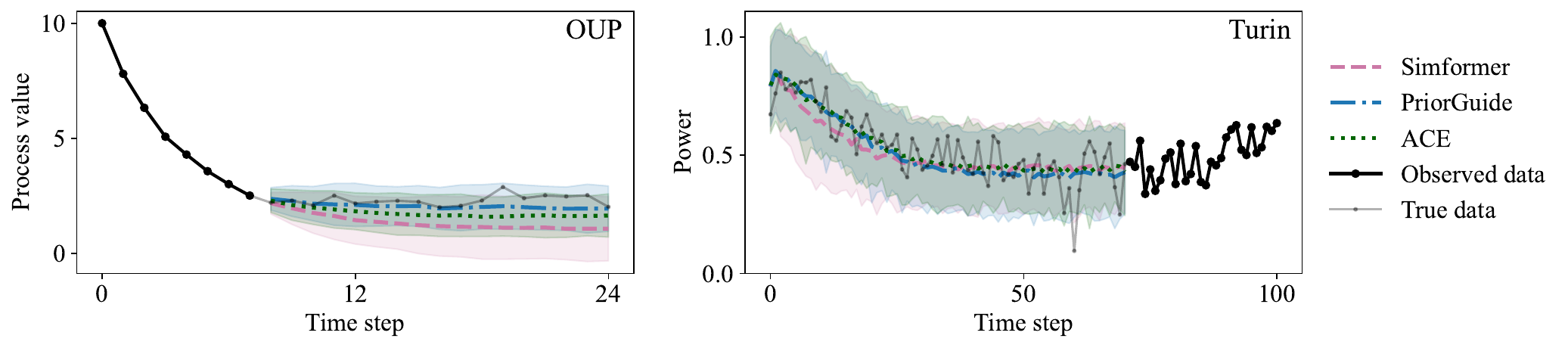}
\vspace{-2em}
  \caption{Example posterior predictive distributions for OUP and Turin models (strong priors).}
  \label{fig:exp_post_pred}
  \vspace{-1.5em}
\end{figure}

\tablepostpred

\subsection{Test-time refinement via corrective Langevin dynamics} \label{sec:langevin}
PriorGuide supports improving the sampling quality by adding Langevin dynamic steps to the diffusion process, at the cost of additional test-time compute.
We examine posterior inference accuracy---measured by MMTV, but similar results hold for other metrics---on the OUP and Turin models as a function of the number of diffusion steps $N$ and Langevin steps $N_L$. 
These two can be combined into a single computational cost metric, the \emph{number of function  evaluations} (NFEs), \ie calls to the score model.
In \cref{fig:mainfig}, we visualize the relationship between MMTV, $N$ and $N_L$. The Pareto front shows that the best posterior inference is achieved by combining moderate diffusion steps ($N \sim$ 25--50) with increasing Langevin corrections if the NFE budget allows it. The sample quality in general improves with more compute, implying that a simple way to calibrate the sampling parameters is to increase NFE until the output distribution does not change.

\begin{figure}[t]
    \centering
    \begin{subfigure}[b]{0.473\textwidth}
        \centering
        \includegraphics[width=\textwidth]{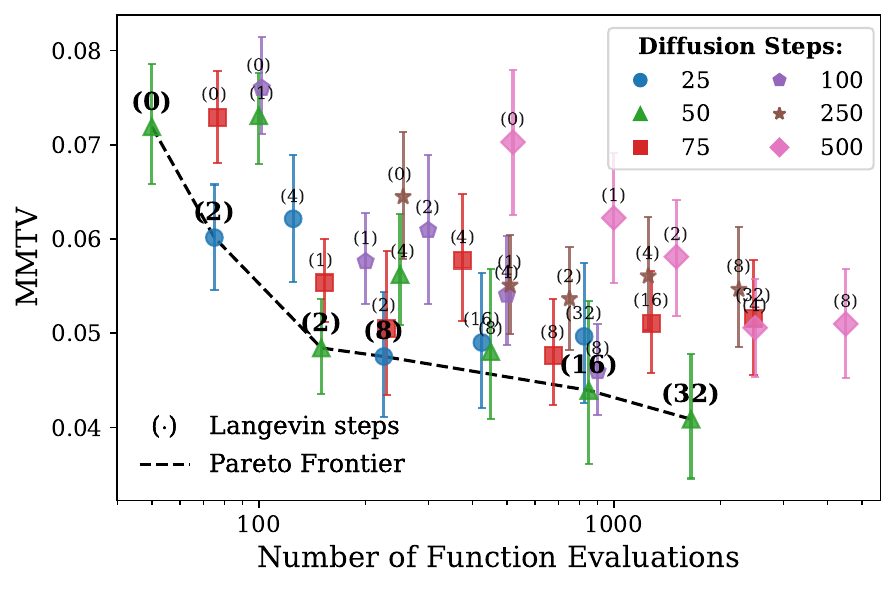}
        \caption{Posterior inference on OUP, strong priors}
        \label{fig:subfig1}
    \end{subfigure}
    \hfill
    \begin{subfigure}[b]{0.48\textwidth}
        \centering
        \includegraphics[width=\textwidth]{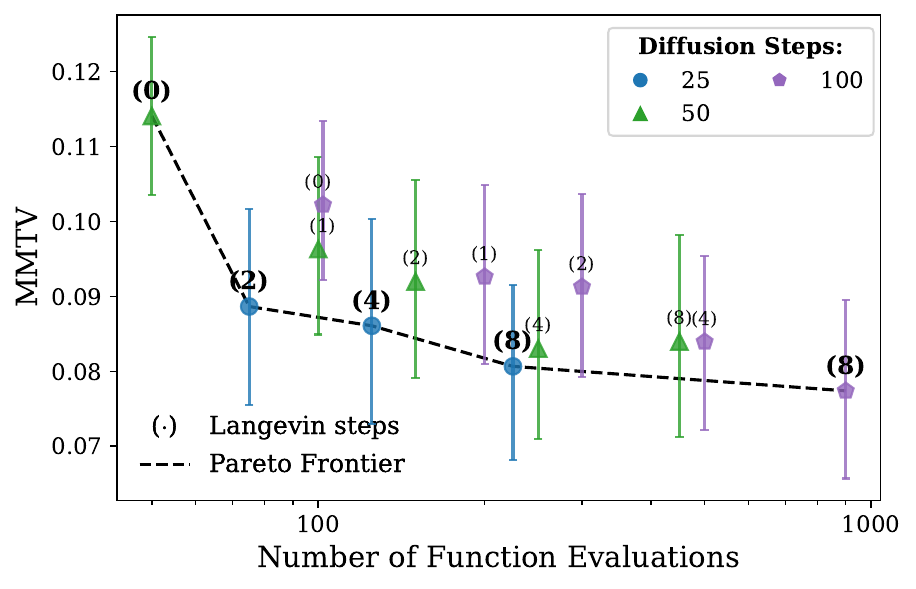}
        \caption{Posterior inference on Turin, strong priors}
        \label{fig:subfig2}
    \end{subfigure}
    \caption{Pareto frontiers with respect to number of function evaluations (NFEs) and MMTV on posterior inference for OUP and Turin, with varying number of diffusion and Langevin steps.\looseness-1 }%
    \label{fig:mainfig}
\vspace{-1.5em}
\end{figure}

\section{Discussion}\label{sec:discussion}

PriorGuide enables amortized diffusion-based SBI models to adapt to new prior distributions without retraining, an example of the \emph{test-time compute} paradigm in extending pre-trained model capabilities with dedicated computations at test time which repurpose diffusion guidance for Bayesian inference. 
In practice, PriorGuide is recommended for moderate-to-high dimensional problems ($4 < D \lesssim 20$) where simulators are mildly-to-very expensive, making retraining a simulator model burdensome; for settings requiring complex, non-factorized priors, where amortized methods restricted by pre-defined meta-priors (\eg, ACE) yield unsatisfactory performance; and for applications where prior adaptation needs to be fast but not strictly \emph{instant}---in terms of pure speed, fully amortized methods without test-time compute remain the best choice.

\paragraph{Limitations} 
PriorGuide's effectiveness relies on the new prior $q(\vtheta)$ having substantial overlap with the training prior $\pt(\vtheta)$; out-of-distribution (OOD) target priors can lead to inaccurate learned scores and unstable guidance. The method also employs approximations---a Gaussian for the reverse transition kernel $p(\vtheta_0 \mid \vtheta_t, \vx)$ and a Gaussian mixture model for the prior ratio function $\ratio(\vtheta)$---which can introduce inaccuracies, particularly for complex prior ratio shapes. Furthermore, current guidance calculations, involving matrix operations for the GMM components, may pose scalability challenges for high-dimensional parameter spaces $(\dim(\vtheta) \gg 20)$. PriorGuide sampling can be computationally intensive: although our method avoids retraining the base model---a key benefit with expensive simulators (\eg, Turin model)---the iterative guided diffusion, particularly with interleaved Langevin refinement steps ($N_L$), incurs a cost. The number of function evaluations (NFEs) increases with diffusion ($N$) and Langevin steps, creating an accuracy-speed trade-off. This may render PriorGuide less suited than fully amortized methods for applications requiring very rapid inference. Advanced covariance approaches mentioned in \cref{sec:guidance_approx} have the potential to speed up the method by requiring fewer NFEs, and represent an interesting future direction.

\paragraph{Conclusions} 
The ability of PriorGuide to decouple expensive simulator runs (for training the base model) from the specification of changing prior beliefs offers significant practical advantages. It allows for post-hoc prior sensitivity analyses and facilitates the direct incorporation of domain expert knowledge post-training, reducing the overall computational footprint in scientific workflows by avoiding the need for repeated model retraining when assumptions change.

\newpage

\subsubsection*{Acknowledgements}
This work was a part of Finland's Ministry of Education and Culture’s Doctoral Education Pilot under Decision No. VN/3137/2024-OKM-6 (The Finnish Doctoral Program Network in Artificial Intelligence, AI-DOC). 
The project was also supported by the Research Council of Finland (Flagship programme: Finnish Center for Artificial Intelligence, FCAI). 
NL was funded by Business Finland (project 3576/31/2023) and LUMI AI Factory (EU Horizon Europe Joint Undertaking and its members including top-up funding 
by Ministry of Education and Culture). 
LA was supported by Research Council of Finland grants 356498 and 358980. SR, MH, and AS acknowledge funding from the Research Council of Finland (grants 339730,
362408, 334600).  The authors also acknowledge the research environment provided by
ELLIS Institute Finland.

We acknowledge CSC – IT Center for Science, Finland, for computational resources provided by the LUMI supercomputer, owned by the EuroHPC Joint Undertaking and hosted by CSC and the LUMI consortium (LUMI projects 462000864 and 462000873). Access was provided through the Finnish LUMI-OKM allocation. We acknowledge the computational resources provided by the
Aalto Science-IT project.

Funded by the European Union. Views and opinions expressed are however those of the author(s) only and do not necessarily reflect those of the European Union or the granting authority. Neither the European Union nor the granting authority can be held responsible for them.

\begin{figure}[htbp]
  \centering
  \begin{minipage}{0.3\textwidth}
    \centering
    \includegraphics[width=\linewidth]{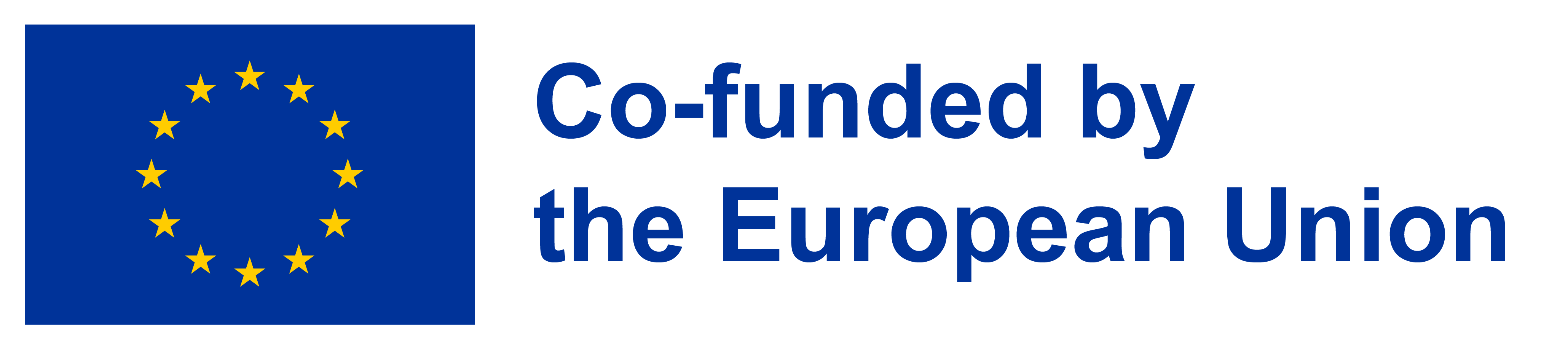}\\
  \end{minipage}\hspace{4em}
  \begin{minipage}{0.3\textwidth}
    \centering
    \includegraphics[width=\linewidth]{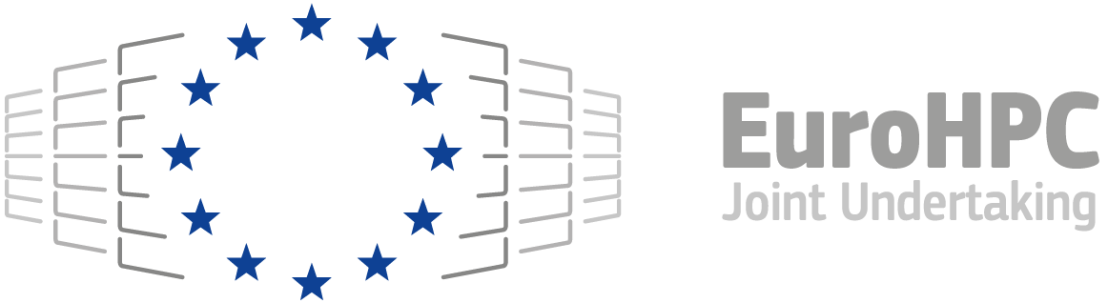}\\
  \end{minipage}
\end{figure}

\subsubsection*{Ethics statement}
This work uses only synthetic datasets, with no sensitive data involved. The methods are for research purposes and pose no foreseeable ethical risks. We have followed the ICLR Code of Ethics.

\subsubsection*{Reproducibility statement}
The code is available at \url{https://github.com/acerbilab/prior-guide}. All experiments use synthetic datasets. Algorithmic details are presented in \cref{app:methods}, and all experimental details are specified in \cref{app:experimental_details}.

\bibliography{references}
\bibliographystyle{iclr2026_conference}

\newpage
\begin{center}
\LARGE
    \textsc{Appendix}
\end{center}

\appendix
\input{appendix}

\end{document}

%% file: math_commands.tex
\usepackage{amsmath,amsfonts,bm}

\def\eqref#1{equation~\ref{#1}}

\def\1{\bm{1}}

\def\rvb{{\mathbf{b}}}

\def\rvs{{\mathbf{s}}}

\def\rmI{{\mathbf{I}}}

\def\vmu{{\bm{\mu}}}
\def\vtheta{{\bm{\theta}}}
\def\vphi{{\bm{\phi}}}
\def\vsigma{{\bm{\sigma}}}
\def\vSigma{{\bm{\Sigma}}}
\def\va{{\bm{a}}}

\def\vx{{\mathbf{x}}}

\def\vtheta{{\bm{\theta}}}

\DeclareMathAlphabet{\mathsfit}{\encodingdefault}{\sfdefault}{m}{sl}
\SetMathAlphabet{\mathsfit}{bold}{\encodingdefault}{\sfdefault}{bx}{n}

\def\gM{{\mathcal{M}}}
\def\gN{{\mathcal{N}}}

\newcommand{\diffvar}{\mathbf{z}}

\newcommand{\vxi}{\bm{\xi}}

\newcommand{\x}{\mathbf{x}}

\newcommand{\z}{\mathbf{z}}

\newcommand{\TV}{\mathrm{TV}}
\newcommand{\ratio}{r}

\newcommand{\vxit}{\bm{\xi}^\star}

\newcommand{\pt}{p_\text{train}}
\renewcommand{\mid}{\,|\,} %
\newcommand{\diff}{\,\mathrm{d}}

%% file: tables.tex
\newcommand{\tableposterior}{
\begin{table}
\caption{Posterior inference ($\vtheta$). Mean \textcolor{gray}{(standard dev.)} over 5 independent training runs (10 random target priors $\times$ 10 simulated datasets). Significantly best results (Wilcoxon signed-rank test) in bold.} \label{tab:uniform}
\centering
\scalebox{0.67}{
\begin{tabular}{ccccccccccc}
\toprule
\multicolumn{2}{c}{} & \multicolumn{3}{c}{$q_\text{mild}(\vtheta)$} & \multicolumn{3}{c}{$q_\text{strong}(\vtheta)$} & \multicolumn{3}{c}{$q_\text{mixture}(\vtheta)$}\\
\cmidrule(lr){3-11}
    &             & RMSE & C2ST & MMTV & RMSE & C2ST & MMTV & RMSE & C2ST & MMTV \\ \cmidrule(lr){3-11}

\multirow{3}{*}{\shortstack{Two \\ Moons}} 
    & $\text{Simformer}$ 
    & 0.39\textcolor{gray}{(0.19)} & 0.56\textcolor{gray}{(0.05)} & 0.21\textcolor{gray}{(0.10)}
    
    & 0.39\textcolor{gray}{(0.20)} & 0.75\textcolor{gray}{(0.06)} & 0.54\textcolor{gray}{(0.13)}
    
    & 0.33\textcolor{gray}{(0.22)} & 0.68\textcolor{gray}{(0.08)} & 0.42\textcolor{gray}{(0.17)} \\  
    
    & $\text{ACE}$ 
    & \textbf{0.34}\textcolor{gray}{(0.16)} & 0.83\textcolor{gray}{(0.04)} & 0.23\textcolor{gray}{(0.07)} 
    & \textbf{0.09}\textcolor{gray}{(0.03)} & 0.79\textcolor{gray}{(0.05)} & 0.35\textcolor{gray}{(0.12)} 
    & \textbf{0.16}\textcolor{gray}{(0.15)} & 0.80\textcolor{gray}{(0.07)} & 0.34\textcolor{gray}{(0.14)} \\

    & $\text{PriorGuide}$ 
    & 0.37\textcolor{gray}{(0.16)} & \textbf{0.52}\textcolor{gray}{(0.02)} & \textbf{0.11}\textcolor{gray}{(0.04)}
    & \textbf{0.09}\textcolor{gray}{(0.04)} & \textbf{0.52}\textcolor{gray}{(0.02)} & \textbf{0.08}\textcolor{gray}{(0.02)}
    & 0.20\textcolor{gray}{(0.17)} & \textbf{0.55}\textcolor{gray}{(0.05)} & \textbf{0.16}\textcolor{gray}{(0.11)} \\

\midrule

\multirow{3}{*}{OUP} 
    & $\text{Simformer}$ 
    & 0.18\textcolor{gray}{(0.07)} & 0.58\textcolor{gray}{(0.08)} & 0.18\textcolor{gray}{(0.11)}
    
    & 0.24\textcolor{gray}{(0.08)} & 0.73\textcolor{gray}{(0.08)} & 0.37\textcolor{gray}{(0.10)}
    
    & 0.23\textcolor{gray}{(0.09)} & 0.70\textcolor{gray}{(0.08)} & 0.33\textcolor{gray}{(0.11)} \\  
    
    & $\text{ACE}$ 
    & \textbf{0.17}\textcolor{gray}{(0.07)} & 0.58\textcolor{gray}{(0.03)} & 0.11\textcolor{gray}{(0.05)}
    
    & 0.14\textcolor{gray}{(0.04)} & 0.56\textcolor{gray}{(0.03)} & 0.12\textcolor{gray}{(0.05)}
    
    & 0.17\textcolor{gray}{(0.07)} & 0.62\textcolor{gray}{(0.08)} & 0.20\textcolor{gray}{(0.11)} \\

    & $\text{PriorGuide}$ 
    & \textbf{0.17}\textcolor{gray}{(0.06)} & \textbf{0.52}\textcolor{gray}{(0.02)} & \textbf{0.08}\textcolor{gray}{(0.04)}
    
    & \textbf{0.13}\textcolor{gray}{(0.04)} & \textbf{0.51}\textcolor{gray}{(0.02)} & \textbf{0.06}\textcolor{gray}{(0.02)}
    
    & \textbf{0.15}\textcolor{gray}{(0.06)} & \textbf{0.51}\textcolor{gray}{(0.02)} & \textbf{0.07}\textcolor{gray}{(0.04)} \\

\midrule
\multirow{3}{*}{Turin} 
    & $\text{Simformer}$ 
    & 0.22\textcolor{gray}{(0.03)} & 0.80\textcolor{gray}{(0.03)} & 0.31\textcolor{gray}{(0.04)}
    
    & 0.23\textcolor{gray}{(0.03)} & 0.95\textcolor{gray}{(0.02)} & 0.56\textcolor{gray}{(0.07)}
    
    & 0.23\textcolor{gray}{(0.03)} & 0.94\textcolor{gray}{(0.02)} & 0.50\textcolor{gray}{(0.06)} \\  
    
    & $\text{ACE}$ 
    & 0.21\textcolor{gray}{(0.02)} & 0.78\textcolor{gray}{(0.03)} & 0.27\textcolor{gray}{(0.04)}
    
    & 0.17\textcolor{gray}{(0.02)} & 0.92\textcolor{gray}{(0.04)} & 0.47\textcolor{gray}{(0.07)}
    
    & 0.19\textcolor{gray}{(0.03)} & 0.91\textcolor{gray}{(0.03)} & 0.44\textcolor{gray}{(0.06)} \\

    & $\text{PriorGuide}$ 
    & \textbf{0.14}\textcolor{gray}{(0.02)} & \textbf{0.64}\textcolor{gray}{(0.06)} & \textbf{0.13}\textcolor{gray}{(0.03)}
    
    & \textbf{0.06}\textcolor{gray}{(0.01)} & \textbf{0.55}\textcolor{gray}{(0.04)} & \textbf{0.08}\textcolor{gray}{(0.03)}
    
    & \textbf{0.13}\textcolor{gray}{(0.06)} & \textbf{0.62}\textcolor{gray}{(0.08)} & \textbf{0.19}\textcolor{gray}{(0.12)} \\

\midrule
\multirow{3}{*}{\shortstack{Gaussian\\Linear 10D}} 
    & $\text{Simformer}$ 
    & 0.29\textcolor{gray}{(0.02)} & 0.89\textcolor{gray}{(0.02)} & 0.30\textcolor{gray}{(0.03)}
    
    & 0.31\textcolor{gray}{(0.03)} & 1.00\textcolor{gray}{(0.00)} & 0.65\textcolor{gray}{(0.03)}
    
    & 0.30\textcolor{gray}{(0.03)} & 0.99\textcolor{gray}{(0.00)} & 0.61\textcolor{gray}{(0.05)} \\  
    
    & $\text{ACE}$ 
    & \textbf{0.22}\textcolor{gray}{(0.02)} & 0.67\textcolor{gray}{(0.04)} & 0.12\textcolor{gray}{(0.02)} 
    & \textbf{0.10}\textcolor{gray}{(0.01)} & 0.78\textcolor{gray}{(0.06)} & 0.19\textcolor{gray}{(0.04)} 
    & \textbf{0.19}\textcolor{gray}{(0.05)} & 0.96\textcolor{gray}{(0.03)} & 0.40\textcolor{gray}{(0.11)} \\

    & $\text{PriorGuide}$ 
    & 0.31\textcolor{gray}{(0.08)} & \textbf{0.53}\textcolor{gray}{(0.03)} & \textbf{0.05}\textcolor{gray}{(0.01)}
    & 0.23\textcolor{gray}{(0.09)} & \textbf{0.54}\textcolor{gray}{(0.05)} & \textbf{0.06}\textcolor{gray}{(0.02)}
    & 0.26\textcolor{gray}{(0.09)} & \textbf{0.57}\textcolor{gray}{(0.06)} & \textbf{0.12}\textcolor{gray}{(0.06)} \\

\midrule
\multirow{3}{*}{\shortstack{Gaussian\\Linear 20D}} 
    & $\text{Simformer}$ 
    & 0.29\textcolor{gray}{(0.02)} & 0.95\textcolor{gray}{(0.01)} & 0.29\textcolor{gray}{(0.02)}
    & 0.30\textcolor{gray}{(0.02)} & 1.00\textcolor{gray}{(0.00)} & 0.64\textcolor{gray}{(0.02)}
    & 0.30\textcolor{gray}{(0.02)} & 1.00\textcolor{gray}{(0.00)} & 0.63\textcolor{gray}{(0.03)} \\  
    
    & $\text{ACE}$ 
    & \textbf{0.22}\textcolor{gray}{(0.02)} & 0.72\textcolor{gray}{(0.06)} & 0.11\textcolor{gray}{(0.02)}
    & \textbf{0.10}\textcolor{gray}{(0.01)} & 0.82\textcolor{gray}{(0.05)} & 0.16\textcolor{gray}{(0.03)} 
    & \textbf{0.20}\textcolor{gray}{(0.04)} & 0.99\textcolor{gray}{(0.01)} & 0.44\textcolor{gray}{(0.09)} \\  
    
    & $\text{PriorGuide}$ 
    & 0.27\textcolor{gray}{(0.07)} & \textbf{0.54}\textcolor{gray}{(0.03)} & \textbf{0.05}\textcolor{gray}{(0.01)}
    & 0.28\textcolor{gray}{(0.12)} & \textbf{0.58}\textcolor{gray}{(0.03)} & \textbf{0.11}\textcolor{gray}{(0.03)}
    & 0.28\textcolor{gray}{(0.10)} & \textbf{0.59}\textcolor{gray}{(0.05)} & \textbf{0.14}\textcolor{gray}{(0.06)} \\

\midrule
\multirow{3}{*}{BCI} 
    & $\text{Simformer}$ 
    & 0.79\textcolor{gray}{(0.13)} & 0.86\textcolor{gray}{(0.05)} & 0.35\textcolor{gray}{(0.09)}
    
    & 1.03\textcolor{gray}{(0.16)} & 0.98\textcolor{gray}{(0.01)} & 0.61\textcolor{gray}{(0.09)}
    
    & 0.99\textcolor{gray}{(0.24)} & 0.98\textcolor{gray}{(0.02)} & 0.63\textcolor{gray}{(0.09)} \\  
    
    & $\text{ACE}$ 
    & \textbf{0.55}\textcolor{gray}{(0.12)} & \textbf{0.84}\textcolor{gray}{(0.08)} & \textbf{0.28}\textcolor{gray}{(0.10)}
    
    & 0.31\textcolor{gray}{(0.13)}  & 0.82\textcolor{gray}{(0.09)} & 0.29\textcolor{gray}{(0.12)}
    
    & 1.02\textcolor{gray}{(0.37)} & 0.97\textcolor{gray}{(0.02)} & 0.56\textcolor{gray}{(0.11)} \\

    & $\text{PriorGuide}$ 
    
    & 0.56\textcolor{gray}{(0.10)} & 0.88\textcolor{gray}{(0.06)} & 0.41\textcolor{gray}{(0.08)}
    
    & \textbf{0.25}\textcolor{gray}{(0.04)} & \textbf{0.72}\textcolor{gray}{(0.10)} & \textbf{0.21}\textcolor{gray}{(0.12)}
    
    & \textbf{0.87}\textcolor{gray}{(0.68)} & \textbf{0.78}\textcolor{gray}{(0.13)} & \textbf{0.37}\textcolor{gray}{(0.27)} \\
\bottomrule
\end{tabular}}
\end{table}
}

\newcommand{\tablepostpred}{
\begin{table}[hbt]
\caption{Data prediction ($\vx$). Mean \textcolor{gray}{(standard dev.)} over 5 independent training runs (10 random target priors $\times$ 10 simulated datasets). Significantly best results (Wilcoxon signed-rank test) in bold. } \label{tab:postpred}
\centering
\setlength{\tabcolsep}{10pt}
\scalebox{0.8}{
\begin{tabular}{cccccccc}
\toprule
\multicolumn{2}{c}{} & \multicolumn{2}{c}{$q_\text{mild}(\vtheta)$} & \multicolumn{2}{c}{$q_\text{strong}(\vtheta)$} & \multicolumn{2}{c}{$q_\text{mixture}(\vtheta)$}\\
\cmidrule(lr){3-8}
    &             & RMSE & $\text{MMD}_x$  & RMSE & $\text{MMD}_x$ & RMSE & $\text{MMD}_x$ \\ \cmidrule(lr){3-8}
\multirow{3}{*}{OUP} 
    & $\text{Simformer}$ 
    & 0.32\textcolor{gray}{(0.10)} & 0.44\textcolor{gray}{(0.24)}
    & 0.39\textcolor{gray}{(0.11)} & 0.54\textcolor{gray}{(0.24)}
    & 0.34\textcolor{gray}{(0.11)} & 0.47\textcolor{gray}{(0.25)} \\  
    & $\text{ACE}$ 
    & \textbf{0.26}\textcolor{gray}{(0.08)} & \textbf{0.44}\textcolor{gray}{(0.31)}
    & 0.22\textcolor{gray}{(0.06)} & 0.30\textcolor{gray}{(0.20)}
    & \textbf{0.24}\textcolor{gray}{(0.10)} & 0.38\textcolor{gray}{(0.31)} \\  
    & $\text{PriorGuide}$ 
    & 0.28\textcolor{gray}{(0.09)} & 0.45\textcolor{gray}{(0.28)}
    & \textbf{0.21}\textcolor{gray}{(0.05)} & \textbf{0.29}\textcolor{gray}{(0.17)}
    & 0.25\textcolor{gray}{(0.11)} & \textbf{0.34}\textcolor{gray}{(0.22)} \\

\midrule
\multirow{3}{*}{Turin} 
    & $\text{Simformer}$ 
    & 0.14\textcolor{gray}{(0.01)} & 0.49\textcolor{gray}{(0.09)} 
    & 0.14\textcolor{gray}{(0.01)} & 0.49\textcolor{gray}{(0.09)}
    & 0.14\textcolor{gray}{(0.01)} & 0.48\textcolor{gray}{(0.09)}\\  
    & $\text{ACE}$ 
    & 0.16\textcolor{gray}{(0.03)} & 0.62\textcolor{gray}{(0.19)}
    & 0.16\textcolor{gray}{(0.03)} & 0.61\textcolor{gray}{(0.20)} 
    & 0.16\textcolor{gray}{(0.03)} & 0.61\textcolor{gray}{(0.19)} \\
    & $\text{PriorGuide}$ 
    & \textbf{0.13}\textcolor{gray}{(0.01)} & \textbf{0.47}\textcolor{gray}{(0.08)} 
    & \textbf{0.13}\textcolor{gray}{(0.01)} & \textbf{0.46}\textcolor{gray}{(0.07)} 
    & \textbf{0.13}\textcolor{gray}{(0.01)} & \textbf{0.46}\textcolor{gray}{(0.09)}\\
\bottomrule
\end{tabular}}
\end{table}
}

%% file: fig/fig1/2.theta_posterior_topmid.tex
\begin{tikzpicture}

\definecolor{darkcyan0114178}{RGB}{0,114,178}
\definecolor{darkgrey176}{RGB}{176,176,176}
\definecolor{orange2301590}{RGB}{230,159,0}

\begin{axis}[
legend cell align={left},
legend style={fill opacity=0.8, draw opacity=1, text opacity=1, draw=none, fill=none},
tick align=outside,
tick pos=left,
x grid style={darkgrey176},
xlabel={\(\displaystyle \theta_1\)},
xmin=-0.1, xmax=2.1,
xtick style={color=black},
y grid style={darkgrey176},
ymin=-2.1, ymax=2.1,
ytick style={color=black}
]
\addplot [draw=black, fill=black, mark=*, only marks, mark size=1.2pt, forget plot]
table{%
x  y
0.81 0.55
};
\path [draw=orange2301590, semithick]
(axis cs:0.625869414334329,-2.1)
--(axis cs:0.560334945193953,-1.97272727272727)
--(axis cs:0.534338915838479,-1.29393939393939)
--(axis cs:0.646668715917796,0.0636363636363635)
--(axis cs:0.788888888888889,0.340099620701139)
--(axis cs:0.822668382525007,0.318181818181818)
--(axis cs:0.86075234630376,0.106060606060606)
--(axis cs:0.725178135383592,-1.97272727272727)
--(axis cs:0.683415180811562,-2.1)
--(axis cs:0.683415180811562,-2.1);
\path [draw=orange2301590, semithick]
(axis cs:0.585920146941827,-1.76060606060606)
--(axis cs:0.582706304223509,-1.71818181818182)
--(axis cs:0.580130223930422,-1.67575757575758)
--(axis cs:0.578051494822286,-1.63333333333333)
--(axis cs:0.576483660988045,-1.59090909090909)
--(axis cs:0.575412553048493,-1.54848484848485)
--(axis cs:0.57474274986097,-1.50606060606061)
--(axis cs:0.574369003472008,-1.46363636363636)
--(axis cs:0.574253506131048,-1.42121212121212)
--(axis cs:0.574430027613369,-1.37878787878788)
--(axis cs:0.57493551375095,-1.33636363636364)
--(axis cs:0.575736221139219,-1.29393939393939)
--(axis cs:0.576732670574751,-1.25151515151515)
--(axis cs:0.577840019473806,-1.20909090909091)
--(axis cs:0.579053231641428,-1.16666666666667)
--(axis cs:0.580440861865673,-1.12424242424242)
--(axis cs:0.582089495725508,-1.08181818181818)
--(axis cs:0.584041603442785,-1.03939393939394)
--(axis cs:0.586251320998346,-0.996969696969697)
--(axis cs:0.58856761181927,-0.954545454545455)
--(axis cs:0.588888888888889,-0.948111544422807)
--(axis cs:0.590145118676492,-0.912121212121212)
--(axis cs:0.591483720245624,-0.86969696969697)
--(axis cs:0.59276905201406,-0.827272727272727)
--(axis cs:0.594177876898193,-0.784848484848485)
--(axis cs:0.595897795753528,-0.742424242424242)
--(axis cs:0.598013546918844,-0.7)
--(axis cs:0.60052106393883,-0.657575757575757)
--(axis cs:0.603455479506211,-0.615151515151515)
--(axis cs:0.60697974873898,-0.572727272727273)
--(axis cs:0.611111111111111,-0.532435270772069)
--(axis cs:0.611276354989734,-0.53030303030303)
--(axis cs:0.614742817407943,-0.487878787878788)
--(axis cs:0.61854827998128,-0.445454545454546)
--(axis cs:0.622807181008404,-0.403030303030303)
--(axis cs:0.627583297463996,-0.360606060606061)
--(axis cs:0.632838199710937,-0.318181818181818)
--(axis cs:0.633333333333333,-0.314146316231641)
--(axis cs:0.63859518885699,-0.275757575757576)
--(axis cs:0.644713172727211,-0.233333333333333)
--(axis cs:0.651128090283364,-0.190909090909091)
--(axis cs:0.655555555555556,-0.162171364415186)
--(axis cs:0.657649914755767,-0.148484848484848)
--(axis cs:0.665533334480436,-0.106060606060606)
--(axis cs:0.675083006210215,-0.0636363636363635)
--(axis cs:0.677777777777778,-0.0518183353545523)
--(axis cs:0.683291250493931,-0.021212121212121)
--(axis cs:0.692795599448583,0.0212121212121215)
--(axis cs:0.7,0.0477381522270477)
--(axis cs:0.705997472310216,0.0636363636363635)
--(axis cs:0.722222222222222,0.103967921556364)
--(axis cs:0.744444444444444,0.0919889887431257)
--(axis cs:0.766666666666667,0.0844061052257603)
--(axis cs:0.775112302425622,0.0636363636363635)
--(axis cs:0.786229891034555,0.0212121212121215)
--(axis cs:0.788888888888889,0.00519977432425281)
--(axis cs:0.792189836827724,-0.021212121212121)
--(axis cs:0.795456328864548,-0.0636363636363635)
--(axis cs:0.797239944773841,-0.106060606060606)
--(axis cs:0.798058612572626,-0.148484848484848)
--(axis cs:0.798133715848294,-0.190909090909091)
--(axis cs:0.797537784068136,-0.233333333333333)
--(axis cs:0.796326042308996,-0.275757575757576)
--(axis cs:0.794619294149253,-0.318181818181818)
--(axis cs:0.792579857853483,-0.360606060606061)
--(axis cs:0.790346851870078,-0.403030303030303)
--(axis cs:0.788888888888889,-0.429751733763755)
--(axis cs:0.788162420796667,-0.445454545454546)
--(axis cs:0.786252826315933,-0.487878787878788)
--(axis cs:0.784425230468152,-0.53030303030303)
--(axis cs:0.782563555172877,-0.572727272727273)
--(axis cs:0.780537015365542,-0.615151515151515)
--(axis cs:0.778300361565716,-0.657575757575757)
--(axis cs:0.775870833722574,-0.7)
--(axis cs:0.7732305088612,-0.742424242424242)
--(axis cs:0.770276633234075,-0.784848484848485)
--(axis cs:0.766842076162502,-0.827272727272727)
--(axis cs:0.766666666666667,-0.829242471038288)
--(axis cs:0.763950075875802,-0.86969696969697)
--(axis cs:0.760951310790681,-0.912121212121212)
--(axis cs:0.757647262697713,-0.954545454545455)
--(axis cs:0.753856611641574,-0.996969696969697)
--(axis cs:0.749469855013656,-1.03939393939394)
--(axis cs:0.74454567886366,-1.08181818181818)
--(axis cs:0.744444444444444,-1.08275405144427)
--(axis cs:0.741813062003891,-1.12424242424242)
--(axis cs:0.739325300122722,-1.16666666666667)
--(axis cs:0.736951245126534,-1.20909090909091)
--(axis cs:0.734748636740917,-1.25151515151515)
--(axis cs:0.733085413442901,-1.29393939393939)
--(axis cs:0.732564992111722,-1.33636363636364)
--(axis cs:0.73348013107853,-1.37878787878788)
--(axis cs:0.735389241028936,-1.42121212121212)
--(axis cs:0.737426777834401,-1.46363636363636)
--(axis cs:0.738849184216245,-1.50606060606061)
--(axis cs:0.739303568158698,-1.54848484848485)
--(axis cs:0.738765951291764,-1.59090909090909)
--(axis cs:0.737269255568615,-1.63333333333333)
--(axis cs:0.734659829852092,-1.67575757575758)
--(axis cs:0.73054955038274,-1.71818181818182)
--(axis cs:0.724473412339199,-1.76060606060606)
--(axis cs:0.722222222222222,-1.77283428858496)
--(axis cs:0.714345925828892,-1.8030303030303)
--(axis cs:0.70233720341735,-1.84545454545455)
--(axis cs:0.7,-1.85345698541922)
--(axis cs:0.683585681491893,-1.88787878787879)
--(axis cs:0.677777777777778,-1.89963978231858)
--(axis cs:0.655555555555556,-1.92033528561676)
--(axis cs:0.633333333333333,-1.92639604729317)
--(axis cs:0.614345610033035,-1.88787878787879)
--(axis cs:0.611111111111111,-1.87950941118157)
--(axis cs:0.60049557707139,-1.84545454545455)
--(axis cs:0.591014238765656,-1.8030303030303)
--(axis cs:0.588888888888889,-1.79054844184368)
--cycle;
\path [draw=orange2301590, semithick]
(axis cs:0.608309360813943,-1.50606060606061)
--(axis cs:0.604659530044991,-1.46363636363636)
--(axis cs:0.602720338440875,-1.42121212121212)
--(axis cs:0.601903002122128,-1.37878787878788)
--(axis cs:0.60178439925664,-1.33636363636364)
--(axis cs:0.602104950820444,-1.29393939393939)
--(axis cs:0.602727620523377,-1.25151515151515)
--(axis cs:0.603589259984051,-1.20909090909091)
--(axis cs:0.604655155845665,-1.16666666666667)
--(axis cs:0.605908571046108,-1.12424242424242)
--(axis cs:0.607379213728,-1.08181818181818)
--(axis cs:0.609148313629335,-1.03939393939394)
--(axis cs:0.611111111111111,-1.00030592500283)
--(axis cs:0.611275043501835,-0.996969696969697)
--(axis cs:0.613535431526848,-0.954545454545455)
--(axis cs:0.61568061074343,-0.912121212121212)
--(axis cs:0.617471457658185,-0.86969696969697)
--(axis cs:0.618823899722822,-0.827272727272727)
--(axis cs:0.61984989384697,-0.784848484848485)
--(axis cs:0.620809267842314,-0.742424242424242)
--(axis cs:0.622002170038194,-0.7)
--(axis cs:0.623680365276985,-0.657575757575757)
--(axis cs:0.626010646047619,-0.615151515151515)
--(axis cs:0.629089121903466,-0.572727272727273)
--(axis cs:0.633004475892726,-0.53030303030303)
--(axis cs:0.633333333333333,-0.527279960175737)
--(axis cs:0.639450031564331,-0.487878787878788)
--(axis cs:0.646522074172254,-0.445454545454546)
--(axis cs:0.653841801386793,-0.403030303030303)
--(axis cs:0.655555555555556,-0.393146256268039)
--(axis cs:0.660900154315185,-0.360606060606061)
--(axis cs:0.667357451823697,-0.318181818181818)
--(axis cs:0.673715643494862,-0.275757575757576)
--(axis cs:0.677777777777778,-0.251061733015816)
--(axis cs:0.680999779190273,-0.233333333333333)
--(axis cs:0.688472457394603,-0.190909090909091)
--(axis cs:0.695768621121945,-0.148484848484848)
--(axis cs:0.7,-0.125576810497391)
--(axis cs:0.710734801650525,-0.106060606060606)
--(axis cs:0.722222222222222,-0.0827837819799268)
--(axis cs:0.744444444444444,-0.0833048304100428)
--(axis cs:0.752580919396909,-0.106060606060606)
--(axis cs:0.760123410467714,-0.148484848484848)
--(axis cs:0.763239307264286,-0.190909090909091)
--(axis cs:0.764371405736238,-0.233333333333333)
--(axis cs:0.764386517624432,-0.275757575757576)
--(axis cs:0.76367308134835,-0.318181818181818)
--(axis cs:0.76255052225248,-0.360606060606061)
--(axis cs:0.761318475183682,-0.403030303030303)
--(axis cs:0.760168545230651,-0.445454545454546)
--(axis cs:0.75915273459931,-0.487878787878788)
--(axis cs:0.758218261263,-0.53030303030303)
--(axis cs:0.757241844908548,-0.572727272727273)
--(axis cs:0.756051921433458,-0.615151515151515)
--(axis cs:0.754468117938005,-0.657575757575757)
--(axis cs:0.752382244548008,-0.7)
--(axis cs:0.749847419988073,-0.742424242424242)
--(axis cs:0.747054004752446,-0.784848484848485)
--(axis cs:0.744444444444444,-0.822870310898931)
--(axis cs:0.744203224092843,-0.827272727272727)
--(axis cs:0.741750720784359,-0.86969696969697)
--(axis cs:0.739048505107779,-0.912121212121212)
--(axis cs:0.736111202321405,-0.954545454545455)
--(axis cs:0.733034875872482,-0.996969696969697)
--(axis cs:0.729864482804818,-1.03939393939394)
--(axis cs:0.726497527046796,-1.08181818181818)
--(axis cs:0.722675000886357,-1.12424242424242)
--(axis cs:0.722222222222222,-1.12880355376296)
--(axis cs:0.718250158868716,-1.16666666666667)
--(axis cs:0.714148086143587,-1.20909090909091)
--(axis cs:0.710274986961313,-1.25151515151515)
--(axis cs:0.706324747074591,-1.29393939393939)
--(axis cs:0.701916444092225,-1.33636363636364)
--(axis cs:0.7,-1.35442236091821)
--(axis cs:0.695664659193305,-1.37878787878788)
--(axis cs:0.689880462029695,-1.42121212121212)
--(axis cs:0.685189864466006,-1.46363636363636)
--(axis cs:0.680605018763735,-1.50606060606061)
--(axis cs:0.677777777777778,-1.52894154669631)
--(axis cs:0.675999439866168,-1.54848484848485)
--(axis cs:0.67223587066582,-1.59090909090909)
--(axis cs:0.667522899788843,-1.63333333333333)
--(axis cs:0.660184884569574,-1.67575757575758)
--(axis cs:0.655555555555556,-1.69345168733715)
--(axis cs:0.643182502211033,-1.67575757575758)
--(axis cs:0.633333333333333,-1.66035983473441)
--(axis cs:0.6292312554018,-1.63333333333333)
--(axis cs:0.622713545734575,-1.59090909090909)
--(axis cs:0.615120931671568,-1.54848484848485)
--(axis cs:0.611111111111111,-1.52817483510237)
--cycle;
\path [draw=orange2301590, semithick]
(axis cs:0.631545071408116,-1.29393939393939)
--(axis cs:0.630048662619421,-1.25151515151515)
--(axis cs:0.629448058734644,-1.20909090909091)
--(axis cs:0.629468651633293,-1.16666666666667)
--(axis cs:0.630010870880041,-1.12424242424242)
--(axis cs:0.631064523573992,-1.08181818181818)
--(axis cs:0.632620034471518,-1.03939393939394)
--(axis cs:0.633333333333333,-1.02407687332496)
--(axis cs:0.635352028424035,-0.996969696969697)
--(axis cs:0.638200801471414,-0.954545454545455)
--(axis cs:0.640442385756381,-0.912121212121212)
--(axis cs:0.642063303779884,-0.86969696969697)
--(axis cs:0.643161239797901,-0.827272727272727)
--(axis cs:0.64397892410488,-0.784848484848485)
--(axis cs:0.644895393837352,-0.742424242424242)
--(axis cs:0.64635678541385,-0.7)
--(axis cs:0.648770797949048,-0.657575757575757)
--(axis cs:0.652413659903505,-0.615151515151515)
--(axis cs:0.655555555555556,-0.588032148496516)
--(axis cs:0.657729008602536,-0.572727272727273)
--(axis cs:0.664362475763449,-0.53030303030303)
--(axis cs:0.67124853370178,-0.487878787878788)
--(axis cs:0.677777777777778,-0.448731244414393)
--(axis cs:0.678639505565608,-0.445454545454546)
--(axis cs:0.688357670606404,-0.403030303030303)
--(axis cs:0.696234525265103,-0.360606060606061)
--(axis cs:0.7,-0.338707416108918)
--(axis cs:0.710793550286797,-0.318181818181818)
--(axis cs:0.722222222222222,-0.287733039330773)
--(axis cs:0.726652375520494,-0.318181818181818)
--(axis cs:0.728997207471678,-0.360606060606061)
--(axis cs:0.7299265584304,-0.403030303030303)
--(axis cs:0.730600798330306,-0.445454545454546)
--(axis cs:0.731501780866287,-0.487878787878788)
--(axis cs:0.732586193833258,-0.53030303030303)
--(axis cs:0.733442566009025,-0.572727272727273)
--(axis cs:0.733682504120997,-0.615151515151515)
--(axis cs:0.733201317228184,-0.657575757575757)
--(axis cs:0.732104539035918,-0.7)
--(axis cs:0.730530377393791,-0.742424242424242)
--(axis cs:0.728568667747098,-0.784848484848485)
--(axis cs:0.726276073173873,-0.827272727272727)
--(axis cs:0.72372230163578,-0.86969696969697)
--(axis cs:0.722222222222222,-0.893166578859103)
--(axis cs:0.72046190685477,-0.912121212121212)
--(axis cs:0.716722998154873,-0.954545454545455)
--(axis cs:0.713058313207631,-0.996969696969697)
--(axis cs:0.709290510114846,-1.03939393939394)
--(axis cs:0.705462789801047,-1.08181818181818)
--(axis cs:0.701714784236291,-1.12424242424242)
--(axis cs:0.7,-1.1435619505039)
--(axis cs:0.694139581503833,-1.16666666666667)
--(axis cs:0.684139760084114,-1.20909090909091)
--(axis cs:0.677777777777778,-1.2375954144206)
--(axis cs:0.67450493923599,-1.25151515151515)
--(axis cs:0.665764585259332,-1.29393939393939)
--(axis cs:0.658217850766928,-1.33636363636364)
--(axis cs:0.655555555555556,-1.35160716979691)
--(axis cs:0.640872653572868,-1.33636363636364)
--(axis cs:0.633333333333333,-1.32119150271449)
--cycle;
\path [draw=orange2301590, semithick]
;

\path [draw=darkcyan0114178, semithick]
(axis cs:0.718487391678489,0.318181818181818)
--(axis cs:0.719477547985156,0.360606060606061)
--(axis cs:0.722222222222222,0.378203690760396)
--(axis cs:0.72471872260582,0.403030303030303)
--(axis cs:0.725361025953435,0.445454545454545)
--(axis cs:0.725536280303161,0.487878787878788)
--(axis cs:0.733325202722478,0.53030303030303)
--(axis cs:0.744444444444444,0.564978411065599)
--(axis cs:0.746691506862379,0.572727272727273)
--(axis cs:0.75721969918055,0.615151515151515)
--(axis cs:0.766666666666667,0.6416846369076)
--(axis cs:0.772367399304149,0.657575757575758)
--(axis cs:0.788888888888889,0.694040127494965)
--(axis cs:0.79438134496865,0.7)
--(axis cs:0.811111111111111,0.721728062894771)
--(axis cs:0.833333333333333,0.740962982803631)
--(axis cs:0.835215541601477,0.742424242424243)
--(axis cs:0.855555555555556,0.763924976383923)
--(axis cs:0.877777777777778,0.762967833925861)
--(axis cs:0.9,0.764616981962694)
--(axis cs:0.922222222222222,0.764048973610625)
--(axis cs:0.943479810986748,0.742424242424243)
--(axis cs:0.944444444444445,0.741488535454159)
--(axis cs:0.966666666666667,0.725336738173615)
--(axis cs:0.986581373588284,0.7)
--(axis cs:0.988888888888889,0.697473013379952)
--(axis cs:1.00671654273914,0.657575757575758)
--(axis cs:1.01111111111111,0.62752061702192)
--(axis cs:1.01426451435802,0.615151515151515)
--(axis cs:1.01923647087242,0.572727272727273)
--(axis cs:1.01413375052025,0.53030303030303)
--(axis cs:1.01111111111111,0.518585577621393)
--(axis cs:1.00670769368534,0.487878787878788)
--(axis cs:0.998716720583337,0.445454545454545)
--(axis cs:0.990201121892287,0.403030303030303)
--(axis cs:0.988888888888889,0.397769142000603)
--(axis cs:0.981433606112261,0.360606060606061)
--(axis cs:0.966666666666667,0.327193716317387)
--(axis cs:0.963195143513842,0.318181818181818)
--(axis cs:0.944444444444445,0.276793289063935)
--(axis cs:0.943937569543515,0.275757575757576)
--(axis cs:0.922994264664615,0.233333333333333)
--(axis cs:0.922222222222222,0.23176637938882)
--(axis cs:0.907205451689795,0.190909090909091)
--(axis cs:0.9,0.171550997017954)
--(axis cs:0.888394054321892,0.148484848484848)
--(axis cs:0.877777777777778,0.13063446238573)
--(axis cs:0.855572401735415,0.106060606060606)
--(axis cs:0.855555555555556,0.106040188982177)
--(axis cs:0.833333333333333,0.0820243195103504)
--(axis cs:0.811111111111111,0.0739496293068652)
--(axis cs:0.788888888888889,0.0740014627553147)
--(axis cs:0.766666666666667,0.104070934021611)
--(axis cs:0.765282735768653,0.106060606060606)
--(axis cs:0.744444444444444,0.141470196259036)
--(axis cs:0.742489605492234,0.148484848484848)
--(axis cs:0.732646314700367,0.190909090909091)
--(axis cs:0.727158423222257,0.233333333333333)
--(axis cs:0.723503820525003,0.275757575757576)
--(axis cs:0.722222222222222,0.288527937812036)
--cycle;
\path [draw=darkcyan0114178, semithick]
(axis cs:0.765817516189759,0.318181818181818)
--(axis cs:0.761796271224613,0.360606060606061)
--(axis cs:0.765016236321208,0.403030303030303)
--(axis cs:0.766666666666667,0.419735512939518)
--(axis cs:0.768356261996105,0.445454545454545)
--(axis cs:0.771378792769848,0.487878787878788)
--(axis cs:0.778605729487967,0.53030303030303)
--(axis cs:0.788888888888889,0.567738548313049)
--(axis cs:0.790180529530047,0.572727272727273)
--(axis cs:0.807697751357323,0.615151515151515)
--(axis cs:0.811111111111111,0.621760241676826)
--(axis cs:0.833333333333333,0.648887832963121)
--(axis cs:0.840557319857401,0.657575757575758)
--(axis cs:0.855555555555556,0.673987470678083)
--(axis cs:0.877777777777778,0.685522647835522)
--(axis cs:0.9,0.69061559731543)
--(axis cs:0.922222222222222,0.658905368717918)
--(axis cs:0.923421613792095,0.657575757575758)
--(axis cs:0.944444444444445,0.6382042867964)
--(axis cs:0.958849530975064,0.615151515151515)
--(axis cs:0.966132366333185,0.572727272727273)
--(axis cs:0.964537269754759,0.53030303030303)
--(axis cs:0.960440295947339,0.487878787878788)
--(axis cs:0.953189285724061,0.445454545454545)
--(axis cs:0.944444444444445,0.414015187588715)
--(axis cs:0.942257846593193,0.403030303030303)
--(axis cs:0.93259217060569,0.360606060606061)
--(axis cs:0.922222222222222,0.33492411174739)
--(axis cs:0.914432455292119,0.318181818181818)
--(axis cs:0.9,0.292319252684595)
--(axis cs:0.885753696278176,0.275757575757576)
--(axis cs:0.877777777777778,0.260711896916271)
--(axis cs:0.862075685253649,0.233333333333333)
--(axis cs:0.855555555555556,0.221917530253632)
--(axis cs:0.833333333333333,0.193005791886303)
--(axis cs:0.811111111111111,0.200041711003196)
--(axis cs:0.790623636288325,0.233333333333333)
--(axis cs:0.788888888888889,0.2370587796855)
--(axis cs:0.775431210968285,0.275757575757576)
--(axis cs:0.766666666666667,0.31402707876553)
--cycle;
\path [draw=darkcyan0114178, semithick]
(axis cs:0.787482840032633,0.403030303030303)
--(axis cs:0.788888888888889,0.434592223958254)
--(axis cs:0.789402200584692,0.445454545454545)
--(axis cs:0.79473344743686,0.487878787878788)
--(axis cs:0.805493300865586,0.53030303030303)
--(axis cs:0.811111111111111,0.548518857789573)
--(axis cs:0.822392743825649,0.572727272727273)
--(axis cs:0.833333333333333,0.588504627064397)
--(axis cs:0.85175623169326,0.615151515151515)
--(axis cs:0.855555555555556,0.620382073184939)
--(axis cs:0.877777777777778,0.631702990131217)
--(axis cs:0.9,0.618948627452232)
--(axis cs:0.902235101623055,0.615151515151515)
--(axis cs:0.922222222222222,0.581131906890479)
--(axis cs:0.926965400906546,0.572727272727273)
--(axis cs:0.933833226157816,0.53030303030303)
--(axis cs:0.933236272830774,0.487878787878788)
--(axis cs:0.928504171488999,0.445454545454545)
--(axis cs:0.922222222222222,0.41698006301193)
--(axis cs:0.919068920776412,0.403030303030303)
--(axis cs:0.905289990064785,0.360606060606061)
--(axis cs:0.9,0.349643331544438)
--(axis cs:0.877777777777778,0.325899964145063)
--(axis cs:0.871320506011337,0.318181818181818)
--(axis cs:0.855555555555556,0.295142824986461)
--(axis cs:0.834706029547144,0.275757575757576)
--(axis cs:0.833333333333333,0.274327287333314)
--(axis cs:0.830315332018093,0.275757575757576)
--(axis cs:0.811111111111111,0.294465533491485)
--(axis cs:0.805079171597684,0.318181818181818)
--(axis cs:0.794056180953343,0.360606060606061)
--(axis cs:0.788888888888889,0.389301602239109)
--cycle;
\path [draw=darkcyan0114178, semithick]
(axis cs:0.830677654704504,0.360606060606061)
--(axis cs:0.817718764175795,0.403030303030303)
--(axis cs:0.812988673908653,0.445454545454545)
--(axis cs:0.818780515190536,0.487878787878788)
--(axis cs:0.833004135933282,0.53030303030303)
--(axis cs:0.833333333333333,0.530989665487941)
--(axis cs:0.855555555555556,0.567740572033016)
--(axis cs:0.877777777777778,0.569481120615975)
--(axis cs:0.898346057885991,0.53030303030303)
--(axis cs:0.9,0.523062560293182)
--(axis cs:0.90513914349047,0.487878787878788)
--(axis cs:0.903405102480512,0.445454545454545)
--(axis cs:0.9,0.425545130606297)
--(axis cs:0.895019850009208,0.403030303030303)
--(axis cs:0.877777777777778,0.373812136493465)
--(axis cs:0.866882435782541,0.360606060606061)
--(axis cs:0.855555555555556,0.348314829106229)
--(axis cs:0.833333333333333,0.351808057351322)
--cycle;
\path [draw=darkcyan0114178, semithick]
;
\addplot [semithick, orange2301590]
table {%
-100 100
};
\addlegendentry{Training prior}

\addplot [semithick, darkcyan0114178]
table {%
-100 100
};
\addlegendentry{New prior}

\end{axis}

\end{tikzpicture}

%% file: fig/fig1/3.theta_amortized_topright.tex
\begin{tikzpicture}

\definecolor{darkgrey176}{RGB}{176,176,176}
\definecolor{gold25323137}{RGB}{253,231,37}
\definecolor{indigo68184}{RGB}{68,1,84}
\definecolor{palevioletred204121167}{RGB}{204,121,167}

\begin{axis}[
legend cell align={left},
legend style={
  fill opacity=0.8,
  draw opacity=1,
  text opacity=1,
  draw=none,
  fill=none,
  mark options={mark size=0.5}
},
tick align=outside,
tick pos=left,
x grid style={darkgrey176},
xmin=-0.1, xmax=2.1,
xtick style={color=black},
y grid style={darkgrey176},
ymin=-2.1, ymax=2.1,
ytick style={color=black}
]
\addplot [
  draw=palevioletred204121167,
  fill=palevioletred204121167,
  mark size=0.5pt,
  mark=*,
  only marks,
  opacity=0.4,
  forget plot
]
table{%
x  y
0.71270489692688 -1.3552827835083
0.633641362190247 -1.1477986574173
0.568736433982849 -1.16169774532318
0.712988674640656 -1.31856906414032
0.68178117275238 -0.976359128952026
0.651987910270691 -0.683544278144836
0.680538773536682 -1.20012390613556
0.638742506504059 -0.688362896442413
0.699142098426819 -1.56100976467133
0.638157486915588 -1.67285823822021
0.708730041980743 -0.229848131537437
0.590646386146545 -1.35516476631165
0.635264277458191 -1.15673208236694
0.627842962741852 0.197411671280861
0.64440381526947 -0.598259747028351
0.678569674491882 -0.551928162574768
0.680497169494629 -0.398095786571503
0.689876794815063 -0.26226818561554
0.644915103912354 -1.97706305980682
0.665931701660156 -0.157030999660492
0.675902485847473 -0.644343614578247
0.721286654472351 -1.22347009181976
0.60634046792984 -0.696545481681824
0.628849148750305 -1.86968791484833
0.583377003669739 -1.66210639476776
0.715587615966797 -0.296412497758865
0.611458778381348 -1.52332425117493
0.745936751365662 0.246631070971489
0.577731490135193 -1.25443959236145
0.721129536628723 -0.455012798309326
0.783988237380981 -0.41842332482338
0.567031264305115 -0.98400741815567
0.773796737194061 -1.05200123786926
0.646220445632935 -1.1860773563385
0.814661622047424 -0.200801253318787
0.581303358078003 -0.753917753696442
0.730434656143188 -0.696227729320526
0.835870504379272 -0.807157576084137
0.725294649600983 -0.804486751556396
0.69959819316864 -1.63872706890106
0.639619708061218 -1.06593930721283
0.76877760887146 -0.274426966905594
0.574584066867828 -1.94276177883148
0.644214510917664 -0.581098139286041
0.670476734638214 -0.775755167007446
0.736147284507751 -0.24959272146225
0.673606157302856 -0.662505447864532
0.693182945251465 -1.26211845874786
0.662180364131927 -1.53406596183777
0.570161402225494 -1.38142991065979
0.627697587013245 -0.528429985046387
0.639120101928711 -0.600867211818695
0.665782332420349 -0.987963378429413
0.628117203712463 -1.1661022901535
0.589538931846619 -0.685224652290344
0.702063024044037 -1.45923507213593
0.612248301506042 -1.23500645160675
0.840646207332611 0.259858161211014
0.682891607284546 -1.65974199771881
0.680546879768372 -0.371955722570419
0.720149517059326 -0.564428687095642
0.583131730556488 -1.27840423583984
0.615476489067078 -1.25935459136963
0.599697828292847 -1.12505698204041
0.715044260025024 -1.3038432598114
0.642392933368683 -1.51915919780731
0.699656665325165 -0.131200894713402
0.671167016029358 -1.79054498672485
0.61334216594696 -0.876755893230438
0.622345864772797 -0.481715619564056
0.693850636482239 -0.899914085865021
0.706446647644043 -1.72808635234833
0.719734072685242 -0.811404407024384
0.578217804431915 -1.21204793453217
0.666824698448181 -1.44164705276489
0.670698404312134 -1.34575843811035
0.615795731544495 0.036951944231987
0.593874871730804 -0.643764913082123
0.659599184989929 -0.684252738952637
0.63265722990036 -1.13659358024597
0.554808378219604 -1.51737058162689
0.663563966751099 -1.13697731494904
0.634235620498657 -0.424676358699799
0.788359582424164 -0.152972862124443
0.671224772930145 -1.02226424217224
0.618219256401062 -1.36642575263977
0.572964727878571 -1.67462539672852
0.671889305114746 -0.71858537197113
0.554846882820129 -1.41698527336121
0.602430701255798 -1.81166803836823
0.697737634181976 -0.459250509738922
0.530772924423218 -1.37482583522797
0.737456023693085 -0.855405330657959
0.838278353214264 -1.84586644172668
0.770511031150818 0.0568163581192493
0.640547811985016 -0.680834650993347
0.665526986122131 -1.4426611661911
0.633205771446228 -1.57264292240143
0.576600551605225 -1.88413488864899
0.673145294189453 -1.56233990192413
0.690494298934937 -1.19830620288849
0.665917873382568 -1.60519850254059
0.719566106796265 -0.666868150234222
0.812743425369263 -1.97837674617767
0.735161423683167 -0.0235392060130835
0.661212086677551 -1.15921235084534
0.712634205818176 -0.506908297538757
0.801856219768524 -0.295207053422928
0.769285202026367 -0.211042791604996
0.665903091430664 -1.06059670448303
0.673154592514038 -1.85308361053467
0.62192165851593 -0.662696361541748
0.681608200073242 -0.461557120084763
0.747479677200317 -0.407853364944458
0.744428038597107 -0.639550030231476
0.550187110900879 -1.96924865245819
0.547412693500519 -1.9101425409317
0.68925142288208 -0.6902996301651
0.710259139537811 0.113368988037109
0.614616870880127 -1.09808158874512
0.662604033946991 -0.512785911560059
0.817621886730194 -0.131083637475967
0.649156749248505 -1.3508061170578
0.696911931037903 -0.734561562538147
0.758016049861908 -0.362539231777191
0.516345024108887 -1.48608362674713
0.689109444618225 -1.26654398441315
0.691004514694214 -0.262834846973419
0.718496203422546 0.105387404561043
0.568173229694366 -0.974983632564545
0.712732911109924 -0.1694515645504
0.639097392559052 -0.727489411830902
0.636125504970551 -0.565967381000519
0.669298052787781 -1.89853060245514
0.602628767490387 -0.849152147769928
0.708569049835205 -0.344615489244461
0.720457673072815 -0.661756634712219
0.775956869125366 -0.324456751346588
0.698305606842041 -0.441625654697418
0.654099464416504 -1.01215851306915
0.667405605316162 -1.12734961509705
0.700177907943726 -0.0119739631190896
0.637588202953339 -0.580389022827148
0.689769089221954 -0.638255894184113
0.704124927520752 -0.508554220199585
0.607797861099243 -1.74662554264069
0.723756074905396 -0.176785841584206
0.615954995155334 -0.979259610176086
0.612450003623962 -1.54316699504852
0.621033906936646 -1.35397458076477
};
\addplot [draw=indigo68184, fill=indigo68184, mark size=0.5pt, mark=*, only marks, opacity=0.4]
table{%
x  y
0.891138434410095 0.502957105636597
0.721259772777557 -0.22870597243309
0.822897732257843 0.0392176918685436
0.839605629444122 0.496981739997864
0.738781154155731 0.23216201364994
0.872882962226868 -0.0344661325216293
0.843467116355896 0.063709057867527
0.906592607498169 0.603764891624451
0.865275204181671 0.29091864824295
0.955910623073578 0.330533981323242
0.795973122119904 0.184496819972992
0.99615603685379 0.694700121879578
0.849178373813629 -0.284225434064865
0.840027630329132 0.52819150686264
0.87300443649292 0.42113071680069
0.72693657875061 -0.148590922355652
0.828889071941376 0.535010695457458
0.759231925010681 0.199804097414017
0.952089667320251 0.00743488362058997
0.804990470409393 0.393968224525452
0.988090574741364 0.398911386728287
0.837620317935944 -0.0580257438123226
1.04443871974945 0.378440529108047
0.93951952457428 0.566780865192413
0.930111467838287 0.558036684989929
0.840327143669128 0.143630042672157
0.729459524154663 -0.307474464178085
0.825420260429382 0.173894166946411
0.84042900800705 0.80269581079483
0.994487464427948 0.455619782209396
0.702790141105652 -0.138157188892365
0.947833776473999 0.457719951868057
0.843162596225739 -0.00393895665183663
0.936845362186432 0.201477870345116
0.811022818088531 0.329393476247787
0.99618661403656 0.396029710769653
0.796526253223419 -0.405470371246338
0.914771974086761 0.192417353391647
1.16465675830841 0.564124882221222
0.921949863433838 0.412139892578125
0.823412597179413 0.147158458828926
0.845344662666321 0.217380583286285
0.775592088699341 -0.332102030515671
0.813461542129517 -0.100747674703598
0.85286408662796 0.23948261141777
0.919950842857361 0.672948062419891
1.07982850074768 0.750552713871002
0.863580584526062 0.31887224316597
0.890567421913147 0.337731093168259
0.741962075233459 0.119171045720577
1.04490554332733 1.0538250207901
1.01387202739716 0.60324501991272
0.888412535190582 0.339040219783783
0.68897420167923 0.171118348836899
0.794546008110046 0.390691548585892
0.764066696166992 0.688700914382935
1.00917029380798 0.41088655591011
0.91616678237915 0.458527803421021
0.764938652515411 -0.0113558396697044
0.874263644218445 0.348233968019485
0.910380423069 0.632466077804565
0.998295187950134 0.587593138217926
0.785445213317871 -0.349579662084579
1.09330201148987 0.720517218112946
0.833898484706879 0.397383868694305
0.792163074016571 0.181596174836159
0.856449246406555 -0.0442917197942734
0.749419569969177 0.604431092739105
1.02163219451904 0.681072175502777
0.787913203239441 0.721873879432678
0.831404209136963 0.160967320203781
0.802629470825195 0.224924266338348
1.06654381752014 0.664856016635895
1.02079665660858 0.478266566991806
0.714581489562988 -0.212210208177567
0.879499614238739 0.147228628396988
0.718794941902161 -0.253257930278778
1.06616544723511 0.304746270179749
0.796262264251709 -0.0249103270471096
0.880431413650513 0.625090420246124
0.928153574466705 0.362740129232407
0.759008407592773 0.111564442515373
1.00341546535492 0.659847557544708
0.841886401176453 0.050568912178278
0.899798452854156 -0.00957900751382113
0.953149914741516 0.55257248878479
1.03813374042511 0.531356513500214
0.960785627365112 0.423959910869598
0.89318174123764 0.177735760807991
0.808058798313141 -0.154984727501869
0.841570794582367 0.391711831092834
0.754194736480713 0.0119214886799455
0.985625028610229 0.242402210831642
0.976629197597504 0.546270251274109
0.838586747646332 0.359196037054062
0.692198276519775 -0.109631478786469
0.82898336648941 0.63361120223999
0.745014309883118 0.359038114547729
0.737101674079895 0.430452823638916
1.00148046016693 0.406795710325241
0.982270956039429 0.376486510038376
1.06715857982635 0.742011964321136
0.783630967140198 -0.320396572351456
0.948034286499023 -0.415663868188858
0.919268369674683 0.451454967260361
0.827957093715668 0.398125857114792
1.03192830085754 0.552170991897583
0.999745428562164 0.552848935127258
0.925321817398071 0.361306071281433
0.902918994426727 0.315372496843338
0.912926375865936 0.56588751077652
0.933903396129608 0.2338627576828
0.911509156227112 -0.288570612668991
0.947220325469971 0.397098571062088
0.879519879817963 0.242900848388672
0.938700675964355 0.672793626785278
1.06950545310974 0.512702703475952
0.782834529876709 0.454013526439667
0.947433769702911 0.49359518289566
0.806058883666992 0.121268481016159
0.89156049489975 -0.178286716341972
1.09296226501465 0.604878067970276
0.993975162506104 0.549414098262787
0.847596347332001 0.366237282752991
0.827226042747498 0.483546584844589
0.956071674823761 0.513203978538513
0.978312611579895 0.529410243034363
0.836857795715332 0.777784109115601
0.913958370685577 0.706325769424438
0.832245588302612 0.199767664074898
0.882434904575348 0.0841208398342133
0.906406819820404 0.036890733987093
0.913155972957611 0.711841702461243
1.00228404998779 0.40854400396347
0.846758246421814 0.595186054706573
0.898774445056915 0.263005882501602
0.804686725139618 -0.00684309704229236
0.910041749477386 0.382016807794571
1.04586541652679 0.330640614032745
0.916567921638489 0.364111691713333
0.841531693935394 0.535658538341522
1.17032253742218 0.723407208919525
0.98365193605423 0.280751138925552
0.878551065921783 0.58971905708313
0.92426472902298 0.430738091468811
1.11514246463776 0.773748457431793
0.918347656726837 0.554449498653412
0.814669966697693 0.251292526721954
0.832369089126587 0.512772381305695
0.789314448833466 0.0383128896355629
};
\addlegendentry{PriorGuide low}
\addplot [draw=gold25323137, fill=gold25323137, mark size=0.5pt, mark=*, only marks, opacity=0.4]
table{%
x  y
0.901173293590546 0.319319546222687
0.909929037094116 0.420607626438141
0.888449311256409 0.556943893432617
0.892888605594635 0.306247770786285
0.805916011333466 0.655839443206787
0.861214697360992 0.412091851234436
0.872725129127502 0.539136111736298
0.89355993270874 0.284972757101059
0.827613353729248 0.463103145360947
0.805457413196564 0.435019135475159
1.02283573150635 0.490507483482361
1.0101546049118 0.551400363445282
0.831964433193207 0.569321274757385
0.974386990070343 0.432106912136078
0.897776901721954 0.413630664348602
0.966703116893768 0.540918231010437
0.872602105140686 0.375736773014069
0.828809440135956 0.33721724152565
0.975706934928894 0.498739898204803
0.792940139770508 0.537148296833038
0.906537055969238 0.519008576869965
0.801601886749268 0.330370455980301
0.992701947689056 0.542954087257385
0.902795553207397 0.435275107622147
0.845970511436462 0.649444162845612
0.958923697471619 0.601137757301331
0.855915009975433 0.402171850204468
0.803735852241516 0.437760859727859
0.879172563552856 0.342731952667236
0.748965263366699 0.595088601112366
0.74495792388916 0.414211839437485
0.917512059211731 0.496867477893829
0.787521421909332 0.271700173616409
0.975950479507446 0.582754075527191
0.884072482585907 0.229481786489487
0.880216300487518 0.572869002819061
0.955962181091309 0.452730417251587
0.773422420024872 0.566815853118896
0.888497471809387 0.495971143245697
0.902412891387939 0.387577831745148
0.98640763759613 0.700771450996399
0.906772792339325 0.445810079574585
0.8736492395401 0.648041307926178
0.879198729991913 0.376052230596542
0.955004990100861 0.553288757801056
0.854386270046234 0.405702769756317
0.947400212287903 0.601551175117493
0.875318944454193 0.518328785896301
0.85973995923996 0.456228643655777
0.927099287509918 0.569131553173065
0.90087628364563 0.44283664226532
0.9107506275177 0.582455337047577
0.813517689704895 0.159319028258324
0.897485196590424 0.592711269855499
0.851054012775421 0.597292304039001
0.952105760574341 0.522494316101074
0.915004551410675 0.538543164730072
0.821135759353638 0.477102041244507
0.870068669319153 0.340678423643112
0.885083317756653 0.411782801151276
0.882194578647614 0.175384655594826
0.818268418312073 0.548916399478912
0.878960132598877 0.470391303300858
0.908355355262756 0.291717827320099
0.911640048027039 0.537315130233765
0.758898019790649 0.231375932693481
0.875997424125671 0.266769409179688
0.759286761283875 0.130963876843452
0.896934568881989 0.487053364515305
0.860713183879852 0.258851140737534
0.806842863559723 0.611517131328583
0.938695251941681 0.440060079097748
0.917744278907776 0.406621158123016
0.918708801269531 0.486437231302261
0.886042892932892 0.364664047956467
0.839555621147156 0.278060644865036
0.890384912490845 0.452365398406982
0.900972902774811 0.569686412811279
0.85913747549057 0.633059978485107
0.849299252033234 0.488122642040253
0.827899932861328 0.484847187995911
0.867057859897614 0.312271535396576
0.957353472709656 0.642569303512573
0.854495167732239 0.509821176528931
0.788315057754517 0.654530048370361
0.936831831932068 0.696637332439423
0.814848005771637 0.35672453045845
0.736214220523834 0.064968928694725
0.736915588378906 0.464631110429764
0.82820737361908 0.441448897123337
0.978299498558044 0.407284140586853
0.856796205043793 0.459920585155487
0.851020693778992 0.588841378688812
0.828620851039886 0.354209631681442
1.03920292854309 0.349847257137299
0.845880091190338 0.57774955034256
0.851799130439758 0.405963718891144
0.912911355495453 0.509682714939117
0.883369326591492 0.587256252765656
0.770640850067139 0.432311087846756
0.870080709457397 0.318758577108383
0.779310882091522 0.514086842536926
0.874260902404785 0.557626962661743
0.799567401409149 0.467595517635345
0.864814460277557 0.649382472038269
0.913068175315857 0.292828857898712
0.829311192035675 0.446727365255356
0.858775734901428 0.455175668001175
0.923168420791626 0.494581252336502
0.889836132526398 0.389905005693436
0.879397809505463 0.406243711709976
0.863540887832642 0.571673035621643
0.891099452972412 0.396397799253464
0.831031978130341 0.462779760360718
0.952053308486938 0.461922109127045
0.814369797706604 0.499302417039871
0.781917572021484 0.667125821113586
0.812588095664978 0.289482206106186
0.940903782844543 0.605016469955444
0.915915608406067 0.542424499988556
0.924816846847534 0.436763793230057
0.818468451499939 0.239673063158989
0.914379298686981 0.359268724918365
0.849668741226196 0.458491325378418
0.856326997280121 0.307124137878418
1.06891572475433 0.561070084571838
0.841920554637909 0.581251204013824
0.905180752277374 0.562742054462433
0.839736819267273 0.527702689170837
0.83315896987915 0.222938895225525
0.9301638007164 0.411020040512085
0.822464168071747 0.676420331001282
0.822251617908478 0.376219093799591
0.966297030448914 0.324484676122665
0.782375395298004 0.409233212471008
0.906196117401123 0.571954250335693
0.867141902446747 0.483882904052734
0.826413691043854 0.510708391666412
0.799533843994141 0.506467342376709
0.858662188053131 0.303810060024261
0.970001697540283 0.572675883769989
0.881883502006531 0.310515433549881
0.936971783638 0.406193822622299
0.909214019775391 0.453278601169586
0.830232977867126 0.564278483390808
0.831827878952026 0.396809875965118
0.90180242061615 0.557364583015442
0.936030149459839 0.443429499864578
0.861597061157227 0.29208180308342
0.856251537799835 0.488399535417557
};
\addlegendentry{PriorGuide high}
\addplot [
  draw=palevioletred204121167,
  fill=palevioletred204121167,
  mark size=0.5pt,
  mark=*,
  only marks,
  opacity=0.4
]
table{%
x  y
0.71270489692688 -1.3552827835083
};
\addlegendentry{Diffusion model}

\addplot [draw=black, fill=black, mark=*, only marks, mark size=1.2pt, forget plot]
table{%
x  y
0.81 0.55
};
\addlegendentry{True parameter}

\end{axis}

\end{tikzpicture}

%% file: fig/fig1/4.x_observed_bottomleft.tex
\begin{tikzpicture}

\definecolor{darkgrey176}{RGB}{176,176,176}

\begin{axis}[
legend cell align={left},
legend style={fill opacity=0.8, draw opacity=1, text opacity=1, draw=none, fill=none},
tick align=outside,
tick pos=left,
x grid style={darkgrey176},
xmin=-1.2, xmax=25.2,
y grid style={darkgrey176},
ymin=-0.5, ymax=10.5,
ytick style={color=black}
]
\addplot [semithick, black, mark=*, mark size=1, mark options={solid}]
table {%
0 10.000000000814
1 8.70195407609928
2 7.72544618482757
3 6.59009560712606
4 5.85602201646523
5 5.15789665354099
6 4.77424155566681
7 4.17717895004503
8 3.81800112550548
9 3.17061982838972
10 2.4742293881755
11 2.32297551394974
12 2.04275280244002
13 1.72574435214115
14 1.40142418709285
15 1.09775208062415
16 0.856239060461355
17 0.935155318117905
18 0.823273791396189
19 0.811630942834807
20 0.869041877377272
21 1.06557292697806
22 1.321057328834
23 1.10454947323341
24 0.921979379876972
};
\addlegendentry{True time-series data}
\end{axis}

\end{tikzpicture}

%% file: fig/fig1/5.x_groundtruth_bottommid.tex
\begin{tikzpicture}

\definecolor{darkcyan0114178}{RGB}{0,114,178}
\definecolor{darkgrey176}{RGB}{176,176,176}
\definecolor{orange2301590}{RGB}{230,159,0}

\begin{axis}[
legend cell align={left},
legend style={fill opacity=0.8, draw opacity=1, text opacity=1, draw=none, fill=none},
tick align=outside,
tick pos=left,
x grid style={darkgrey176},
xmin=-1.2, xmax=25.2,
y grid style={darkgrey176},
ymin=-0.5, ymax=10.5,
ytick style={color=black}
]
\addplot [semithick, black, mark=*, mark size=1, mark options={solid}]
table {%
17 0.935155318117905
18 0.823273791396189
19 0.811630942834807
20 0.869041877377272
21 1.06557292697806
22 1.321057328834
23 1.10454947323341
24 0.921979379876972
};
\addlegendentry{Observed data}
\path [draw=orange2301590, fill=orange2301590, opacity=0.2]
(axis cs:0,10.005220413208)
--(axis cs:0,9.99330139160156)
--(axis cs:1,6.14357089996338)
--(axis cs:2,3.82861709594727)
--(axis cs:3,2.35704183578491)
--(axis cs:4,1.46036005020142)
--(axis cs:5,0.93113899230957)
--(axis cs:6,0.565439701080322)
--(axis cs:7,0.402314305305481)
--(axis cs:8,0.268092155456543)
--(axis cs:9,0.264916062355042)
--(axis cs:10,0.200142025947571)
--(axis cs:11,0.18969988822937)
--(axis cs:12,0.386345565319061)
--(axis cs:13,0.291952192783356)
--(axis cs:14,0.303759515285492)
--(axis cs:15,0.31588351726532)
--(axis cs:16,0.456247091293335)
--(axis cs:16,1.49363744258881)
--(axis cs:16,1.49363744258881)
--(axis cs:15,1.72390425205231)
--(axis cs:14,1.81045079231262)
--(axis cs:13,1.97209882736206)
--(axis cs:12,2.18869400024414)
--(axis cs:11,2.55925130844116)
--(axis cs:10,2.71763277053833)
--(axis cs:9,2.96934127807617)
--(axis cs:8,3.39044904708862)
--(axis cs:7,3.86452913284302)
--(axis cs:6,4.50765705108643)
--(axis cs:5,5.19298505783081)
--(axis cs:4,5.98453187942505)
--(axis cs:3,6.91362333297729)
--(axis cs:2,7.9634313583374)
--(axis cs:1,9.06949329376221)
--(axis cs:0,10.005220413208)
--cycle;

\path [draw=darkcyan0114178, fill=darkcyan0114178, opacity=0.2]
(axis cs:0,10.013204574585)
--(axis cs:0,9.98423004150391)
--(axis cs:1,7.65460538864136)
--(axis cs:2,5.8609356880188)
--(axis cs:3,4.55982303619385)
--(axis cs:4,3.65835762023926)
--(axis cs:5,2.93386840820312)
--(axis cs:6,2.35938763618469)
--(axis cs:7,1.96196460723877)
--(axis cs:8,1.56119585037231)
--(axis cs:9,1.356325507164)
--(axis cs:10,1.24266850948334)
--(axis cs:11,1.15192341804504)
--(axis cs:12,1.01819658279419)
--(axis cs:13,0.886841416358948)
--(axis cs:14,0.757204592227936)
--(axis cs:15,0.628236055374146)
--(axis cs:16,0.635490000247955)
--(axis cs:16,1.57320952415466)
--(axis cs:16,1.57320952415466)
--(axis cs:15,1.95139288902283)
--(axis cs:14,2.18461990356445)
--(axis cs:13,2.37880754470825)
--(axis cs:12,2.62072134017944)
--(axis cs:11,2.86955285072327)
--(axis cs:10,3.1403636932373)
--(axis cs:9,3.45359230041504)
--(axis cs:8,3.77428102493286)
--(axis cs:7,4.21749210357666)
--(axis cs:6,4.67317771911621)
--(axis cs:5,5.19360828399658)
--(axis cs:4,5.78271961212158)
--(axis cs:3,6.6192626953125)
--(axis cs:2,7.56906175613403)
--(axis cs:1,8.73849201202393)
--(axis cs:0,10.013204574585)
--cycle;

\addplot [semithick, darkcyan0114178]
table {%
0 9.99871730804443
1 8.19654846191406
2 6.71499872207642
3 5.58954286575317
4 4.72053861618042
5 4.06373834609985
6 3.51628255844116
7 3.08972835540771
8 2.66773843765259
9 2.40495896339417
10 2.19151616096497
11 2.01073813438416
12 1.81945896148682
13 1.63282442092896
14 1.47091221809387
15 1.28981447219849
16 1.10434973239899
};
\addlegendentry{Posterior predictive (new)}
\addplot [semithick, orange2301590]
table {%
0 9.99926090240479
1 7.60653209686279
2 5.89602422714233
3 4.6353325843811
4 3.72244596481323
5 3.06206202507019
6 2.53654837608337
7 2.1334216594696
8 1.82927060127258
9 1.61712872982025
10 1.4588874578476
11 1.37447559833527
12 1.28751981258392
13 1.13202548027039
14 1.05710518360138
15 1.01989388465881
16 0.974942266941071
};
\addlegendentry{Posterior predictive (old)}

\end{axis}

\end{tikzpicture}

%% file: fig/fig1/6.x_amortized_bottomright.tex
\begin{tikzpicture}

\definecolor{darkgrey176}{RGB}{176,176,176}
\definecolor{gold25323137}{RGB}{253,231,37}
\definecolor{indigo68184}{RGB}{68,1,84}
\definecolor{palevioletred204121167}{RGB}{204,121,167}

\begin{axis}[
legend cell align={left},
legend style={fill opacity=0.8, draw opacity=1, text opacity=1, draw=none, fill=none},
tick align=outside,
tick pos=left,
x grid style={darkgrey176},
xlabel={Time step},
xmin=-1.2, xmax=25.2,
y grid style={darkgrey176},
ymin=-0.5, ymax=10.5,
ytick style={color=black}
]
\addplot [semithick, palevioletred204121167, opacity=0.6, mark=*, mark size=0.5, mark options={solid},forget plot]
table {%
0 9.9976692199707
1 7.46467447280884
2 5.85842227935791
3 4.36041975021362
4 3.48922824859619
5 2.88770198822021
6 2.53678512573242
7 2.26569366455078
8 2.0921950340271
9 1.7558319568634
10 1.70303058624268
11 1.26667332649231
12 1.19427061080933
13 0.60378098487854
14 0.638120651245117
15 0.711701393127441
16 1.15127491950989
};
\addplot [semithick, palevioletred204121167, opacity=0.6, mark=*, mark size=0.5, mark options={solid}, forget plot]
table {%
0 10.0011558532715
1 8.90381622314453
2 7.64387083053589
3 6.63996410369873
4 6.04995489120483
5 5.03352308273315
6 4.20617771148682
7 3.47500205039978
8 3.20916223526001
9 2.65028810501099
10 2.51760244369507
11 2.56220865249634
12 2.17250633239746
13 1.64024782180786
14 1.33391690254211
15 1.33437538146973
16 1.24204516410828
};
\addplot [semithick, palevioletred204121167, opacity=0.6, mark=*, mark size=0.5, mark options={solid}, forget plot]
table {%
0 9.99577140808105
1 7.46449661254883
2 5.57063484191895
3 4.52751493453979
4 3.65019130706787
5 2.92704486846924
6 2.81730461120605
7 2.3259072303772
8 2.12935137748718
9 1.68420910835266
10 1.74624180793762
11 1.85851049423218
12 1.85691428184509
13 2.11420178413391
14 1.66781234741211
15 1.14083552360535
16 1.002516746521
};
\addplot [semithick, palevioletred204121167, opacity=0.6, mark=*, mark size=0.5, mark options={solid}, forget plot]
table {%
0 10.0005187988281
1 7.44685077667236
2 5.25803756713867
3 3.77902603149414
4 2.77070665359497
5 2.13316702842712
6 1.67271041870117
7 1.40568590164185
8 1.15327501296997
9 0.884216070175171
10 0.761843681335449
11 0.743271827697754
12 1.07503724098206
13 0.887565851211548
14 0.60185170173645
15 0.691388845443726
16 1.12053108215332
};
\addplot [semithick, palevioletred204121167, opacity=0.6, mark=*, mark size=0.5, mark options={solid}, forget plot]
table {%
0 9.99916076660156
1 7.242506980896
2 4.89753866195679
3 3.73411703109741
4 2.71920919418335
5 2.18460011482239
6 1.46365427970886
7 0.695533752441406
8 0.474246025085449
9 0.563302516937256
10 0.674704790115356
11 0.740236759185791
12 1.1515166759491
13 1.19712162017822
14 1.047039270401
15 0.938114881515503
16 1.18769598007202
};
\addplot [semithick, indigo68184, opacity=0.6, mark=*, mark size=0.5, mark options={solid}]
table {%
0 10.0013771057129
1 8.05552291870117
2 6.31785297393799
3 4.78858757019043
4 4.17738246917725
5 3.76898002624512
6 3.26209211349487
7 2.82662010192871
8 2.59036731719971
9 2.46377015113831
10 1.81461715698242
11 1.88574385643005
12 1.86359691619873
13 1.71390461921692
14 1.31236839294434
15 1.12339091300964
16 1.13762974739075
};
\addlegendentry{PriorGuide low }
\addplot [semithick, gold25323137, opacity=0.6, mark=*, mark size=0.5, mark options={solid}]
table {%
0 10.0009098052979
1 8.14807224273682
2 6.55539321899414
3 5.30996179580688
4 4.59676218032837
5 3.8699688911438
6 3.51088070869446
7 3.13498687744141
8 2.38621807098389
9 2.20856213569641
10 2.01218318939209
11 1.70228052139282
12 1.87865614891052
13 1.65437936782837
14 1.02415895462036
15 0.847754001617432
16 1.13252520561218
};
\addlegendentry{PriorGuide high}
\addplot [semithick, indigo68184, opacity=0.6, mark=*, mark size=0.5, mark options={solid}, forget plot]
table {%
0 9.99675273895264
1 7.54010009765625
2 6.43063306808472
3 4.72170972824097
4 3.90650224685669
5 3.15824842453003
6 2.48434400558472
7 2.05713891983032
8 1.45052695274353
9 1.59794139862061
10 1.16829800605774
11 1.26186299324036
12 0.896781444549561
13 0.941824674606323
14 0.896895408630371
15 0.762180089950562
16 0.743858814239502
};
\addplot [semithick, gold25323137, opacity=0.6, mark=*, mark size=0.5, mark options={solid}, forget plot]
table {%
0 9.99961090087891
1 8.08651065826416
2 6.56615924835205
3 5.02896547317505
4 3.95868134498596
5 3.18705749511719
6 2.40894746780396
7 2.07681584358215
8 1.60773777961731
9 1.74450302124023
10 1.79670834541321
11 1.64686632156372
12 1.27806234359741
13 1.48082494735718
14 1.65974378585815
15 1.36904191970825
16 1.32543158531189
};
\addplot [semithick, indigo68184, opacity=0.6, mark=*, mark size=0.5, mark options={solid}, forget plot]
table {%
0 9.99930763244629
1 7.99609565734863
2 6.22888422012329
3 5.06410694122314
4 4.29272556304932
5 3.46104121208191
6 2.82838106155396
7 2.70933818817139
8 2.63748550415039
9 2.51850271224976
10 2.34969854354858
11 2.06491112709045
12 1.99869799613953
13 1.59165501594543
14 1.52601981163025
15 1.695969581604
16 1.62591958045959
};
\addplot [semithick, gold25323137, opacity=0.6, mark=*, mark size=0.5, mark options={solid}, forget plot]
table {%
0 10.0026893615723
1 7.86498069763184
2 6.49926137924194
3 5.12459468841553
4 4.44417953491211
5 3.95753121376038
6 3.2477695941925
7 3.18909645080566
8 3.16700196266174
9 2.95136952400208
10 2.28857803344727
11 1.91701316833496
12 1.18592214584351
13 1.63158845901489
14 1.67895460128784
15 1.25350666046143
16 1.09328651428223
};
\addplot [semithick, indigo68184, opacity=0.6, mark=*, mark size=0.5, mark options={solid}, forget plot]
table {%
0 10.0015239715576
1 7.95498752593994
2 6.0868444442749
3 5.02568101882935
4 4.38510465621948
5 3.49426341056824
6 3.44836163520813
7 3.26849889755249
8 2.72303938865662
9 2.40355777740479
10 2.34832525253296
11 2.46666312217712
12 2.23531675338745
13 2.35313606262207
14 2.2877082824707
15 2.07071018218994
16 1.47414088249207
};
\addplot [semithick, gold25323137, opacity=0.6, mark=*, mark size=0.5, mark options={solid}, forget plot]
table {%
0 10.0061683654785
1 8.40995597839355
2 7.20882415771484
3 5.88095140457153
4 5.18009471893311
5 4.42397928237915
6 3.84443712234497
7 3.6857385635376
8 3.01422500610352
9 2.6878879070282
10 2.74497294425964
11 2.51650238037109
12 2.13789391517639
13 2.02391767501831
14 1.9574556350708
15 1.51248288154602
16 1.29686331748962
};
\addplot [semithick, indigo68184, opacity=0.6, mark=*, mark size=0.5, mark options={solid}, forget plot]
table {%
0 9.99911308288574
1 7.52173805236816
2 6.08572483062744
3 5.00184535980225
4 4.34190320968628
5 3.73165488243103
6 2.82307720184326
7 2.51927328109741
8 2.3911919593811
9 2.39091968536377
10 2.11238980293274
11 2.07540607452393
12 1.97313976287842
13 1.86403322219849
14 1.77225160598755
15 1.80021381378174
16 1.79664635658264
};
\addplot [semithick, gold25323137, opacity=0.6, mark=*, mark size=0.5, mark options={solid}, forget plot]
table {%
0 9.99923133850098
1 8.52393436431885
2 7.63214015960693
3 6.66232681274414
4 5.72579336166382
5 5.28878688812256
6 4.76930665969849
7 3.98912191390991
8 3.65707397460938
9 3.10312271118164
10 3.09823441505432
11 3.11516737937927
12 2.94994044303894
13 2.26704549789429
14 1.68763518333435
15 2.1097354888916
16 1.85646796226501
};
\addplot [semithick, black, mark=*, mark size=1, mark options={solid}, forget plot]
table {%
17 0.935155318117905
18 0.823273791396189
19 0.811630942834807
20 0.869041877377272
21 1.06557292697806
22 1.321057328834
23 1.10454947323341
24 0.921979379876972
};
\addplot [semithick, palevioletred204121167, opacity=0.6, mark=*, mark size=0.5, mark options={solid}]
table {%
0 9.9976692199707
};
\addlegendentry{Diffusion model}
\end{axis}

\end{tikzpicture}

%% file: appendix.tex
\appendix

\setcounter{figure}{0}
\setcounter{table}{0}
\setcounter{equation}{0}
\renewcommand{\thefigure}{A\arabic{figure}}
\renewcommand{\theequation}{A\arabic{equation}}
\renewcommand{\thetable}{A\arabic{table}}
\renewcommand{\theHfigure}{A\arabic{figure}}
\renewcommand{\theHtable}{A\arabic{table}}
\newtheorem{proposition}{Proposition}
\newtheorem*{proposition*}{Proposition}

The full appendix is organized as follows:

\begin{itemize}
\item \cref{app:methods} provides an extended description of related work and our method.
\item \cref{app:proofs} presents mathematical proofs and derivations.
\item \cref{app:experimental_details} describes our experimental and statistical procedures.
\item \cref{app:additional_results} shows supplementary experimental results and analyses.
\item \cref{app:miscellanea} details computational and software resources.
\end{itemize}

\section{Method details}
\label{app:methods}
In this section we start with an extended discussion of related work (\cref{app:related_work}). We then detail the main PriorGuide algorithm (\cref{app:inference_algorithm}), the Langevin dynamics step size (\cref{app:langevin}), and the prior coverage diagnostics (\cref{app:prior_diagnostic}).

\subsection{Extended related work} \label{app:related_work}

Section \ref{sec:background} in the main paper situates PriorGuide within the broader context of Simulation-Based Inference (SBI) and diffusion models. Here we explore those connections in more detail.

\paragraph{Amortized SBI and prior specification}
The output of standard amortized SBI techniques, such as Neural Posterior Estimation (NPE)~\citep{greenberg2019automatic, lueckmann2017flexible, papamakarios2016fast}, %
is tied to the fixed prior, $\pt(\vtheta)$, used during their training phase. Modifying this prior traditionally requires retraining the entire amortized model, which can be prohibitive given expensive simulators. %
PriorGuide offers a solution specifically for diffusion-based amortized models, enabling adaptation to a new prior $q(\vtheta)$ by modifying the sampling process itself, thus bypassing the need for retraining.
Other SBI techniques such as Neural Likelihood Estimation (NLE)~\citep{papamakarios2019sequential, lueckmann2019likelihood} and Neural Ratio Estimation (NRE)~\citep{hermans2020likelihood, thomas2022likelihood} do not amortize posterior inference, in that they only approximate the likelihood (or likelihood ratio), and traditional approximate inference techniques such as MCMC or variational inference need to be run to obtain a posterior by combining the surrogate likelihood (or likelihood ratio) with a prior.

\paragraph{Diffusion models for SBI}
PriorGuide enhances versatile diffusion-based SBI models like Simformer~\citep{gloeckler2024all}, which we use as our base model. As described in the main text, Simformer leverages a transformer-based diffusion model over the joint space of parameters and data, $p(\vtheta, \vx)$, allowing it to provide amortized samples from arbitrary conditionals (\eg, posteriors, likelihoods) once trained, though this training inherently uses a fixed prior. Other diffusion-based SBI methods include techniques to combine learned scores from posteriors of individual observations to handle multiple data sources~\citep{geffner2023compositional, linhart2024diffusion} and methods that focus on efficient (sequential) training of posterior score models~\citep{sharrock2024sequential}. Since these approaches ultimately yield score functions for posteriors (conditioned on their respective training priors), PriorGuide's test-time score guidance mechanism could also adapt these trained models to new prior beliefs post-hoc.

\paragraph{Amortized prior adaptation}
As discussed in Section~\ref{sec:prior_adaptation} of the main text, several recent amortized SBI methods support prior changes by training over a \emph{meta-prior}---a distribution or discrete set over possible prior specifications---including ACE~\citep{chang2025amortized}, the Distribution Transformer~\citep{whittle2025distribution}, and sensitivity-aware SBI~\citep{elsemuller2024sensitivity}. These approaches learn to incorporate alternative priors at inference time, but rely on pre-training across the chosen meta-prior. PriorGuide, by contrast, performs prior adaptation entirely at test time, without amortizing over priors, and can therefore accommodate new target priors beyond any pre-specified meta-prior.

\paragraph{Prior misspecification}
Effective PriorGuide use requires the target prior $q(\vtheta)$ to overlap substantially with the training prior $\pt(\vtheta)$ to prevent the trained score model from operating out-of-distribution (OOD). This concern for reliable inference echoes broader SBI efforts that address simulator misspecification and its impact on inference reliability using techniques like MMD or robust statistics~\citep{schmitt2023detecting, huang2024learning}. Similarly, \citet{yuyan2025robust} explore robust SBI with classes of priors and assess potential prior-likelihood conflicts. While this paper proposes a simple diagnostics to ensure the new prior is compatible with the learned score model, other techniques from the SBI literature could be used.

\paragraph{Score-based guidance}
PriorGuide's core mechanism is a novel application of score-based guidance. While the general idea of guiding diffusion models is well-established for inverse problems and conditional generation~\citep{dhariwal2021diffusion, ho2022classifier, chung2023diffusion, song2023pseudoinverse, song2023loss}, PriorGuide's specific contribution lies in deriving and applying a guidance term for the \emph{prior ratio}, akin to an importance ratio term. Moreover, its analytical approximation using a GMM for the ratio is tailored to this prior adaptation task. This contrasts with other guidance methods that often focus on incorporating information from a known forward (observation) model (\eg, linear operators imaging;~\citealp{chung2023diffusion, finzi2023user}) or specific loss functions~\citep{song2023loss}. Further, the literature either tends to focus on the analytic case where the guidance function is Gaussian~\citep{song2023pseudoinverse, peng2024improving, rissanen2024hunch} or entirely forego an analytical integral while keeping the Gaussian reverse approximation~\citep{song2023loss}. Our approach strikes a middle ground between analytic tractability and flexible, non-Gaussian guidance functions. The approximation of the reverse transition kernel and its covariance also adapts approaches seen in works like~\citep{song2023pseudoinverse, ho2022video}. 

A number of works has considered Monte Carlo corrections to approximate inference time modifications to diffusion models, such as guidance. \citet{wu2023practical, dou2024diffusion, cardosomonte, thorntoncomposition, lee2025debiasing, skretafeynman} propose variants of Sequential Monte Carlo for asymptotically exact modifications to the generative distribution. \citet{du2023reduce, geffner2023compositional} use MCMC corrections for compositional generation. Early work on using Langevin dynamics in unconditional score based generative models includes \citet{song2019generative} and the predictor-corrector sampler from \citet{song2021score}.

\subsection{PriorGuide inference algorithm} \label{app:inference_algorithm}

\Cref{alg:prior_guide} details the PriorGuide posterior inference algorithm and \cref{alg:prior_guide_post_pred} contains the posterior predictive inference version.

\begin{algorithm}[ht]
\caption{PriorGuide posterior inference}
\label{alg:prior_guide}
\begin{algorithmic}[1]
\State \textbf{Input:} Trained diffusion-based inference model $\gM$, training prior $\pt(\vtheta)$, test-time prior $q(\vtheta)$, number of mixture components $K$ for the prior ratio, min diffusion time $T_\mathrm{min}$, max diffusion time $T_\mathrm{max}$, generation schedule nonlinearity parameter $\rho$, number of diffusion steps $N$, Langevin ratio $\eta$, number of Langevin steps $N_L$, conditioning information $\vx$.

\State \textbf{Output:} Posterior samples $\vtheta_{T_\mathrm{min}}$ at time $T_\mathrm{min}$.

\State $\{r(\vtheta) \mid \{w_i, \vmu_i, \vSigma_i\}_i^K \} \gets \Call{FitGMM}{p_\text{train}(\vtheta), q(\vtheta), K}$, \qquad with $r(\vtheta) \approx \frac{q(\vtheta)}{\pt(\vtheta)}$
\State $t_N, \dots, t_0 \gets \text{Linspace}({1, 0, N})^\rho \cdot (T_\mathrm{max} - T_\mathrm{min}) + T_\mathrm{min}$
\State $\vtheta_{t_N} \sim \mathcal{N}(0, \sigma(T_\mathrm{max})^2\mathbf{I})$
\For{$j = N \to 1$}
    \State $t = t_j, \Delta t = t_{j-1} - t_{j}$
    \For{$\ell = 1 \to N_L$}
        \State Compute original score $s(\vtheta_t, t, \vx)$ using $\gM$
        \State Compute prior guidance $s_p \leftarrow \nabla_{\vtheta_t} \log \mathbb{E}_{p(\vtheta_{0}\mid\vtheta_t,\x)}[\ratio(\vtheta_{0})]$ with $\{w_i, \vmu_i, \vSigma_i\}_i^K$
        \State Compute guided score $s_L \gets s(\vtheta_t, t, \vx) + s_p $
        \State Langevin dynamics step $\vtheta_t \gets \vtheta_t + \eta \frac{\dot\sigma(t)\sigma(t)}{2}s_L + \sqrt{\eta\dot\sigma(t)\sigma(t)} \varepsilon,\quad \varepsilon\sim \mathcal{N}(0, \mathbf{I})$
    \EndFor
    \State Compute original score $s(\vtheta_t, t, \vx)$ using $\gM$
    \State Compute prior guidance $s_p \leftarrow \nabla_{\vtheta_t} \log \mathbb{E}_{p(\vtheta_{0}\mid\vtheta_t,\x)}[\ratio(\vtheta_{0})]$ with $\{w_i, \vmu_i, \vSigma_i\}_i^K$
    \State Compute new guided score $\tilde{s} \gets s(\vtheta_t, t, \vx) + s_p$
    \State Euler-Maruyama step $\vtheta_{t_{j-1}} \gets \vtheta_t + 2\dot\sigma(t)\sigma(t) \Delta t \tilde s + \sqrt{2\dot \sigma(t)\sigma(t) \Delta t} \varepsilon, \quad \varepsilon\sim \mathcal{N}(0, \mathbf{I})$%
\EndFor \\
\Return $\vtheta_{t_0}$
\end{algorithmic}
\end{algorithm}

\paragraph{\textsc{FitGMM} subroutine} The \textsc{FitGMM($p(\vtheta)$, $q(\vtheta)$, K)} subroutine takes as input two distributions, $p(\vtheta)$ and $q(\vtheta)$, and fits a generalized Gaussian mixture model (GMM) with $K$ components to approximate the ratio $r(\vtheta) = q(\vtheta) / p(\vtheta)$, as described in \cref{sec:methods} of the main text. This is a standard GMM but coefficients are not constrained to sum to one. The subroutine returns the approximated ratio as the weights, means, and covariance matrices of the mixture. This is implemented as a stochastic optimization procedure over the parameters of the generalized mixture by minimizing the $L_2$ error between the GMM and the ratio (see \cref{app:posterior-inference-setup}).

\begin{algorithm}[ht]
\caption{PriorGuide posterior predictive inference}
\label{alg:prior_guide_post_pred}
\begin{algorithmic}[1]
\State \textbf{Input:} Trained diffusion-based inference model $\gM$, training prior $p_\text{train}(\vtheta)$, test-time prior $q(\vtheta)$, number of mixture components $K$ for the prior ratio, min diffusion time $T_\mathrm{min}$, max diffusion time $T_\mathrm{max}$, generation schedule nonlinearity parameter $\rho$, number of diffusion steps $N$, Langevin ratio $\eta$, number of Langevin steps $N_L$, conditioning information $\vx$. 

\State \textbf{Output:} Posterior predictive samples $\vx^\star_{T_\mathrm{min}}$ at time $T_\mathrm{min}$.

\State $\{r(\vtheta) \mid \{w_i, \vmu_i, \vSigma_i\}_i^K \} \gets \Call{FitGMM}{p_\text{train}(\vtheta), q(\vtheta), K}$, \qquad with $r(\vtheta) \approx \frac{q(\vtheta)}{\pt(\vtheta)}$
\State $t_N, \dots, t_0 \gets \text{Linspace}({1, 0, N})^\rho \cdot (T_\mathrm{max} - T_\mathrm{min}) + T_\mathrm{min}$
\State $\x_{t_N}^\star \sim \mathcal{N}(0, \sigma(T_\mathrm{max})^2\mathbf{I})$
\State $\vtheta_{t_N}^\star \sim \mathcal{N}(0, \sigma(T_\mathrm{max})^2\mathbf{I})$
\State $\vxi_{t_N} = (\vtheta_{t_N}, \x_{t_N}^\star)$
\For{$j = N \to 1$}
    \State $t = t_j, \Delta t = t_{j-1} - t_{j}$
    \For{$\ell = 1 \to N_L$}
        \State Compute original score $s(\vxi_t, t, \vx)$ using $\gM$
        \State Compute prior guidance $s_p \leftarrow \nabla_{\vxi_t} \log \mathbb{E}_{p(\vtheta_{0}\mid\vxi_t,\vx)}[\ratio(\vtheta_{0})]$ with $\{w_i, \vmu_i, \vSigma_i\}_i^K$
        \State Compute guided score $s_L \gets s(\vxi_t, t, \vx) + s_p $
        \State Langevin dynamics step $\vxi_t \gets \vxi_t + \eta \frac{\dot\sigma(t)\sigma(t)}{2}s_L + \sqrt{\eta\dot\sigma(t)\sigma(t)} \varepsilon,\quad \varepsilon\sim \mathcal{N}(0, \mathbf{I})$
    \EndFor
    \State Compute original score $s(\vxi_t, t, \vx)$ using $\gM$
    \State Compute prior guidance $s_p \leftarrow \nabla_{\vxi_t} \log \mathbb{E}_{p(\vtheta_{0}\mid\vxi_t,\vx)}[\ratio(\vtheta_{0})]$ with $\{w_i, \vmu_i, \vSigma_i\}_i^K$
    \State Compute new guided score $\tilde{s} \gets s(\vxi_t, t, \vx) + s_p$
    \State Euler-Maruyama step $\vxi_{t_{j-1}} \gets \vxi_t + 2\dot\sigma(t)\sigma(t) \Delta t \tilde s + \sqrt{2\dot \sigma(t)\sigma(t) \Delta t} \varepsilon, \quad \varepsilon\sim \mathcal{N}(0, \mathbf{I})$%
\EndFor
\\
\Return $\x^\star_{t_0}$

\end{algorithmic}
\end{algorithm}

\subsection{Langevin dynamics step size}
\label{app:langevin}

 Since Langevin dynamics becomes exact only at small step sizes, a schedule for the step size is an important detail of the PriorGuide algorithm. At larger noise levels $\sigma(t)$, the distribution $p(\z_t)$ is more spread out and thus we can take larger steps, while a smaller step size is necessary for lower $\sigma(t)$ levels. We take inspiration from the similarity of the Euler-Maruyama reverse SDE sampler and the Langevin dynamics algorithm (see, \eg, \citep{karras2022elucidating}), and calibrate the step size such that the noise added in the sampling step is proportional to the noise added in the Euler-Maruyama step when moving to the next noise level. In particular, the update rule for Langevin dynamics at noise level $\sigma(t)$ is
\begin{align}
    \z_t \leftarrow \z_t + \delta(t) \nabla_{\z}\log p(\z_t) + \sqrt{2\delta(t)} \varepsilon, \quad \varepsilon \sim \mathcal{N}(0,\mathbf{I}), \quad \delta(t) = \eta \frac{\dot \sigma(t)\sigma(t)\Delta t}{2},
\end{align}
where $\Delta t$ is the step size for the next step in the reverse SDE, and $\eta$ is a scaling parameter. $\eta=1$ corresponds to the same noise level as the reverse Euler-Maruyama step. The overall scaling $\eta$ can be tuned to lower values for improved accuracy of the MCMC procedure, at the cost of slower mixing. We use $\eta=0.5$ for all our experiments. 

\subsection{OOD Diagnostic for test-time priors} 
\label{app:prior_diagnostic}

We assess out-of-distribution (OOD) behavior using a Monte Carlo sample-based diagnostic that estimates the mass of the test-time prior $q(\vtheta)$ falling into the $\alpha$-quantile of the training prior $p_\mathrm{train}(\vtheta)$.
We outline the procedures as follows:
\begin{itemize}
    \item First, we compute the log-density threshold $t$ under $p_\mathrm{train}$.
    \begin{itemize}
        \item Draw $M_p$ samples $\{\vtheta^{(i)}\}_{i=1}^{M_p}\sim p_\mathrm{train}$. 
        \item Compute their log-densities $\ell_i = \log p_\mathrm{train}(\vtheta^{(i)})$.
        \item Let $t$ be the empirical $\alpha$-quantile of $\{\ell_i\}$, i.e.~$t = \mathrm{quantile}\bigl(\{\ell_i\},\,\alpha \bigr)$.
        \item By construction, we have $\Pr_{p_\mathrm{train}}[\log p_\mathrm{train}(\vtheta)<t]\approx\alpha$.
    \end{itemize}
    
    \item Then, we estimate the OOD fraction under $q$.
    \begin{itemize}
        \item Draw $M_q$ samples $\{\vphi^{(j)}\}_{j=1}^{M_q}\sim q$.
        \item Evaluate each under $p$: $\ell'_j = \log p(\vphi^{(j)})$.
        \item Count the fraction $\widehat{r}$ with $\ell'_j < t$: $\widehat{r} = \frac1{M_q}\sum_{j=1}^{M_q} \bigl[\ell'_j < t\bigr]$.
        \item If $\widehat{r} > \alpha$, declare $q$ OOD.
    \end{itemize}
\end{itemize}

Across all simulators, we employ a quantile threshold of $\alpha=0.001$, chosen as $10 / N_\text{train}$, where $N_\text{train}$ is the number of simulated parameters used to train the amortized method. We verify that each newly constructed prior (procedures detailed in \cref{app:test-time-prior-generations}) successfully passes the above OOD diagnostics.

\paragraph{Validity of the prior ratio}
For the prior ratio $\ratio(\vtheta) = q(\vtheta) / \pt(\vtheta)$ to be well-defined, we require $\pt(\vtheta) = 0 \rightarrow q(\vtheta) = 0$, \ie the target prior cannot have nonzero density where the training prior has zero density. This is not fully addressed by the OOD diagnostic, which looks for substantial overlap of prior mass. To avoid pointwise issues with zero-density regions (\eg, when $\pt(\vtheta)$ is a bounded uniform distribution), we further truncate the tails of $q(\vtheta)$, setting its density to zero if $\pt(\vtheta) = 0$. Note that this is done after the OOD check, which means that this adjustment only affects the far tails of $q(\vtheta)$, with negligible influence on inference performance.

\section{Theoretical results}\label{app:proofs}

We provide in this section full derivations and proofs of our theoretical results. This includes the derivation of the guidance term (\cref{app:guidance_derivation}) and extended statements and proofs for Proposition \ref{eq:prop1} (\cref{app:proof1}) and Proposition \ref{eq:prop2} (\cref{app:proof2}) from the main text.

\subsection{Derivation of the guidance term}
\label{app:guidance_derivation}

Here we provide a detailed derivation for the guidance term, Eq.~\ref{eq:priorw} from the main text. We start from Eq.~\ref{eq:prior_approx}, which writes the guidance as the score of the expectation of the prior ratio under the reverse transition kernel, which are approximated by a Gaussian mixture model and a Gaussian, respectively:
\begin{align}
\nabla_{\vtheta_t} \log \mathbb{E}_{p(\vtheta_0\mid\vtheta_t,\x)}[\ratio(\vtheta_0)] &\approx \nabla_{\vtheta_t} \log \int \sum_{i=1}^K w_i \mathcal{N}(\vtheta_0 | \vmu_i, \vSigma_i) \mathcal{N}(\vtheta_0 | \vmu_{0|t}(\vtheta_t), \vSigma_{0|t}) d\vtheta_0, \\
&= \nabla_{\vtheta_t} \log \sum_{i=1}^K w_i \int \mathcal{N}(\vmu_i |\vtheta_0, \vSigma_i) \mathcal{N}(\vtheta_0 | \vmu_{0|t}(\vtheta_t), \vSigma_{0|t}) d\vtheta_0.
\end{align}
The step above uses the symmetry property of Gaussian distributions: if $\va \sim \mathcal{N}(\vmu, \vSigma)$ then $\vmu \sim \mathcal{N}(\va, \vSigma)$. This allows us to swap $\vtheta_0$ and $\vmu_i$ in the first Gaussian. Furthermore, using the standard result for the convolution of two Gaussian distributions:
\begin{equation}
\int \mathcal{N}(\x|\vmu_1, \vSigma_1)\mathcal{N}(\vmu_1|\vmu_2, \vSigma_2)d\vmu_1 = \mathcal{N}(\x|\vmu_2, \vSigma_1 + \vSigma_2),
\end{equation}
we get
\begin{align}
\nabla_{\vtheta_t} \log \mathbb{E}_{p(\vtheta_0\mid\vtheta_t,\x)}[\ratio(\vtheta_0)] &\approx  \nabla_{\vtheta_t} \log \sum_{i=1}^K w_i \mathcal{N}(\vmu_i | \vmu_{0|t}(\vtheta_t), \vSigma_i + \vSigma_{0|t}).
\end{align}

For notational convenience, we define $\widetilde{\vSigma}_i = \vSigma_i + \vSigma_{0|t}$ continuing with the derivation:
\begin{align}
&= \nabla_{\vtheta_t} \log \sum_{i=1}^K w_i \mathcal{N}(\vmu_i | \vmu_{0|t}(\vtheta_t), \widetilde{\vSigma}_i), \\
&= \frac{\nabla_{\vtheta_t} \sum_{i=1}^K w_i \mathcal{N}(\vmu_i | \vmu_{0|t}(\vtheta_t), \widetilde{\vSigma}_i)}{\sum_{j=1}^K w_j \mathcal{N}(\vmu_j | \vmu_{0|t}(\vtheta_t), \widetilde{\vSigma}_j)} \quad \text{(chain rule)}, \\
&= \frac{\sum_{i=1}^K w_i \mathcal{N}(\vmu_i | \vmu_{0|t}(\vtheta_t), \widetilde{\vSigma}_i) \nabla_{\vtheta_t} \log \mathcal{N}(\vmu_i | \vmu_{0|t}(\vtheta_t), \widetilde{\vSigma}_i)}{\sum_{j=1}^K w_j \mathcal{N}(\vmu_j | \vmu_{0|t}(\vtheta_t), \widetilde{\vSigma}_j)} \quad \text{(since $\nabla f = f\nabla \log f$)}, \\
&= \frac{\sum_{i=1}^K w_i \mathcal{N}(\vmu_i | \vmu_{0|t}(\vtheta_t), \widetilde{\vSigma}_i) \nabla_{\vtheta_t} \left( -\frac{1}{2} (\vmu_{0|t}(\vtheta_t) - \vmu_i)^\top \widetilde{\vSigma}_i^{-1} (\vmu_{0|t}(\vtheta_t) - \vmu_i) \right)}{\sum_{j=1}^K w_j \mathcal{N}(\vmu_j | \vmu_{0|t}(\vtheta_t), \widetilde{\vSigma}_j),} \\
&= \frac{\sum_{i=1}^K w_i \mathcal{N}(\vmu_i | \vmu_{0|t}(\vtheta_t), \widetilde{\vSigma}_i) (\vmu_i - \vmu_{0|t}(\vtheta_t))^{\mathbf{T}} \widetilde{\vSigma}_i^{-1} \nabla_{\vtheta_t} \vmu_{0|t}(\vtheta_t)}{\sum_{j=1}^K w_j \mathcal{N}(\vmu_j | \vmu_{0|t}(\vtheta_t), \widetilde{\vSigma}_j)}.
\end{align}
Finally, with the following definitions:
\begin{align}
\widetilde{\vSigma}_i & = \vSigma_i + \vSigma_{0|t} \\
\tilde w_i & = \frac{w_i \mathcal{N}(\vmu_i \mid \vmu_{0|t}(\vtheta_t, \x), \widetilde{\vSigma}_i)}{{\sum_{j=1}^K w_j \mathcal{N}(\vmu_j \mid \vmu_{0|t}(\vtheta_t, \x), \widetilde{\vSigma}_j)}}
\end{align}
we obtain
\begin{equation}
\nabla_{\vtheta_t} \log \mathbb{E}_{p(\vtheta_0\mid\vtheta_t,\x)}[\ratio(\vtheta_0)] \approx \sum_{i=1}^{K} \tilde w_i (\vmu_i - \vmu_{0|t}(\vtheta_t))^\top \widetilde{\vSigma}_i^{-1} \nabla_{\vtheta_t} \vmu_{0|t}(\vtheta_t), \\ 
\end{equation}
which is Eq.~\ref{eq:priorw} in the main text.

\subsection{Proof of Proposition 1}
\label{app:proof1}

This proposition is fully derived in the main paper, we only include here a natural assumption to guarantee the existence of the ratio (see also \cref{app:prior_diagnostic}).

Let $\ratio(\vtheta) \equiv \frac{q(\vtheta)}{\pt(\vtheta)}$ the prior ratio. We assume that $\pt(\vtheta) = 0 \rightarrow q(\vtheta) = 0$ almost everywhere, \ie $q(\vtheta)$ needs to be zero outside the support of $\pt(\theta)$, except for a set of measure zero.

\begin{proposition*}[Proposition~\ref{eq:prop1}]
Let the posterior under the original prior be $p(\vtheta \mid \x) \propto \pt(\vtheta) p(\x \mid \vtheta)$, and let the target posterior---the posterior under the new prior---be $q(\vtheta \mid \x) \propto q(\vtheta) p(\x \mid \vtheta)$. Then, sampling from $q(\vtheta \mid \x)$ is equivalent to sampling from $\ratio(\vtheta) p(\vtheta \mid \x)$ with $\ratio(\vtheta)$ the prior ratio.
\end{proposition*}

\begin{proof}
We can rewrite the target posterior \(q(\vtheta\mid\x)\) as
\[
q(\vtheta \mid \x) \propto q(\vtheta)p(\x \mid \vtheta) = \frac{q(\vtheta)}{\pt(\vtheta)}\, \pt(\vtheta)p(\x \mid \vtheta) \propto \frac{q(\vtheta)}{\pt(\vtheta)}\, p(\vtheta \mid \x) = \ratio(\vtheta) p(\vtheta\mid\x),
\]
where the prior ratio $\ratio(\vtheta) \equiv \frac{q(\vtheta)}{\pt(\vtheta)}$ takes the role of an importance weighing function, and the above equality applies almost everywhere.
\end{proof}

\subsection{Proof of Proposition 2}
\label{app:proof2}

We provide here the extended statement, with explicit regularity conditions, and then the full proof.

\begin{proposition*}[Proposition~\ref{eq:prop2}]
    As $t,\sigma(t) \to 0$, the approximation $\mathcal{N}(\vtheta_0\mid \vtheta_t + \sigma(t)^2\nabla_{\vtheta_t} \log p(\vtheta_t), \frac{\sigma(t)^2}{1 + \sigma(t)^2} \mathbf{I})$ converges to the true $p(\vtheta_0\mid \vtheta_t)$, assuming that $p(\vtheta_0)$ is two times differentiable everywhere and $\nabla_{\vtheta_0}^2 p(\vtheta_0)$ is bounded.
\end{proposition*}

\paragraph{Notation} To be more precise about the different distributions involved, let us denote $p_0(\vtheta_0)$ as the marginal distribution of the clean data, $p_t(\vtheta_t)$ the marginal distribution at noise level $\sigma(t)$, $p_{t\mid 0}(\vtheta_t\mid \vtheta_0) = \mathcal{N}(\vtheta_t\mid \vtheta_0, \sigma(t)^2 \mathbf{I})$ and $p_{0\mid t}(\vtheta_0 \mid \vtheta_t) = \frac{p_{t\mid 0}(\vtheta_t\mid \vtheta_0) p_0(\vtheta_0)}{p_t(\vtheta_t)}$. Let us also drop the dependence $t$ from the notation $\sigma(t)$, and simply refer to $\sigma$, since we do not have to refer to derivatives of $\sigma(t)$ in the proof. 

\begin{proof}

First, note that as $\sigma\to 0$, $\frac{\sigma^2}{1 + \sigma^2} = \sigma^2 + O(\sigma^4)$ via a Taylor expansion. In other words, our denoising variance is $\sigma^2 \mathbf{I}$ up to fourth-order or higher corrections in $\sigma$, which become negligible at low $\sigma$. We can thus first show the result for $\mathcal{N}(\vtheta_0\mid \vtheta_t + \sigma(t)^2\nabla_{\vtheta_t} \log p(\vtheta_t), \sigma^2 \mathbf{I})$, and at the end we will see that it will trivially transfer to the $\frac{\sigma^2}{1 + \sigma^2} \mathbf{I}$ case as well.

\paragraph{Rescaled coordinates } Set
\begin{equation}
    \rvs = \frac{\vtheta_0 - \vtheta_t}{\sigma}, \quad\phi_d(\rvs) = (2\pi)^{-d/2} \exp\left(-\frac{1}{2}||\rvs||^2\right)
\end{equation}
so that we can express the forward noising distribution as
\begin{equation}
    p_{t\mid 0}(\vtheta_t\mid \vtheta_0) = \sigma^{-d}\phi_d(\rvs)
\end{equation}

\paragraph{Taylor expansion for the true posterior density} Since $p_0(\vtheta_0)$ is two times differentiable everywhere, the multivariate Taylor theorem gives
\begin{equation}
    p_0(\vtheta_t + \sigma \rvs) = p_0(\vtheta_t) + \sigma \nabla_{\vtheta_t} p_0(\vtheta_t)^\top \rvs + \frac{\sigma^2}{2} \rvs^\top \nabla_{\rvb(\rvs)}^2p_0(\rvb(\rvs)) \rvs \label{eq:taylor_for_p0}
\end{equation}
where $\rvb(\rvs)\in\{t\rvs: 0\leq t \leq 1\}$ is some point between $\vtheta_t$ and $\vtheta_0$, chosen separately for each $\rvs$. 

Hence, the unnormalised posterior density can be written as
\begin{equation}
    \sigma^{-d}\phi_d(\rvs) \left[ p_0(\vtheta_t) + \sigma \nabla_{\vtheta_t} p_0(\vtheta_t)^\top \rvs + \frac{\sigma^2}{2} \rvs^\top \nabla_{\rvb(\rvs)}^2p_0(\rvb(\rvs)) \rvs \right]
\end{equation}

The normalizing constant can be expanded as
\begin{align}
    p_t(\vtheta_t) &= \int \sigma^{-d}\phi_d(\rvs) \left[ p_0(\vtheta_t) + \sigma \nabla_{\vtheta_t} p_0(\vtheta_t)^\top \rvs + \frac{\sigma^2}{2} \rvs^\top \nabla_{\rvb(\rvs)}^2p_0(\rvb(\rvs)) \rvs \right] \sigma^d d \rvs\\
    &= p_0(\vtheta_t) + \underbrace{\sigma\int \phi_d(\rvs)  \nabla_{\vtheta_t} p_0(\vtheta_t)^\top \rvs d\rvs}_{=0} + \sigma^2\underbrace{\frac{1}{2}\int \phi_d(\rvs)  \rvs^\top \nabla_{\rvb(\rvs)}^2p_0(\rvb(\rvs)) \rvs d\rvs}_{=C_1(\vtheta_t)}\\
    &= p_0(\vtheta_t) + \sigma^2 C_1(\vtheta_t)
\end{align}
where the odd-powered term in the Taylor expansion goes to zero due to the symmetry of the integral. The integral in the second-order term is finite since we assume $\nabla_{\rvb(\rvs)}^2p_0(\rvb(\rvs))$ is finite everywhere. 

The reciprocal of the normalizing constant is
\begin{align}
    \frac{1}{p_0(\vtheta_t) + \sigma^2 C_1(\vtheta_t)} = \frac{1}{p_0(\vtheta_t)} - \frac{1}{p_0(\vtheta_t)^2} \sigma^2 C_1(\vtheta_t) + O(\sigma^4)
\end{align}
obtained with the Taylor series $\frac{1}{a + \varepsilon} = \frac{1}{a} - \frac{1}{a^2}\varepsilon + O(\varepsilon^2)$. Thus, the normalised posterior can be expressed as
\begin{align}
    &\frac{\sigma^{-d} \phi_d(\rvs) \left[ p_0(\vtheta_t) + \sigma \nabla_{\vtheta_t} p_0(\vtheta_t)^\top \rvs + \frac{\sigma^2}{2} \rvs^\top \nabla_{\rvb(\rvs)}^2p_0(\rvb(\rvs)) \rvs \right]}{p_0(\vtheta_t) + \sigma^2 C_1(\vtheta_t)} \\
    &= \sigma^{-d}\phi_d(\rvs) \left[ p_0(\vtheta_t) + \sigma \nabla_{\vtheta_t} p_0(\vtheta_t)^\top \rvs + \frac{\sigma^2}{2} \rvs^\top \nabla_{\rvb(\rvs)}^2p_0(\rvb(\rvs)) \rvs \right]\left(\frac{1}{p_0(\vtheta_t)} - \frac{1}{p_0(\vtheta_t)^2} \sigma^2 C_1(\vtheta_t) + O(\sigma^4) \right)\\
    &= \sigma^{-d}\phi_d(\rvs) \left[1 + \sigma  \frac{\nabla_{\vtheta_t} p_0(\vtheta_t)^\top}{p_0(\vtheta_t)} \rvs + \sigma^2\left(\frac{1}{2} \rvs^\top \nabla_{\rvb(\rvs)}^2p_0(\rvb(\rvs))\rvs - \frac{1}{p_0(\vtheta_t)} C_1(\vtheta_t) \right) + O(\sigma^3)\right]\\
    &= \sigma^{-d}\phi_d(\rvs) \left[1 + \sigma \nabla_{\vtheta_t} \log p_0(\vtheta_t)^\top \rvs + \sigma^2 C_2(\rvs, \vtheta_t) + O(\sigma^3)\right].\label{eq:normalized_posterior_expansion}
\end{align}
Here, we abstract away $C_2(\rvs, \vtheta_t)$, since later in the proof we only care about the fact that it has a finite integral $\int |\phi_d(\rvs)C_2(\rvs, \vtheta_t)|d\rvs$. 

\paragraph{Taylor expansion for our approximate posterior density} Note that the score $\nabla_{\vtheta_t} \log p_t(\vtheta_t) = \frac{\nabla_{\vtheta_t} p_t(\vtheta_t)}{p_t(\vtheta_t}$, and we can reuse our earlier calculation:
\begin{align}
&\nabla_{\vtheta_t} \log p_t(\vtheta_t) = \frac{\nabla_{\vtheta_t} p_0(\vtheta_t) + \sigma^2 \nabla_{\vtheta_t} C_1(\vtheta_t)}{p_0(\vtheta_t) + \sigma^2 C_1(\vtheta_t)} \\
&= \left(\nabla_{\vtheta_t} p_0(\vtheta_t) + \sigma^2 \nabla_{\vtheta_t} C_1(\vtheta_t)\right) \left(\frac{1}{p_0(\vtheta_t)} - \frac{1}{p_0(\vtheta_t)^2} \sigma^2 C_1(\vtheta_t) + O(\sigma^4) \right)\\
&= \frac{\nabla_{\vtheta_t} p_0(\vtheta_t)}{p_0(\vtheta_t} + \sigma^2 (\nabla_{\vtheta_t} C_1(\vtheta_t) -  \frac{\nabla_{\vtheta_t} p_0(\vtheta_t)}{p_0(\vtheta_t)^2} C_1(\vtheta_t)) + O(\sigma^4)\\
&= \nabla_{\vtheta_t} \log p_0(\vtheta_t) + O(\sigma^2).\label{eq:taylor_for_our_post_mean}
\end{align}
This yields a formula for our posterior mean:
\begin{align}
    \vtheta_t + \sigma^2 \nabla_{\vtheta_t} \log p_t(\vtheta_t) = \vtheta_t + \sigma^2 \nabla_{\vtheta_t} \log p_0(\vtheta_t) + O(\sigma^4).
\end{align}
Now, define:
\begin{align}
    q(\vtheta_0) &= \mathcal{N}(\vtheta_0 \mid \vtheta_t + \sigma^2 \nabla_{\vtheta_t}\log p(\vtheta_t), \sigma^2 \mathbf{I}) \\
    &= \sigma^{-d}\phi_d(s - \sigma \nabla_{\vtheta_t}\log p(\vtheta_t)), \quad s=\frac{\vtheta_0 - \vtheta_t}{\sigma}.
\end{align}

Then we can Taylor expand our posterior approximation in terms of the shift $\rvs \to \rvs - \sigma \nabla_{\vtheta_t}\log p(\vtheta_t)$:
\begin{align}
    q(\sigma\rvs + \vtheta_t) &= \sigma^{-d}\phi_d(\rvs)[1 + \sigma \nabla_{\vtheta_t}\log p_t(\vtheta_t)^\top \rvs \notag \\
    &\quad + \frac12\sigma^2 \nabla_{\vtheta_t} \log p_t(\vtheta_t)^\top (\rvs \rvs^\top - I) \nabla_{\vtheta_t} \log p_t(\vtheta_t) + O(\sigma^3)]\\
    &= \sigma^{-d}\phi_d(\rvs)[1 + \sigma \nabla_{\vtheta_t}\log p_0(\vtheta_t)^\top \rvs \notag \\
    &\quad + \frac12\sigma^2 \nabla_{\vtheta_t} \log p_0(\vtheta_t)^\top (\rvs \rvs^\top - I) \nabla_{\vtheta_t} \log p_0(\vtheta_t) + O(\sigma^3)]\\
    &= \sigma^{-d}\phi_d(\rvs)[1 + \sigma \nabla_{\vtheta_t}\log p_0(\vtheta_t)^\top \rvs + \sigma^2 C_3(\rvs, \vtheta_t) + O(\sigma^3)].
    \label{eq:our_posterior_expansions}
\end{align}
where in the second line we used \cref{eq:taylor_for_our_post_mean} and ignored resulting terms that are higher order than $\sigma^2$. 

\paragraph{Total variation bound} 

The total variation distance is 
\begin{align}
    \|p_{0\mid t}-q\|_{\TV}
    = \frac12\int_{\mathbb{R}^{d}}\!|p_{0\mid t}(\vtheta_0)-q(\vtheta_0)|\,d\vtheta_0.
\end{align}
Now, plugging in \cref{eq:normalized_posterior_expansion} and \cref{eq:our_posterior_expansions}, we get (note $d\vtheta_0 = \sigma^d d\rvs$)
\begin{align}
    \frac12 \int \big| &\sigma^{-d} \phi_d(\rvs) \left[1 + \sigma \nabla_{\vtheta_t} \log p_0(\vtheta_t)^\top \rvs + \sigma^2 C_2(\rvs, \vtheta_t) + O(\sigma^3)\right] \\
    &- \sigma^{-d}\phi_d(\rvs)[1 + \sigma \nabla_{\vtheta_t}\log p_0(\vtheta_t)^\top \rvs + \sigma^2 C_3(\rvs, \vtheta_t) + O(\sigma^3)] \big| \sigma^d d\rvs\\
    &= \frac12 \int \big| \sigma^2 \phi_d(\rvs) (C_2(\rvs, \vtheta_t) - C_3(\rvs, \vtheta_t) + O(\sigma^3)) \big| d\rvs \leq \sigma^2 C_4(\vtheta_t) + O(\sigma^3).
\end{align}
Thus, the total variation distance converges at a rate $\sigma^2$. To see the final step more clearly, we can bound the integral with the triangle inequality:
\begin{align}
    &\int \big| \sigma^2 \phi_d(\rvs) (C_2(\rvs, \vtheta_t) - C_3(\rvs, \vtheta_t) + O(\sigma^3)) \big| d\rvs \\
    &\leq \sigma^2 \int \big| \phi_d(\rvs) C_2(\rvs, \vtheta_t) \big|d\rvs + \sigma^2 \int \big| \phi_d(\rvs) C_3(\rvs, \vtheta_t) \big|d\rvs + O(\sigma^3).\\
\end{align}
Recall the definitions of $C_3, C_2$ and $C_1$:
\begin{align}
    C_1(\vtheta_t) &= \frac{1}{2}\int \phi_d(\rvs)  \rvs^\top \nabla_{\rvb(\rvs)}^2p_0(\rvb(\rvs)) \rvs d\rvs\\
    C_2(\rvs, \vtheta_t) &=\left(\frac{1}{2} \rvs^\top \nabla_{\rvb(\rvs)}^2p_0(\rvb(\rvs))\rvs - \frac{1}{p_0(\vtheta_t)} C_1(\vtheta_t)\right)\\
    \frac12 C_3(\rvs, \vtheta_t) &=\nabla_{\vtheta_t} \log p_0(\vtheta_t)^\top (\rvs \rvs^\top - I) \nabla_{\vtheta_t} \log p_0(\vtheta_t).
\end{align}
We can see that the terms depending on $\sigma^2$ are finite, due to the assumption that the Hessian of $p_0$ is finite everywhere. Thus, 
\begin{align}
    C_4(\vtheta_t) = \int \big| \phi_d(\rvs) (C_2(\rvs, \vtheta_t) \big|d\rvs + \int \big| \phi_d(\rvs) (C_3(\rvs, \vtheta_t) \big|d\rvs.
\end{align}
Finally, we note that $\frac{\sigma^2}{1 +\sigma^2} \to \sigma^2$ as $\sigma\to 0$. Thus, the the convergence result here applies to the case where the posterior covariance approximation is $\frac{\sigma^2}{1 + \sigma^2}\mathbf{I}$. 
\end{proof}

\section{Experimental details} 
\label{app:experimental_details}

This section provides extended methodological and experimental details.
We describe the simulator models used in our experiments (\cref{app:simulators}), the training setup for each method (\cref{app:training_details}), the testing procedure (\cref{app:testing-setup}), and the statistical methodology for evaluation (\cref{app:statistical_methods}).

\subsection{Simulators} \label{app:simulators}

\textbf{Two Moons} \citep{greenberg2019automatic} is a widely used benchmark in SBI, designed as a two-dimensional task that presents a posterior distribution with both global (bimodality) and local (crescent shape) structure. For a given parameter vector $\vtheta = (\theta_1, \theta_2) \in \mathbb{R}^2$, the simulator generates data $\vx \in \mathbb{R}^2$ according to the following process:
\begin{align*} 
a &\sim U(-\pi/2, \pi/2), \\
r &\sim \mathcal{N}(0.1, 0.01^2), \\ 
p &= (r \cos(a) + 0.25, r \sin(a)), \\ 
\vx^T &= p + \left( \frac{-|\theta_1 + \theta_2|}{\sqrt{2}}, \frac{-\theta_1 + \theta_2}{\sqrt{2}} \right). 
\end{align*}
The training prior we use is $p_{\text{train}}(\vtheta)=U([-1, 1]^2)$.
To obtain ground-truth posterior samples, we perform rejection sampling using the new prior as a proposal distribution.
Rejection sampling relies on finding a constant $M$ such that $f(\vtheta) \leq Mq(\vtheta)$, where $f(\vtheta) = p(\vx \mid \vtheta) \cdot q(\vtheta)$ is the target density and $q(\vtheta)$ is the proposal density.
In this Two Moons model, we set $M$ to be the upper bound of the likelihood $p(\vx \mid \vtheta)$ which is achieved at $r=0.1$.

\textbf{Ornstein-Uhlenbeck Process (OUP)} \citep{uhlenbeck1930theory} is a well-established stochastic process frequently applied in financial mathematics and evolutionary biology for modeling mean-reverting dynamics \citep{uhlenbeck1930theory}. The model is defined as:
\begin{equation*} y_{t+1} = y_t + \Delta y_t, \quad \Delta y_t = \theta_1 \left[\exp(\theta_2) - y_t \right] \Delta t + 0.5w, \quad \text{ for } t = 1, \ldots, T, 
\end{equation*}
where we set $ T = 25 $, $\Delta t = 0.2 $, and initialize $ x_0 = 10 $. The noise term follows a Gaussian distribution, $w \sim \mathcal{N}(0, \Delta t)$. The original prior is a uniform distribution, $ U([0, 2] \times [-2, 2]) $, over the latent parameters $ \vtheta = (\theta_1, \theta_2) $. For numerical convenience, we reparameterize the parameter space by mapping the original parameters to $[-1, 1]^2$, yielding $p_{\text{train}}(\vtheta)=U([-1, 1]^2)$. We perform inference in this normalized parameter space and later rescale the sampled parameters to the original space during simulation.
The simulated data are also normalized using standardization.
We use normalized parameter-data pairs $(\tilde{\vtheta}, \tilde{\vx})$ to train all amortized inference models.

Since the OUP likelihood is implicit, obtaining ground-truth posterior samples is intractable---the kind of problem that requires simulation-based inference. 
As a practical surrogate, we adopt a neural posterior estimation (NPE) model \citep{greenberg2019automatic}---trained on \emph{one million} simulations (compared to the 10,000 simulations used for our diffusion-based inference model)---to serve as ground truth. 
We then applied simulation-based calibration (SBC) \citep{talts2018validating} to verify that the NPE model remained well-calibrated across virtually all observation seeds for each new prior.

\textbf{Turin} \citep{turin1972statistical} is a widely used time-series model for simulating radio wave propagation \citep{turin1972statistical, pedersen2019stochastic}. This model generates high-dimensional, complex-valued time-series data and is governed by four key parameters: $G_0$ determines the reverberation gain, $T$ controls the reverberation time, $\lambda_0$ defines the arrival rate of the point process, and $\sigma^2_N$ represents the noise variance.

The model assumes a frequency bandwidth of $B=0.5$ GHz and simulates the transfer function $H_k$ at $N_s = 101$ evenly spaced frequency points. The observed transfer function at the $k$-th frequency point, $Y_k$, is defined as: \begin{equation*} 
Y_k = H_k + W_k, \quad k = 0, 1, \ldots, N_s - 1, 
\end{equation*} 
where $W_k$ represents additive zero-mean complex Gaussian noise with circular symmetry and variance $\sigma^2_W$. The transfer function $H_k$ is expressed as: 
\begin{equation*} 
H_k = \sum_{l=1}^{N_{\text{points}}} \alpha_l \exp(-j 2 \pi \Delta f k \tau_l), 
\end{equation*} 
where the time delays $\tau_l$ are sampled from a homogeneous Poisson point process with rate $\lambda_0$, and the complex gains $\alpha_l$ are modeled as independent zero-mean complex Gaussian random variables. The conditional variance of the gains is given by: 
\begin{equation*} 
\mathbb{E}[\vert \alpha_l \vert^2 | \tau_l] = \frac{G_0 \exp(-\tau_l / T)}{\lambda_0}. 
\end{equation*}

To obtain the time-domain signal $\tilde{y}(t)$, an inverse Fourier transform is applied: 
\begin{equation*} 
\tilde{y}(t) = \frac{1}{N_s} \sum_{k=0}^{N_s - 1} Y_k \exp(j 2 \pi k \Delta f t), 
\end{equation*} 
where $\Delta f = B / (N_s - 1) $ represents the frequency spacing. Finally, the real-valued output is computed by taking the absolute square of the complex signal and applying a logarithmic transformation: 
\begin{equation*} 
y(t) = 10 \log_{10}(|\tilde{y}(t)|^2). 
\end{equation*}

In Turin, the true parameter bounds are: $
G_0 \in [10^{-9}, 10^{-8}], \quad
T \in [10^{-9}, 10^{-8}], \quad
\lambda_0 \in [10^{7}, 5 \times 10^{9}], \quad
\sigma^2_N \in [10^{-10}, 10^{-9}]
$.
We follow a similar normalization setup as in OUP.
First, we define the training prior as $p_\mathrm{train}(\vtheta) = U([0, 1]^4)$ and rescale the sampled parameters $\tilde{\vtheta}$ to the original space using the true parameter bounds. 
Then, we normalize the simulator outputs using standardization and use normalized $(\tilde{\vtheta}, \tilde{\vx})$ pairs to train all inference models.

The Turin likelihood is also implicit. 
Therefore, we use a similar setup to the OUP case by training an NPE model with one million simulations and validating its reliability using SBC.

\textbf{Gaussian Linear} \citep{lueckmann2021benchmarking} is a standard SBI benchmark task used to infer the mean of a multivariate Gaussian distribution when the covariance is fixed. In this model, both the parameters $\vtheta$ and the data $\vx$ are 10-dimensional vectors. The simulator is defined as:
\begin{align*}
    \vx | \vtheta \sim \mathcal{N}(\vtheta, \Sigma_s),
\end{align*}
where $\Sigma_s=0.1 \cdot \mathbf{I}_{10}$ with $\mathbf{I}_{10}$ is the 10-dimensional identity matrix. The training prior for the parameters $\vtheta$ is a 10-dimensional Gaussian distribution $p_{\text{train}}(\vtheta)= \mathcal{N}(\mathbf{0}, 0.1 \cdot \mathbf{I}_{10})$.

We also consider a 20-dimensional variant of the Gaussian Linear model, which follows the same fundamental setup as its 10-dimensional counterpart. In this version, both the parameters $\vtheta$ and the data $\vx$ are 20-dimensional vectors.

In all experiments, the test-time priors are constructed as Gaussian or mixture distributions with Gaussian components.
Consequently, each new prior has a closed-form posterior, which we use as the ground-truth posterior.

\textbf{Bayesian Causal Inference model (BCI)} is a common model in computational cognitive neuroscience to represent how the brain determines whether multiple sensory stimuli originate from a common source~\citep{kording2007causal}. This BCI model simulates an observer's responses in an audiovisual localization task. The observer is presented with auditory ($S_A$) and visual ($S_V$) spatial cues---expressed as horizontal location in degrees of visual angle---and must report the perceived location of one of them. The model assumes the observer performs Bayesian causal inference to determine whether the cues originate from a common source or independent sources, and then makes a model-averaged estimate of the target stimulus location~\citep{kording2007causal,acerbi2018bayesian}.

In this BCI model we consider five underlying physical parameters: standard deviation of visual sensory noise ($\sigma_V$), standard deviation of auditory sensory noise ($\sigma_A$), standard deviation of the Gaussian spatial prior over source locations ($\sigma_s$), prior probability that auditory and visual cues share a common cause
($p_{\text{same}}$), and standard deviation of motor noise in the response ($\sigma_m$).
We assume the mean of the Gaussian spatial prior $\mu_p$ is set to 0 (central tendency), and a small lapse rate $\lambda = 0.02$, the probability of making a random response. Additionally, a fixed auditory rescaling factor $\rho_A = 4/3$ is applied to auditory stimulus locations to account for audiovisual adaptation in this experiment.

The simulation process for each trial $i$ (out of 98 fixed trials) proceeds as follows:
\begin{enumerate}
    \item Given true stimulus locations $S_{V,i}$ and $S_{A,i}$, and the modality to be reported (visual or auditory).
    \item Sensory measurements $x_{V,i}$ and $x_{A,i}$ are drawn:
    \begin{align*}
    x_{V,i} &\sim \mathcal{N}(S_{V,i}, \sigma_V^2) \\
    x_{A,i} &\sim \mathcal{N}(\rho_A S_{A,i}, \sigma_A^2)
    \end{align*}
    \item The observer combines these measurements with a spatial prior $p(s) = \mathcal{N}(s; \mu_p, \sigma_s^2)$.
    \item Causal inference is performed to compute the posterior probability of a common source, $P(C=1 | x_{V,i}, x_{A,i})$, using the prior $p_{\text{same}}$ and the likelihoods of the measurements under common-cause ($C=1$) and independent-causes ($C=2$) hypotheses.
    \item A model-averaged estimate of the relevant source location, $\hat{s}_i$, is computed:
    \begin{equation*}
    \begin{split}
    \hat{s}_i &= P(C=1|x_{V,i}, x_{A,i}) \cdot \mu_{C=1}(x_{V,i}, x_{A,i}) \\
    &\quad + \left(1-P(C=1|x_{V,i}, x_{A,i})\right) \cdot \mu_{C=2}(x_{V,i}, x_{A,i}, \text{report\_modality}_i),
    \end{split}
    \end{equation*}
    where $\mu_{C=1}$ and $\mu_{C=2}$ are the posterior mean estimates under the respective causal hypotheses.
    \item A noisy motor response $R'_i$ is generated: $R'_i \sim \mathcal{N}(\hat{s}_i, \sigma_m^2)$.
    \item With probability $\lambda$, the response $R'_i$ is replaced by a lapse response drawn uniformly from $U(-45^\circ, 45^\circ)$. The final response is $R_i$.
\end{enumerate}
The output data $\vx$ is a 98-dimensional vector $(R_1, \ldots, R_{98})$. These responses correspond to a fixed experimental design: a $7 \times 7$ Cartesian grid of stimulus locations $S_V, S_A \in \{-15, -10, -5, 0, 5, 10, 15\}^\circ$. For each of the 49 unique $(S_V, S_A)$ pairs, one visual-report trial and one auditory-report trial are included, totaling 98 trials.

The model's five parameters are internally represented in an unconstrained space: $\sigma_V, \sigma_A, \sigma_s, \sigma_m$ are parameterized by their logarithms, and $p_{\text{same}}$ by its logit,
denoted by the vector $\vtheta$:
\begin{equation}
\vtheta = (\theta_1, \theta_2, \theta_3, \theta_4, \theta_5) = (\log \sigma_V, \log \sigma_A, \log \sigma_s, \log \sigma_m, \text{logit } p_{\text{same}}).
\end{equation}
The original prior over these 5 unconstrained parameters follows a broad multivariate Gaussian distribution informed by the literature, with independent components:
\begin{equation}
    \theta_1, \theta_2 \sim \mathcal{N}(\log 2, 0.35^2), \;
    \theta_3 \sim \mathcal{N}(\log 5, 0.5^2), \;
    \theta_4 \sim \mathcal{N}(\log 0.3, 0.35^2), \;
    \theta_5 \sim \mathcal{N}(0, 1^2).
\end{equation}

For numerical convenience, we reparameterize the parameter space by defining the training prior as a 5-dimensional Gaussian distribution $p_\mathrm{train}(\vtheta) = \mathcal{N}(\mathbf{0}, \mathbf{I}_5)$. 
Similar to the setup in OUP and Turin, we rescale the sampled parameters to the original space during simulation and normalize the simulated data using standardization. 
We use normalized parameter-data pairs $(\tilde{\vtheta}, \tilde{\vx})$ to train all amortized inference models.

The likelihood function for this model involves integrating over latent sensory measurements and is computationally intensive. To obtain ground-truth posterior samples for BCI, we run the Variational Bayesian Monte Carlo (VBMC) algorithm \citep{acerbi2018variational,huggins2023pyvbmc}.
Using VBMC's internal diagnostics, we retain ten reliable variational posteriors and merge them via posterior stacking \citep{silvestrin2025stacking} to obtain the final ground-truth posterior.

\cref{tab:benchmarks_overview_posterior} summarizes the key properties of the simulator models in our experiments.

\begin{table}
    \centering
    \caption{Characteristics of the simulator models.}
    \label{tab:benchmarks_overview_posterior}
    \resizebox{0.5\textwidth}{!}{%
    \begin{tabular}{lcccc}
        \toprule
        Model & $\dim(\vtheta)$ & $\dim(\vx)$ & $\pt(\vtheta)$ \\
        \midrule
        Two-Moons & 2 & 2 & Uniform \\
        OUP   & 2 & 25  & Uniform  \\
        Turin & 4 & 101 & Uniform  \\
        Gaussian Linear 10D & 10 & 10 & Gaussian \\
        Gaussian Linear 20D & 20 & 20 & Gaussian \\
        BCI   & 5 & 98  & Gaussian \\
        \bottomrule
    \end{tabular}%
    }
\end{table}

\subsection{Training setup} \label{app:training_details}

For all methods that require training, we use 10,000 simulated datasets from the simulator to train the model. Note that PriorGuide is a test-time technique that does not require separate training. For PriorGuide, we use the same base diffusion model as Simformer. Details on the model configurations and dataset setups are provided below.

\paragraph{Simformer} We adopt a similar setup as the Simformer paper \cite{gloeckler2024all}, using the Variance Exploding Stochastic Differential Equation (VE-SDE) technique \cite{song2021score}. It is defined by
\begin{equation}
f_{\mathrm{VE-SDE}}(x, t) \;=\; 0, \qquad 
g_{\mathrm{VE-SDE}}(t) \;=\; \sigma_{\min} \Bigl(\tfrac{\sigma_{\max}}{\sigma_{\min}}\Bigr)^{t} \sqrt{2 \log \!\tfrac{\sigma_{\max}}{\sigma_{\min}}} 
\tag{4}.
\end{equation}

Throughout all experiments, we set $\sigma_{\max} = 15, \quad \sigma_{\min} = 10^{-4},$ and run the process over the time interval \(t\in[10^{-5},\,1]\). We use a transformer configuration similar to Simformer \cite{gloeckler2024all}, with 6 layers, 4 heads (size 10), a token dimension 40, and a 128-dimensional Gaussian Fourier embedding for diffusion time. MLP blocks use a hidden dimension of 150.
In all experiments, the condition mask was sampled per batch by uniformly selecting one of the following: joint, posterior, likelihood, or two random masks. Random masks were drawn from Bernoulli distributions with $p = 0.3$ and $p = 0.7$, respectively. We use the same setup for all of the simulators. 

We train all the Simformer models using a batch size of 1,000 and an initial learning rate of 0.001. A linear learning rate schedule is used to decay the learning rate to $1 \times 10^{-6}$, starting at half of the total number of training steps and completing by the final step. The optimizer combines adaptive gradient clipping with a maximum norm of 10.0 and the Adam optimizer \cite{kingma15}. Early stopping is applied based on validation loss, with the number of training steps constrained to a minimum of 5,000 and a maximum of 100,000 steps.

\paragraph{Amortized Conditioning Engine (ACE)} 
ACE~\citep{chang2025amortized} is a type of Neural Process (NP)~\citep{garnelo2018conditional, nguyen2022transformer,muller2022transformers}, a family of models that learn to perform amortized inference by conditioning on a \textit{context set} of input-output pairs to predict outputs for a \textit{target set} of inputs. Differently from other neural processes which focus on pure data prediction, ACE is trained to condition on, and predict, both data and latent variables (\eg, model parameters in the case of SBI). During training, ACE was provided with simulator parameters that were randomly assigned to either the context or target set, so the model learns to generalize across varying observational conditions. 
For each experiment, a random number of data points $N_d$ were sampled for the context set, with the remaining used as targets. Specifically, $N_d \sim U(1, 2)$ for Two-Moons, $N_d \sim U(7, 25)$ for OUP, $N_d \sim U(30, 101)$ for Turin, $N_d \sim U(3, 10)$ for Gaussian Linear, and $N_d \sim U(29, 98)$ for BCI.

Furthermore, ACE can be trained with a meta-prior (a distribution of priors) over latent variables, where each prior is expressed as a factorized histogram, affording amortized test-time prior adaptation.
For the ACE baseline, we use the same prior generation process as the original ACE paper~\citep{chang2025amortized}, which constructs diverse priors through a hierarchical method over a bounded range. For each latent variable $\theta_l$, the prior is sampled either from a mixture of Gaussians (with 80\% probability) or a uniform distribution. When using a Gaussian mixture, the number of components $K$ is drawn from a geometric distribution with parameter $q = 0.5$, and one of three configurations is chosen at random: shared means with varying standard deviations, varying means with shared standard deviations, or both means and standard deviations varying. Means and standard deviations are sampled from predefined uniform distributions based on the latent variable's range, and mixture weights are drawn from a Dirichlet distribution with $\alpha_0 = 1$. The resulting mixture is discretized into a histogram over $N_\text{bins} = 100$ uniform bins by evaluating CDF differences at bin edges and normalizing. If a uniform prior is selected, equal probability is assigned to all bins. 

For the network setup, we use a configuration similar to that of the ACE paper~\cite{chang2025amortized}. The ACE model has a 64-dimensional embedding, 6 transformer layers, 4 attention heads, and MLP blocks with hidden dimension 128. The output head includes $20$ MLP components with hidden dimension 128. Training was carried out in $5 \times 10^4$ steps with batch size 32 and learning rate $5 \times 10^{-4}$ using the Adam optimizer with the cosine annealing scheduler.

\subsection{Testing procedure}  
\label{app:testing-setup}

\subsubsection{Test-time prior generations}
\label{app:test-time-prior-generations}
We provide additional details about generating test-time priors $q$ in different scenarios.
All the generated test-time priors have passed the OOD diagnostic detailed in \cref{app:prior_diagnostic} with a quantile threshold $\alpha=0.001$.

\paragraph{Training prior definition} For \emph{uniform} training priors, let $\pt$ be defined as a uniform distribution over a $D$-dimensional hypercube $\prod_{i=1}^D [a_i, b_i]$. We define $s_i = (b_i - a_i)/\sqrt{12}$ for each dimension (the standard deviation of a uniform distribution over the range). For \emph{Gaussian} training priors, $\pt$ is a multivariate normal distribution with diagonal covariance, with mean $m_i$ and standard deviation $s_i$ along each dimension.

\paragraph{Target priors}
We first consider the cases where the test-time prior $q$ is a multivariate Gaussian distribution
$q(\vtheta) = \gN(\vtheta \mid \vmu, \vsigma^2\rmI)$.
The procedure to generate a test prior is as follows, separately for each dimension $1 \leq i \leq D$:

\begin{itemize}
    \item We first set the standard deviation $\sigma^\text{mild}_i = 0.5 \cdot s_i$ and $\sigma^\text{strong}_i = 0.2 \cdot s_i$ for mildly informative and strongly informative priors, respectively.

    \item  We define the sampling lower ($L_i$) and upper ($U_i$) bounds for the mean of the target prior based on the class of training prior:
\begin{equation}
L_i = \left\{ \begin{array}{cl} a_i + 3\sigma_i & \text{(Uniform)} \\ m_i - 3 s_i & \text{(Gaussian)} \end{array}, \right. \quad 
U_i = \left\{ \begin{array}{cl} b_i - 3\sigma_i & \text{(Uniform)} \\ m_i + 3 s_i & \text{(Gaussian)} \end{array} \right. .
\end{equation}
    
\item With the bounds chosen, we sample the target prior mean for each coordinate, $\mu_i$, from 
\begin{equation}
\mu_i \sim U(L_i, U_i).
\end{equation}
This procedure ensures that the target prior is well-overlapping with the training prior.
\end{itemize}

Finally, we define
\begin{equation}
    q_\text{mild}(\vtheta) = \gN(\vtheta \mid \vmu, \vSigma^\text{mild}), \quad q_\text{strong}(\vtheta) = \gN(\vtheta \mid \vmu, \vSigma^\text{strong})
\end{equation}
where $\vSigma^\text{mild} = \text{diag}\left[{\sigma_1^\text{mild}}^2, \ldots, {\sigma_D^\text{mild}}^2\right]$ and $\vSigma^\text{strong} = \text{diag}\left[{\sigma_1^\text{strong}}^2, \ldots, {\sigma_D^\text{strong}}^2\right]$.

We then consider $q$ as a mixture Gaussian distribution with two components
\begin{equation}
    q_{\text{mixture}}(\vtheta) = \pi \gN (\vtheta \mid \vmu_1, \vSigma^\text{strong}) + (1 - \pi) \gN (\vtheta \mid \vmu_2, \vSigma^\text{strong}).
\end{equation}
We sample the mean vector for each component using the above procedures with \emph{strongly informative} standard deviations. The mixture weights are sampled from $\pi \sim U(0.2, 0.8)$.

\subsubsection{Posterior inference} 
\label{app:posterior-inference-setup}

For each prior type---\emph{mild}, \emph{strong} and \emph{mixture}---we construct 10 priors $q_i$ following the procedure detailed in \cref{app:test-time-prior-generations}. 
For each target prior $q_i$, we draw 10 parameter vectors $\vtheta_{ij}$. For each parameter vector $\vtheta_{ij}$, we simulate one observed dataset, yielding a pair $(\vtheta_{ij}, \x_{ij})$.
For a given tested method, we perform posterior inference by drawing 1,000 posterior samples $\vtheta$ using the method conditioned on each $\vx_{ij}$, repeating this evaluation across 5 independent training runs.
Consequently, each amortized inference model produces $5 \times 10 \times 10 = 500$ collections of 1,000 $\vtheta$ posterior samples for each prior type.

\paragraph{Remark on prior specification} 
By sampling the ground-truth parameters $\vtheta_{ij}$ from the target prior $q_i$, we are ensuring that these priors are \emph{informative} or well-specified, \ie, the prior used for inference matches the true data generation process~\citep{gelman2013bayesian}. This is \emph{not} a requirement of our method---PriorGuide can steer sampling towards any user-specified prior at test time (under the assumptions mentioned in the paper), steering the inference process towards the true Bayesian posterior under the given prior. Notably, a well-specified prior (or any reasonably specified prior) will \emph{on average} improve inference accuracy, while a badly specified prior might hinder performance---this is a general property of Bayesian inference, which is reflected in PriorGuide. To avoid potential confusion, in this paper we focused on well-specified priors.

\paragraph{Simformer}
With a trained Simformer model $\gM$, we perform the same generation procedure as \citep{gloeckler2024all}, with a number of diffusion steps $N=500$.

\paragraph{PriorGuide}  
Following \cref{alg:prior_guide}, we use the same set of core hyperparameters across all simulators—Two-Moons, OUP, Turin, Gaussian Linear, and BCI—unless otherwise noted.  These include a trained diffusion-based inference model $\gM$ (same as Simformer above), a training prior distribution $p_\text{train}(\vtheta)$, and a test-time prior distribution $q(\vtheta)$. The diffusion process is controlled by a minimum time $T_\mathrm{min} = 1 \times 10^{-10}$, a maximum time $T_\mathrm{max} = 1.0$, a generation schedule nonlinearity parameter $\rho=2$ and a Langevin ratio $\eta = 0.5$. 
For Two Moons, OUP, Turin and Gaussian Linear, we set the number of diffusion steps as $N = 25$ and the number of Langevin steps as $N_L = 8$.
For BCI, we set $N = 250$ and $N_L = 0$.
In our experiments, we found that PriorGuide works better on BCI with a relatively large number of diffusion steps.
The test-time compute across all simulators resides in the regime of $10^2$ NFEs (number of function evaluations, \ie calls to the score model).

As detailed in Section \ref{sec:methods} of the main text, we approximate the prior ratio $r(\vtheta)$ by fitting a Gaussian mixture model (cf.~\textsc{FitGMM} in \cref{alg:prior_guide}).
For Two Moons, OUP and Turin, we do not need to fit a GMM since the training prior $p_\mathrm{train}$ is a uniform distribution, thus $r(\vtheta) = q(\vtheta)$.
For this reason, we can directly use the target prior $q(\vtheta)$ as the GMM parameterization for $r(\vtheta)$.
For Gaussian Linear and BCI, which have a Gaussian training prior, we fit a GMM with $K=20$ components to the ratio $r(\vtheta)$ by optimizing the L2 loss with gradient descent.
For numerical convenience, we parameterize the GMM components with diagonal covariance matrices.\footnote{Note that PriorGuide also supports full covariance matrix representations for the GMM components of $r(\vtheta)$.}
The optimization is carried out in $1 \times 10^6$ steps with batch size 1,000 and learning rate $0.01$ using the Adam optimizer \citep{kingma15}. We verified both via the L2 loss and via visualizations that the GMM fits achieved a satisfactory ratio approximation.

\paragraph{Amortized conditioning engine}
In ACE, first, we convert the test-time prior $q(\vtheta)$ generated in \cref{app:test-time-prior-generations} into a binned histogram distribution using the same steps discussed in \cref{app:training_details}. Using the offline true data we have already sampled, we then condition on the full data ($x$) in the context set and predict the parameters $\theta$ independently to obtain the posterior. Using this posterior distribution, which is a Gaussian Mixture Model (GMM), we then sample $\theta$ according to the number of posterior samples specified earlier. Note that ACE only supports factorized priors, so the \emph{mixture} prior $q$ cannot be specified correctly within ACE (\ie, ACE will represent the prior as a product of marginals).

\subsubsection{Posterior predictive inference} 
In the data prediction task (posterior predictive), with equal probability we condition on the first 30\% or the last 30\% of the values in $\x$, and then predict the remaining portion of $\x$. Note that in this task, we do not condition on $\vtheta$. We draw 100 samples from the data predictions for evaluation. 

\paragraph{Simformer and PriorGuide}
We use the same hyperparameter settings and a similar procedure as in the posterior inference setup (\cref{app:posterior-inference-setup}). In this case, the diffusion inference model will condition on the selected $30\%$ values of $\x$ to generate the remaining $70\%$. The steps for performing posterior predictive inference with PriorGuide are outlined in \cref{alg:prior_guide_post_pred}.

\paragraph{Amortized conditioning engine}
For ACE, similar to the posterior inference setup, we first convert the prior into a binned histogram distribution. We then place the portion of $\x$ to be conditioned on into the context set. Data prediction is performed autoregressively by first randomly permuting the order of the target data points. This procedure is adapted from \cite{bruinsma2023autoregressive}, as ACE is a neural process-based method. Consistent with the findings in \cite{bruinsma2023autoregressive}, we observed that using random ordering for the autoregressive procedure yields the most robust performance.

\subsubsection{Pareto frontiers of test-time compute assignments}
PriorGuide supports improving the sampling quality by adding Langevin dynamic steps to the diffusion process, at the cost of additional test-time compute.
With the same testing setup described above (\cref{app:posterior-inference-setup}), we applied PriorGuide to posterior inference in the OUP and Turin simulators to investigate the influence of additional Langevin dynamic steps.
We measure the test-time compute (including diffusion steps $N$ and Langevin steps $N_L$) in the number of function evaluations (NFEs), \ie calls to the score model, recalling the basic relation $\text{NFE} = N \times (N_L + 1)$.
We tested the posterior inference performance in the regime up to $10^3$ NFEs and presented Pareto-efficient assignments of test-time compute.

\subsection{Statistical methodology}  \label{app:statistical_methods}

\subsubsection{Evaluation metrics} \label{app:metrics}

\textbf{Root Mean Squared Error} (RMSE) measures the average magnitude of the errors between the predicted posterior samples and the ground-truth. A lower RMSE indicates a more concentrated prediction around the true parameters or data points. The RMSE is defined as
\begin{equation}
\mathrm{RMSE} = 
    \sqrt{
      \frac{1}{L\,N_{\mathrm{post}}}
      \sum_{l=1}^{L}
      \sum_{j=1}^{N_{\mathrm{post}}}
        \bigl(y_{l} - \widehat{y}_{l,j}\bigr)^2
    } \, ,
\end{equation}

where $L$ is the feature dimension, $N_\mathrm{post}$ is the number of posterior samples, $y_{i,l}$ is the ground-truth and $\widehat{y}_{l,j}$ is the prediction for the $l$-th feature (data point or parameter) and $j$-th posterior sample.
In posterior inference experiments, we compute the RMSE between $N_\mathrm{post} = 1,000$ samples and the ground-truth parameters.
Conversely, in data prediction experiments, we compute the RMSE between $N_\mathrm{post}=100$ predicted data points and the ground-truth.

\textbf{Classifier Two-Sample Test} (C2ST;~\citep{lopez-paz2017revisiting}) is a method to assess whether two sets of samples originate from the same distribution. In our context, it is used to compare the estimated posterior samples against samples from a reference posterior distribution. In our experiments, a random forest classifier is trained to distinguish between samples from the two distributions. An accuracy close to 0.5 suggests that the two distributions are indistinguishable.

\textbf{Mean Marginal Total Variation Distance} (MMTV) quantifies the dissimilarity between two multivariate distributions by considering their marginal distributions, defined as
\begin{equation}
    \mathrm{MMTV}(p,q)
      = \sum_{d=1}^{D}
        \int_{-\infty}^{\infty}
          \frac{\bigl|p_{d}^{\mathrm{M}}(x_{d}) - q_{d}^{\mathrm{M}}(x_{d})\bigr|}{2\,D}
        \,\mathrm{d}x_{d},
\end{equation}

where $p_{d}^{\mathrm{M}}$ and $q_{d}^{\mathrm{M}}$ denote the marginal densities of $p$ and $q$ along the $d$-th dimension. An MMTV metric of $0.2$ indicates that, on average across dimensions, the posterior marginals have an $80\%$ overlap, often indicated as a desirable threshold for an approximate posterior~\citep{acerbi2018variational}.

\textbf{Maximum Mean Discrepancy} (MMD) measures the distance between the mean embeddings of the distributions in a reproducing kernel Hilbert space. For our evaluations, we employ the MMD with with an exponentiated quadratic
kernel with a lengthscale of 1.

\subsubsection{Significance testing} \label{app:significance_testing}

To assess the statistical significance of the differences in performance between methods, we employ the Wilcoxon signed-rank test \citep{wilcoxon1945individual}.  In our experimental comparisons, we use this test to determine if the observed differences in metric scores between PriorGuide and baseline methods are statistically significant across multiple benchmark problems or datasets. The test evaluates the null hypothesis that the median difference between paired observations is zero, providing a p-value to indicate the significance of any observed deviation from this null hypothesis. We consider a result to be significantly different if the p-value is below 0.05.

\section{Additional experimental results}
\label{app:additional_results}

We present here additional experimental results. First, in \cref{app:illustration_two_moons}, we provide a visual example of test-time adaptation to supplement the discussion in \cref{sec:2dexamples}. In \cref{app:sec:seq_inf}, we demonstrate PriorGuide's applicability to structurally complex priors through a sequential Bayesian inference experiment on Two Moons. In \cref{app:additional_baselines}, we introduce additional baselines, including rejection sampling, sampling-importance-resampling and neural likelihood estimation (NLE; \citealp{papamakarios2019sequential}) with MCMC, to provide a broader comparison. We also present a study on the model's sensitivity to the distance between training and test-time priors (\cref{app:sensitity_prior_distance}) and conduct an ablation study on the GMM prior ratio approximation (\cref{app:ablation_gmm_components}), which is further validated by a scalability analysis of the ratio fit in high-dimensional space (\cref{app:ratio_fit}). Furthermore, we include example posterior visualizations for the BCI model (\cref{app:bci_viz}) and an analysis of training and test wall-clock time costs for Simformer and PriorGuide (\cref{app:inference_time}).

\subsection{Illustration of test-time prior adaptation on Two Moons}
\label{app:illustration_two_moons}
This section provides the figure referenced in \cref{sec:2dexamples}, illustrating the core capability of PriorGuide. We use the Two Moons model to provide an intuitive visualization of test-time prior adaptation. The base diffusion model is trained under a uniform prior $\pt(\vtheta)$ over $[-1,1]^2$. At test time, we introduce a more localized Gaussian prior and employ PriorGuide to guide the diffusion sampling towards the correct posterior distribution (\cref{fig:exp1}).

\begin{figure}[ht]
  \centering
  \includegraphics[width=0.9\textwidth]{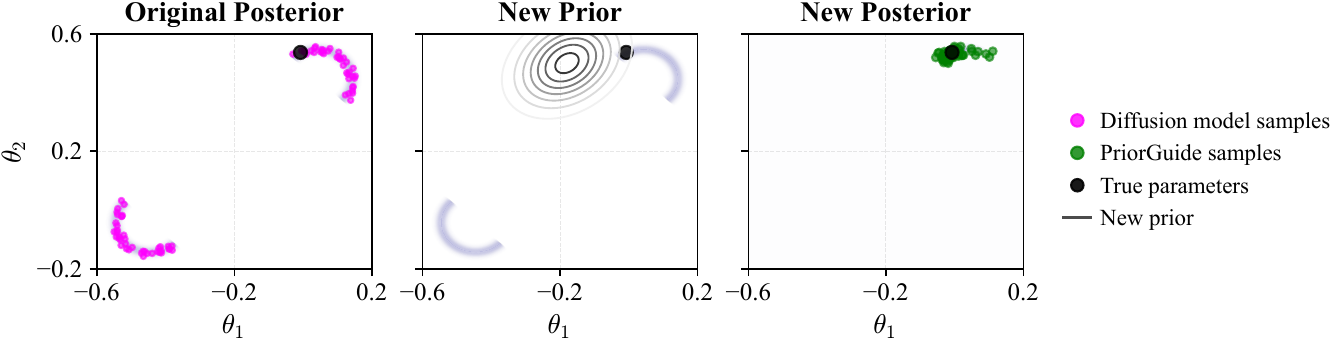}
  \caption{\emph{Left:} Original Two Moons posterior and posterior samples from the base diffusion model trained on a uniform prior. \emph{Middle:} New prior information about the parameters becomes available. \emph{Right:} PriorGuide steers the diffusion process to match the Bayesian posterior under the new prior.}
  \label{fig:exp1}
\end{figure}

\subsection{Sequential Bayesian updating with independent observations}
\label{app:sec:seq_inf}

To futher demonstrate PriorGuide’s applicability to complicated and non-factorized priors, we examine a sequential Bayesian inference scenario. We utilize the Two Moons model under a uniform prior $\pt(\vtheta)$ over $[-1,1]^2$, and consider a two-step Bayesian update scenario:
\begin{enumerate}
    \item We simulate two observations $\vx_1, \vx_2$ given a parameter vector $\vtheta^\star$.
    \item Initial update: We calculate the intermediate posterior $p(\vtheta \mid \vx_1)$ given the first observation $\vx_1$. In the Two Moons problem, this posterior is by construction a complex bimodal distribution (see~\cref{fig:exp1} and~\cref{fig:rebuttal1}). This becomes our ground-truth \emph{target prior} $q(\vtheta)$. 
    \begin{itemize}
    \item For PriorGuide, we perform prior ratio fitting by fitting a Gaussian mixture to this intermediate posterior (the ratio is $\propto q(\vtheta)$ due to the uniform prior).
    \end{itemize}
    \item Sequential update: We employ the inference method of choice (PriorGuide or another baseline) to sample from the final posterior $p(\vtheta \mid \vx_1, \vx_2)$, using $q(\vtheta)$ as the prior and conditioning on the second observation $\vx_2$. Notably, the fact that $q(\vtheta) = p(\vtheta | \vx_1)$ is itself a posterior does not require any special treatment: ``yesterday's posterior is today's prior''.
\end{enumerate}

Using this setup, we evaluate PriorGuide and baselines across 10 true parameter points randomly sampled from $p_\text{train}$, simulating corresponding observation pairs for each. To establish ground truths for both the intermediate posteriors (and target priors) $p(\vtheta \mid \vx_1)$ as well as the joint posteriors $p(\vtheta \mid \vx_1, \vx_2)$, we numerically calculate these posteriors on a dense discrete grid ($1000 \times 1000$ points) over the domain $[-1,1]^2$.

On posterior quality metrics (C2ST, MMTV), PriorGuide strongly outperforms the baselines (\cref{tab:app:seq_inf_two_moons}).
Here ACE underperforms, likely due to the fact that this complex prior differs significantly from the distribution of priors---the \emph{meta-prior}---it has been pretrained on (see~\cref{app:training_details}), showing the importance of test-time prior adaptation.
No difference is observed across methods for RMSE, but we recall that RMSE measures the error between the posterior and ground-truth parameter $\boldsymbol{\theta}^\star$. RMSE is not meaningful when the posterior is multimodal like in this case: the error is dominated by the regions of the posterior in the wrong mode, even though it is correct Bayesian behavior for the posterior to be multimodal. An illustration of such bimodality and of the update procedure for PriorGuide is visualized in \cref{fig:rebuttal1} for one example pair $(\vx_1, \vx_2)$.
Overall, these results confirm PriorGuide's capability to effectively handle complex, data-derived priors.

\begin{table}[htbp]
\caption{Posterior inference ($p(\vtheta \mid \vx_1, \vx_2)$). Mean $\pm$ standard dev. over 5 independent training runs (10 randomly generated true parameters and observation sets). Significantly best results (Wilcoxon signed-rank test) in bold.}
\label{tab:app:seq_inf_two_moons}
\centering
\setlength{\tabcolsep}{10pt}
\scalebox{1.0}{\begin{tabular}{llccc}
\toprule
\textbf{Dataset} & \textbf{Method} & \textbf{RMSE} & \textbf{C2ST} & \textbf{MMTV} \\
\midrule
\multirow{3}{*}{Two Moons}
  & Simformer         & $0.63 \pm 0.34$ & $0.89 \pm 0.02$ & $0.66 \pm 0.07$ \\
  & ACE               & $0.63 \pm 0.37$ & $0.94 \pm 0.03$ & $0.71 \pm 0.07$ \\
  & PriorGuide        & $0.63 \pm 0.34$ & $\mathbf{0.60}\pm 0.07$ & $\textbf{0.22} \pm 0.13$ \\
\bottomrule
\end{tabular}}
\end{table}

\begin{figure}[ht]
  \centering
  \includegraphics[width=0.9\textwidth]{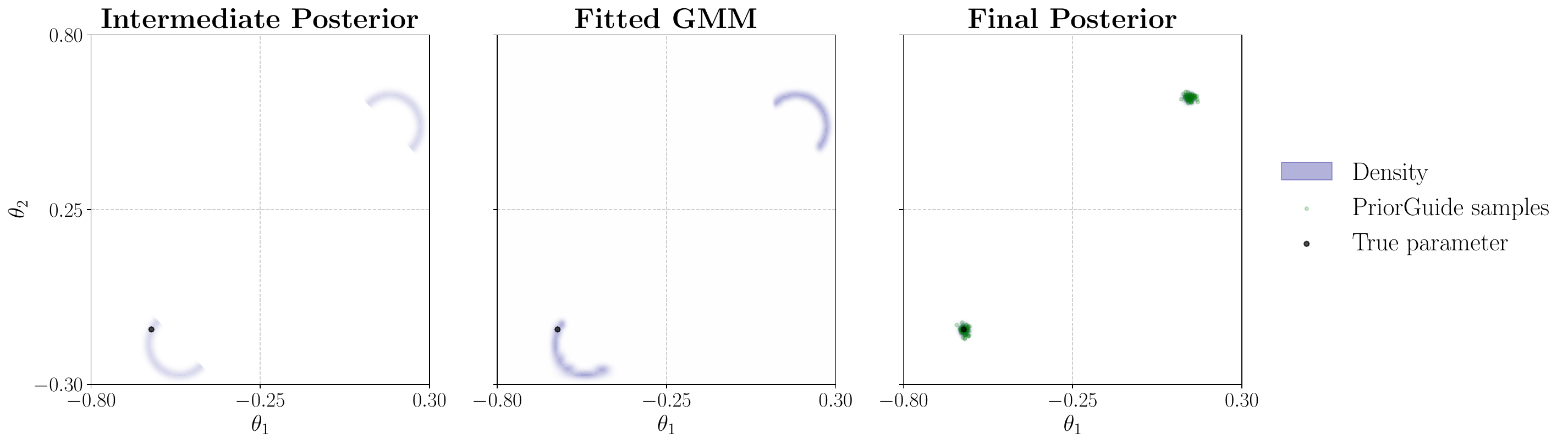}
\caption{\emph{Left:} The complex, bimodal posterior $p(\vtheta \mid x_1)$ derived from the first observation $\vx_1$. \emph{Middle:} This posterior is fitted with a GMM to serve as the new target prior $q(\vtheta)$ for the next step. \emph{Right:} PriorGuide samples from the final posterior $p(\vtheta|\vx_1, \vx_2)$ by guiding the diffusion model with the new prior $q(\vtheta)$ and conditioning on the second observation $\vx_2$. PriorGuide's samples match well the ground-truth posterior density.}
  \label{fig:rebuttal1}
\end{figure}

\subsection{Additional baselines} 
\label{app:additional_baselines}

\subsubsection{Rejection Sampling, and Sampling Important Resampling}
\label{app:baseline_rs_sir}

In addition to the main results, we compare PriorGuide against two straightforward sampling-based baselines: Sampling-Importance-Resampling (SIR) and Rejection Sampling (RS). Both methods are widely used and theoretically consistent. However, their efficiency deteriorates rapidly when the proposal distribution---in our case, the posterior from the base model under the training prior---is a poor match for the target posterior under the new test-time prior, particularly under high-dimensional or complex posterior geometries. Our experiments quantify this inefficiency and provide a direct performance comparison under matched computational budgets, highlighting the practical advantages of PriorGuide.

\paragraph{Setup.} For SIR, we drew 20,000 samples from the base diffusion model, applied importance weights, and evaluated performance alongside two diagnostics: effective sample size (ESS) and the $\hat{k}$ instability measure~\citep{vehtari2024pareto}, where $\hat{k}>0.7$ indicates instability. For RS, we continued sampling until either 1,000 accepted samples (matching PriorGuide) were obtained or a two-hour time limit on an Nvidia A100 GPU was reached. All three methods (PriorGuide, RS, and SIR) rely on samples from the same Simformer model, with PriorGuide requiring additional per-sample computation. This allows us to align costs directly. As shown in \cref{tab:time}, generating a single PriorGuide sample is roughly as expensive as producing 20 base Simformer samples. For our comparison we use 1,000 PriorGuide samples, 20,000 SIR samples, and cap RS at two hours to equalize compute budgets. This cap is generous as RS acceptance rates are very low, while PriorGuide sampling itself takes about $\sim$150 seconds (see \cref{tab:time}). Results are summarized below.

\begin{table}[htbp]
\caption{Posterior inference ($\vtheta$) results on Turin with RS and SIR, under \textit{strong} (S) and \textit{mixture} (M) test-time priors.}
\label{tab:app:rs_sir_turin}
\centering
\setlength{\tabcolsep}{10pt}
\scalebox{0.7}{\begin{tabular}{llcccccc}
\toprule
\textbf{Dataset} & \textbf{Method} & \textbf{RMSE} & \textbf{C2ST} & \textbf{MMTV} & \textbf{Acc Rate (\%)} & \textbf{ESS (\%)} & $\hat{k}$ \\
\midrule
\multirow{4}{*}{Turin (S)}
  & Simformer         & $0.23\pm0.04$ & $0.95\pm0.02$ & $0.55\pm0.07$ & \textemdash & \textemdash & \textemdash \\
  & ACE               & $0.18\pm0.02$ & $0.92\pm0.03$ & $0.47\pm0.08$ & \textemdash & \textemdash & \textemdash \\
  & PriorGuide        & $0.06\pm0.01$ & $0.55\pm0.03$ & $0.08\pm0.02$ & \textemdash & \textemdash & \textemdash \\
  & Diffusion + SIR   & $0.06\pm0.01$ & $0.81\pm0.07$ & $0.18\pm0.24$ &  & $0.02\pm0.01$ & $0.35\pm0.42$\\
  & Diffusion + RS    & $0.06\pm0.01$ & $0.56\pm0.04$ & $0.08\pm0.02$ & $0.63\pm0.33$& \textemdash & \textemdash \\
\midrule
\multirow{4}{*}{Turin (M)}
  & Simformer         & $0.24\pm0.04$ & $0.95\pm0.02$ & $0.52\pm0.07$ & \textemdash & \textemdash & \textemdash \\
  & ACE               & $0.20\pm0.03$ & $0.91\pm0.03$ & $0.45\pm0.07$ & \textemdash & \textemdash & \textemdash \\
  & PriorGuide        & $0.14\pm0.05$ & $0.64\pm0.09$ & $0.21\pm0.12$ & \textemdash & \textemdash & \textemdash \\
  & Diffusion + SIR   & $0.09\pm0.03$ & $0.80\pm0.07$ & $0.19\pm0.22$ &  & $0.02\pm0.01$ & $0.39\pm0.38$ \\
  & Diffusion + RS    & $0.09\pm0.02$ & $0.55\pm0.03$ & $0.08\pm0.04$ & $0.49\pm0.35$& \textemdash & \textemdash \\
\bottomrule
\end{tabular}}
\end{table}

\begin{table}[htbp]
\caption{Posterior inference ($\vtheta$) results on BCI with RS and SIR, under \textit{strong} (S) and \textit{mixture} (M) test-time priors.}
\label{tab:app:rs_sir_bci}
\centering

\setlength{\tabcolsep}{10pt}
\scalebox{0.7}{\begin{tabular}{llcccccc}
\toprule
\textbf{Dataset} & \textbf{Method} & \textbf{RMSE} & \textbf{C2ST} & \textbf{MMTV} & \textbf{Acc Rate (\%)} & \textbf{ESS (\%)} & $\hat{k}$ \\
\midrule
\multirow{4}{*}{BCI (S)}
  & Simformer         & $0.89\pm0.09$ & $0.97\pm0.02$ & $0.66\pm0.09$ & \textemdash & \textemdash & \textemdash \\
  & ACE               & $0.34\pm0.15$ & $0.91\pm0.05$ & $0.40\pm0.17$ & \textemdash & \textemdash & \textemdash \\
  & PriorGuide        & $0.24\pm0.04$ & $0.81\pm0.11$ & $0.33\pm0.19$ &  \textemdash & \textemdash & \textemdash \\
  & Diffusion + SIR   & $0.22\pm0.03$ & $0.95\pm0.04$ & $0.63\pm0.32$ &  & $0.01\pm0.01$ & $1.28\pm0.98$\\
  & Diffusion + RS    & \textit{Fail} & \textit{Fail} & \textit{Fail} & $\sim 0$    & \textemdash & \textemdash \\
\midrule
\multirow{4}{*}{BCI (M)}
  & Simformer         & $1.07\pm0.18$ & $0.99\pm0.01$ & $0.63\pm0.10$ &  \textemdash & \textemdash & \textemdash \\
  & ACE               & $1.00\pm0.31$ & $0.97\pm0.02$ & $0.58\pm0.10$ & \textemdash & \textemdash & \textemdash \\
  & PriorGuide        & $0.79\pm0.70$ & $0.78\pm0.11$ & $0.35\pm0.24$ & \textemdash & \textemdash & \textemdash \\
  & Diffusion + SIR   & $0.25\pm0.03$ & $0.95\pm0.04$ & $0.79\pm0.31$ &  & $0.00\pm0.00$ & $2.61\pm1.78$ \\
  & Diffusion + RS    & \textit{Fail} & \textit{Fail} & \textit{Fail} & $\sim 0$    & \textemdash & \textemdash \\
\bottomrule
\end{tabular}}
\end{table}

\begin{table}[htbp]
\caption{Posterior inference ($\vtheta$) results on Gaussian Linear 20D with RS and SIR, under \textit{strong} (S) and \textit{mixture} (M) test-time priors.}
\label{tab:app:rs_sir_gl20d}
\centering
\setlength{\tabcolsep}{10pt}
\scalebox{0.65}{\begin{tabular}{llcccccc}
\toprule
\textbf{Dataset} & \textbf{Method} & \textbf{RMSE} & \textbf{C2ST} & \textbf{MMTV} & \textbf{Acc Rate (\%)} & \textbf{ESS (\%)} & $\hat{k}$ \\
\midrule
\multirow{5}{*}{\shortstack{Gaussian\\Linear 20D (S)}}
  & Simformer         & $0.30\pm0.02$ & $1.00\pm0.00$ & $0.64\pm0.02$ &  \textemdash & \textemdash & \textemdash \\
  & ACE               &$0.10\pm0.01$ & $0.80\pm0.04$ & $0.15\pm0.02$ & \textemdash & \textemdash & \textemdash \\
  & PriorGuide        & $0.21\pm0.08$ & $0.58\pm0.03$ & $0.10\pm0.02$ & \textemdash & \textemdash & \textemdash \\
  & Diffusion + SIR   & $0.15\pm0.02$ & $\sim 1$ & $\sim 1$ &  & $0.00\pm0.00$ & $10.85\pm1.32$ \\
  & Diffusion + RS    & \textit{Fail} & \textit{Fail} & \textit{Fail} & $\sim 0$    & \textemdash & \textemdash \\
\midrule
\multirow{5}{*}{\shortstack{Gaussian\\Linear 20D (M)}}
  & Simformer         & $0.30\pm0.01$ & $1.00\pm0.00$ & $0.64\pm0.02$ & \textemdash & \textemdash & \textemdash \\
  & ACE               & $0.24\pm0.03$ & $1.00\pm0.00$ & $0.52\pm0.07$ & \textemdash & \textemdash & \textemdash \\
  & PriorGuide        & $0.24\pm0.08$ & $0.59\pm0.06$ & $0.14\pm0.09$ & \textemdash & \textemdash & \textemdash \\
  & Diffusion + SIR   & $0.15\pm0.01$ & $ \sim 1$ & $ \sim 1$ &  & $0.00\pm0.00$ & $10.61\pm1.12$\\
  & Diffusion + RS    & \textit{Fail} & \textit{Fail} & \textit{Fail} & $\sim 0$    & \textemdash & \textemdash \\
\bottomrule
\end{tabular}}
\end{table}

\paragraph{Results.} Under equal computational budgets, the results exhibit a dimensionality-driven behavior consistent with performance expectations of Monte Carlo methods in moderate dimensions~\citep{robert2004monte}. For the 4-dimensional Turin example (\cref{tab:app:rs_sir_turin}), the simple baselines remain viable, with performance roughly comparable to PriorGuide (sometimes slightly better, sometimes worse). RS achieves acceptable metrics but with highly variable acceptance rates, yielding only 98–126 accepted samples out of 20k. 

Conversely, in the 5-dimensional BCI problem (\cref{tab:app:rs_sir_bci}), these baselines catastrophically fail. Rejection sampling produces zero accepted samples on most runs, while SIR's $\hat{k} > 0.7$ indicates statistical instability that makes its estimates unreliable. This transition due to dimensionality (4D vs. 5D) and the increased posterior complexity reveal a threshold where simple methods go from mildly unreliable to completely unusable, while PriorGuide maintains consistent performance across both problems.

Extending the analysis further to the 20D Gaussian Linear problem confirms this trend (\cref{tab:app:rs_sir_gl20d}). Despite being given a generous budget of 20k proposals (whereas each PriorGuide sample costs only a fraction of a Simformer sample), SIR and RS fail completely, while PriorGuide continues to deliver stable performance.

\paragraph{Larger computational budget.}
To examine behavior with substantially more compute, we ran rejection sampling until reaching 1,000 accepted samples or an 8-hour limit on a single Nvidia A100 GPU (up to $\sim$3.1M proposals) on Turin and BCI examples. While this setting is not typical for practitioners, it allows us to probe performance under very large budgets. In this regime, rejection sampling could sometimes match or slightly outperform PriorGuide. On Turin with a strong or mixture prior, both methods achieved similar accuracy (RMSE $\approx $0.06--0.09, C2ST $\approx$0.54--0.62, MMTV $\approx $0.07--0.19) with acceptance rates of 0.38--0.57\%. On BCI, rejection sampling occasionally improved RMSE and MMTV (e.g., 0.22 vs.\ 0.63 RMSE on BCI-M), but acceptance was extremely low (0.06--0.16\%), many runs hit the timeout, and some produced only a handful of accepted samples.  

These results illustrate a trade-off: simple methods can approach strong performance if one is willing to spend hours of compute, but they remain unreliable and inefficient. PriorGuide achieves comparable quality in minutes ($\sim$160 s for 1k samples), making it suitable for interactive use, repeated posterior evaluations, and sensitivity analysis, rather than only in very large-budget scenarios.

\subsubsection{Neural likelihood estimation with MCMC}
\label{app:baseline_nle_mcmc}

As an additional natural baseline, we evaluate a method based on neural likelihood estimation (NLE; \citealp{papamakarios2019sequential}) followed by MCMC sampling. 
This two-stage method is a representative of several likelihood-based approaches in SBI \citep{lueckmann2021benchmarking}, which trains a neural density estimator (e.g.~a normalizing flow) to form a surrogate of the likelihood function $p(\vx \mid \vtheta)$ \citep{papamakarios2019sequential}. 
Once the surrogate is trained, it is then used with a standard MCMC algorithm to draw samples from the posterior distribution $p(\vtheta \mid \vx_o)$ for a given observation $\vx_o$. 

While decoupling likelihood learning from posterior sampling is conceptually natural in SBI, we find this approach may lack robustness even in low-dimensional cases, particularly for problems with multi-modal posterior geometries.
For instance, when this approach is applied in the bimodal Two Moons simulator \cref{app:simulators}, the MCMC sampler guided by the learned likelihood often becomes trapped in a single posterior mode and fails to discover the other (\cref{fig:nle_mcmc_failure}). 
This failure is a known MCMC limitation that the sampler can easily get stuck in modes. 

\begin{figure}[ht]
\centering
\includegraphics[width=0.7\linewidth]{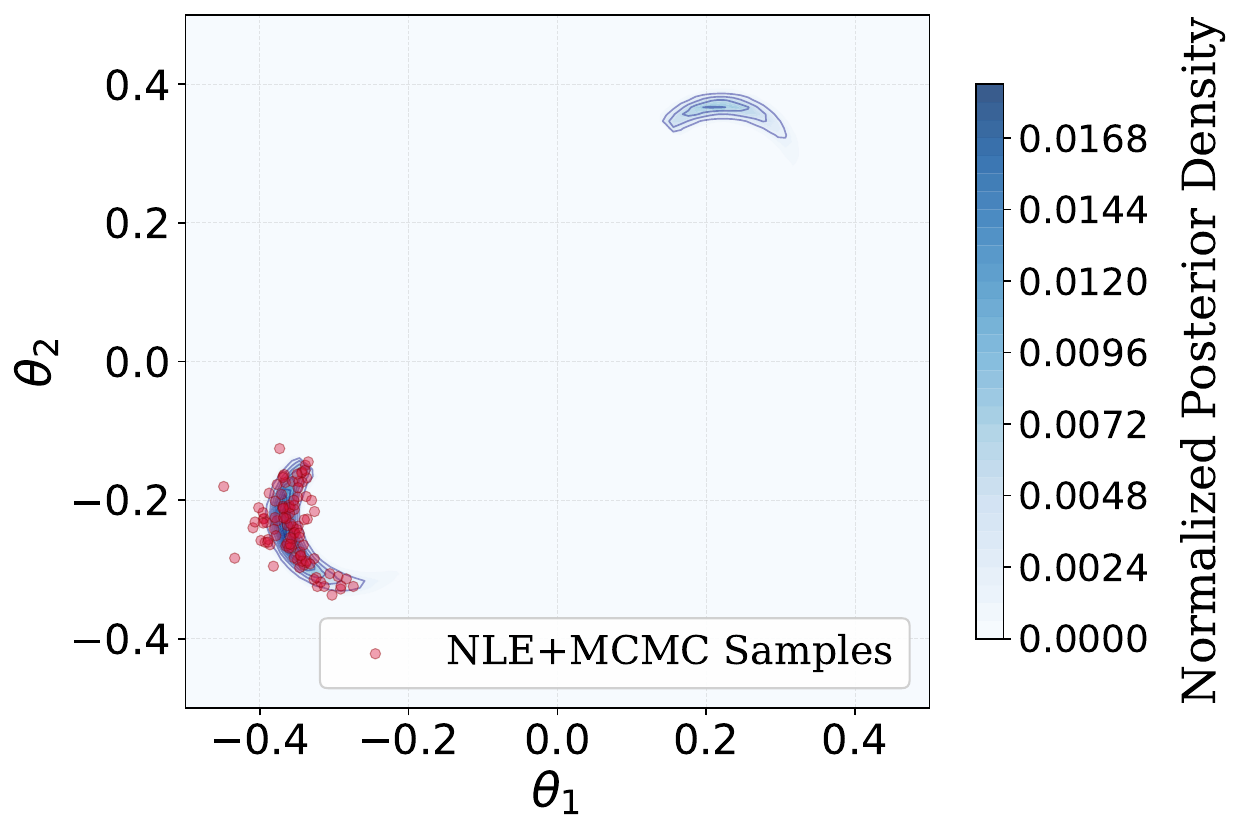}
\caption{
Failure case of the NLE with MCMC baseline on the Two Moons problem. The posterior is approximated by running an MCMC sampler on a learned neural likelihood.
}
\label{fig:nle_mcmc_failure}
\end{figure}

\subsection{Sensitivity to the distance between training and test-time priors}
\label{app:sensitity_prior_distance}

To investigate the sensitivity of PriorGuide to the distance between the training and test-time priors, we conduct an experiment using the Gaussian Linear 10D model. We use the same training prior as in \cref{sec:experiments}: a 10-dimensional Gaussian distribution $p_{\text{train}}(\vtheta)= \mathcal{N}(\mathbf{0}, 0.1 \cdot \mathbf{I}_{10})$. For the test-time priors, we use strongly informative standard deviations ($\sigma_i^\text{strong} = 0.2 \cdot s_i$) and systematically shift their means ($\vmu^q$) away from the center of the training prior ($\vmu^p$). We quantify this shift using several metrics, including the Mahalanobis distance ($d'$), the 2-Wasserstein distance, and our proposed out-of-distribution (OOD) diagnostic (\cref{app:prior_diagnostic}).

The results, presented in \cref{tab:sensitivity_prior_distance}, reveal how PriorGuide's performance responds to this increasing prior shift. The key take-home message is that PriorGuide is robust and degrades gracefully rather than failing catastrophically. For small to moderate shifts (e.g., $d' \leq 2$), performance remains nearly identical to the no-shift scenario, with all metrics (C2ST, MMTV, RMSE) showing negligible changes. As the prior is shifted into a substantially different region, we observe a smooth and predictable decline in accuracy. For instance, at an extreme shift where $d' = 4.4$, PriorGuide achieves an MMTV of $0.16$, indicating that the generated posterior marginals have an $84\%$ overlap with the ground truth on average across all dimensions, which corresponds to a reasonable posterior approximation \citep{acerbi2018variational}. This analysis confirms that PriorGuide provides provides meaningful approximations even when the test-time prior is far from what was seen during training.

\begin{table}[ht]
\centering
\caption{Sensitivity analysis of posterior inference to prior shift. The degree of shift is quantified using both $d'$ (Mahalanobis distance) and 2-Wasserstein distance. The results indicate that PriorGuide maintains competitive performance even under substantial prior shifts.} 
\label{tab:sensitivity_prior_distance}
\begin{tabular}{ccccccc}
\toprule
$\lVert \boldsymbol{\mu}^q - \boldsymbol{\mu}^p\rVert$ & $d'$ & 2-Wasserstein & OOD Frac. & C2ST & MMTV & RMSE \\
\midrule
0.00 & 0.00 & 0.84 & 0.00 & $0.52 \pm 0.05$ & $0.05 \pm 0.01$ & $0.21 \pm 0.10$ \\
0.32 & 0.44 & 0.89 & 0.00 & $0.53 \pm 0.06$ & $0.06 \pm 0.02$ & $0.20 \pm 0.08$ \\
0.63 & 0.88 & 1.05 & 0.00 & $0.54 \pm 0.07$ & $0.06 \pm 0.02$ & $0.23 \pm 0.09$ \\
0.95 & 1.32 & 1.27 & 0.00 & $0.56 \pm 0.08$ & $0.07 \pm 0.03$ & $0.32 \pm 0.13$ \\
1.26 & 1.76 & 1.52 & 0.00 & $0.59 \pm 0.09$ & $0.08 \pm 0.04$ & $0.41 \pm 0.15$ \\
1.58 & 2.20 & 1.79 & 0.02 & $0.61 \pm 0.09$ & $0.10 \pm 0.04$ & $0.58 \pm 0.14$ \\
1.90 & 2.64 & 2.07 & 1.00 & $0.62 \pm 0.09$ & $0.10 \pm 0.04$ & $0.73 \pm 0.16$ \\
2.21 & 3.08 & 2.37 & 1.00 & $0.64 \pm 0.08$ & $0.11 \pm 0.04$ & $0.84 \pm 0.18$ \\
2.53 & 3.52 & 2.66 & 1.00 & $0.65 \pm 0.08$ & $0.12 \pm 0.05$ & $1.04 \pm 0.19$ \\
2.85 & 3.96 & 2.97 & 1.00 & $0.67 \pm 0.08$ & $0.14 \pm 0.05$ & $1.25 \pm 0.19$ \\
3.16 & 4.40 & 3.27 & 1.00 & $0.70 \pm 0.08$ & $0.16 \pm 0.06$ & $1.42 \pm 0.24$ \\
\bottomrule
\end{tabular}
\end{table}

\subsection{Impact of GMM component count on prior ratio approximation}
\label{app:ablation_gmm_components}
We conduct an additional ablation study on the influence of the number of GMM components used to approximate the prior ratio. To isolate this effect, this experiment uses a controlled setup that differs from the main experiments: we evaluate performance using one target prior and one trained diffusion model, with metrics averaged across ten observations generated from that prior. As shown in \cref{tab:ablation_gmm_components}, sampling speed remains efficient even with very large GMMs (e.g. 200 components), indicating that it is not a bottleneck at this stage. In contrast, the performance of PriorGuide is sensitive to the quality of the prior ratio approximation when using very few components. We therefore recommend that users fit a flexible and highly expressive GMM tailored to their specific application to ensure accurate and stable guidance. In practice, we see that ~20 components are enough in our experiments.

\subsection{Scalability and robustness of the prior ratio approximation} \label{app:ratio_fit}
To assess the scalability of our GMM fitting procedure for the prior ratio $r(\vtheta)$, we evaluate the quality of the approximation in a 10-dimensional setting. We use the Gaussian Linear 10$D$ setup and set $q(\vtheta)$ as a Gaussian mixture.

We execute the fitting procedure for 100 randomly sampled mixture target priors (using the procedures detailed in \cref{app:test-time-prior-generations}). For each fit, we estimate the weighted $L_2$ distance in log space, that is the $L_2$ distance between the GMM approximation log predictive density (including the fitted log normalization constant $\ln Z$) and the true analytic log ratio using 10,000 samples from the fitted GMM. The method achieves a median $L_2$ error of 0.023 (2.5th-25th-75th-97.5th percentiles: 0.002-0.016-0.032-0.186), confirming that standard gradient-based optimization consistently converges to an accurate approximation of the ratio function in this 10$D$ scenario.\footnote{Note that 0.023 error in log space corresponds to $\sim 2\%$ error on the ratio, since $e^x \approx 1+x$ for small $x$.} In practice, to address the small percentage of fits that converge to suboptimal solutions, we recommend running multiple optimizations with different initializations---which is the standard approach when solving a global nonconvex optimization problem.

To give an example of such fits, \cref{fig:rebuttal2} displays a fitted GMM ratio against the true analytic ratio in log-space. We evaluate both the true analytic ratio and the fitted GMM passing through a fixed parameter vector $\theta^*$, illustrating the function's behavior for each pair of dimensions while holding the remaining dimensions constant. The figure demonstrates that the fitted GMM (orange lines) faithfully recovers the geometry of the true ratio (dashed black lines) across these conditional slices.

\begin{figure}[ht]
    \centering
    \includegraphics[width=\linewidth]{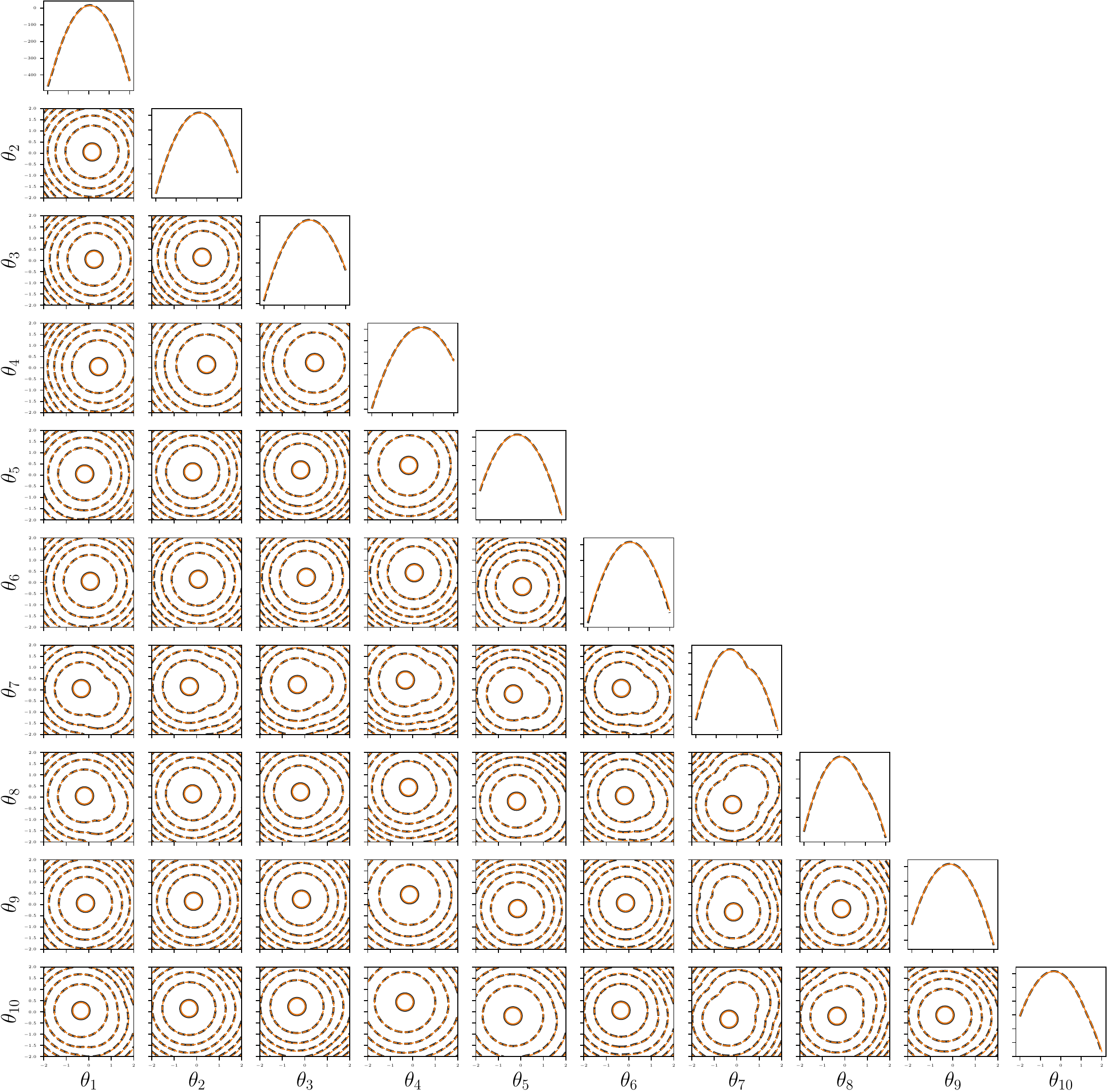}
    \caption{Example visualization of the prior ratio approximation in 10D (Gaussian Linear problem). Conditional 2D cross-sections of the true prior ratio $r(\vtheta)$ (dashed \textcolor{black}{\textbf{black}} lines) versus the fitted GMM approximation (\textcolor{orange}{\textbf{orange}} lines). Note that the target ratio is not Gaussian, as visible in some panels. The GMM approximation achieves a near-perfect fit.}
\label{fig:rebuttal2}
\end{figure}

\begin{table}[ht]
\centering
\caption{Ablation study illustrating the influence of the number of GMM components used to approximate the prior ratio.}
\label{tab:ablation_gmm_components}
\begin{tabular}{cccccc}
\toprule
Dataset & \makecell{\#GMM \\ components} & \makecell{Time cost (s) \\ per sample} & C2ST & MMTV & RMSE \\
\midrule
\multirow{3}{*}{\shortstack{Gaussian\\Linear 10D}} & 2 & $0.02 \pm 0.00$ & $0.97 \pm 0.01$ & $0.31 \pm 0.05$ & $0.20 \pm 0.03$ \\
 & 20 & $0.02 \pm 0.00$ & $0.55 \pm 0.04$ & $0.09 \pm 0.04$ & $0.23 \pm 0.07$ \\
 & 200 & $0.02 \pm 0.00$ & $0.56 \pm 0.05$ & $0.11 \pm 0.04$ & $0.28 \pm 0.11$ \\
\midrule
\multirow{3}{*}{BCI} & 2 & $0.18 \pm 0.01$ & $1.00 \pm 0.00$ & $0.72 \pm 0.04$ & $1.71 \pm 0.21$ \\
 & 20 & $0.18 \pm 0.01$ & $0.88 \pm 0.10$ & $0.57 \pm 0.24$ & $1.36 \pm 0.66$ \\
 & 200 & $0.18 \pm 0.01$ & $0.88 \pm 0.10$ & $0.57 \pm 0.24$ & $1.36 \pm 0.65$ \\
\bottomrule
\end{tabular}
\end{table}

\subsection{Bayesian causal inference (BCI) model posterior samples visualization}
\label{app:bci_viz}

In \cref{fig:bav-posterior-corner}, we visualize the ground-truth posterior distribution alongside samples generated by Simformer (no prior adaptation) and with PriorGuide applied to the same base Simformer model. We use the \texttt{corner} package~\citep{corner}, which displays pairwise joint distributions and marginal histograms of the parameters.

\definecolor{darkgreen}{rgb}{0.0, 0.5, 0.0}
\begin{figure}[ht]
    \centering
  \begin{subfigure}[b]{0.49\linewidth}
    \includegraphics[width=\linewidth]{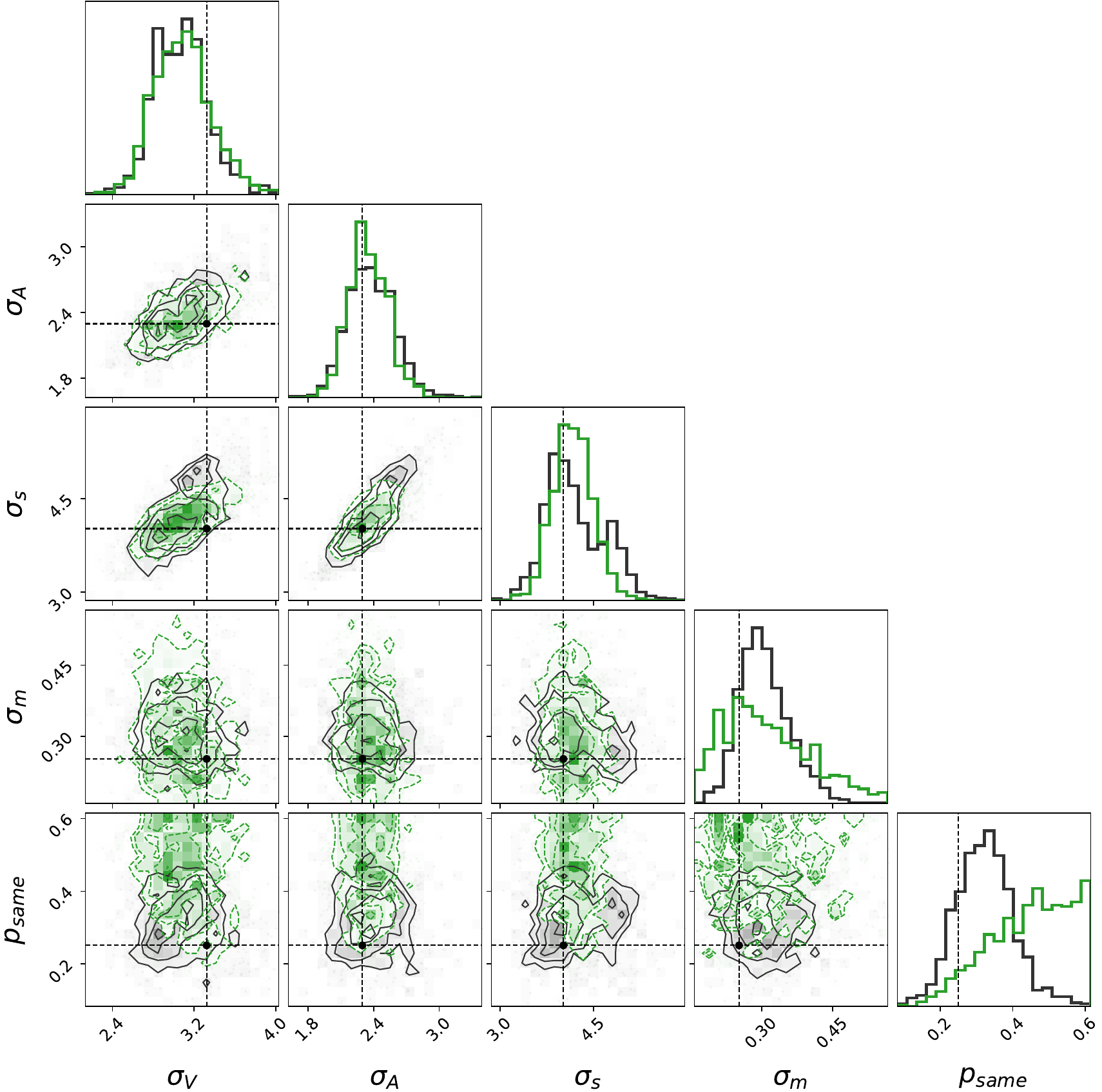}
    \caption{} %
    \label{fig:bav_corner_priorguide}
  \end{subfigure}
  \hfill
  \begin{subfigure}[b]{0.49\linewidth}
    \includegraphics[width=\linewidth]{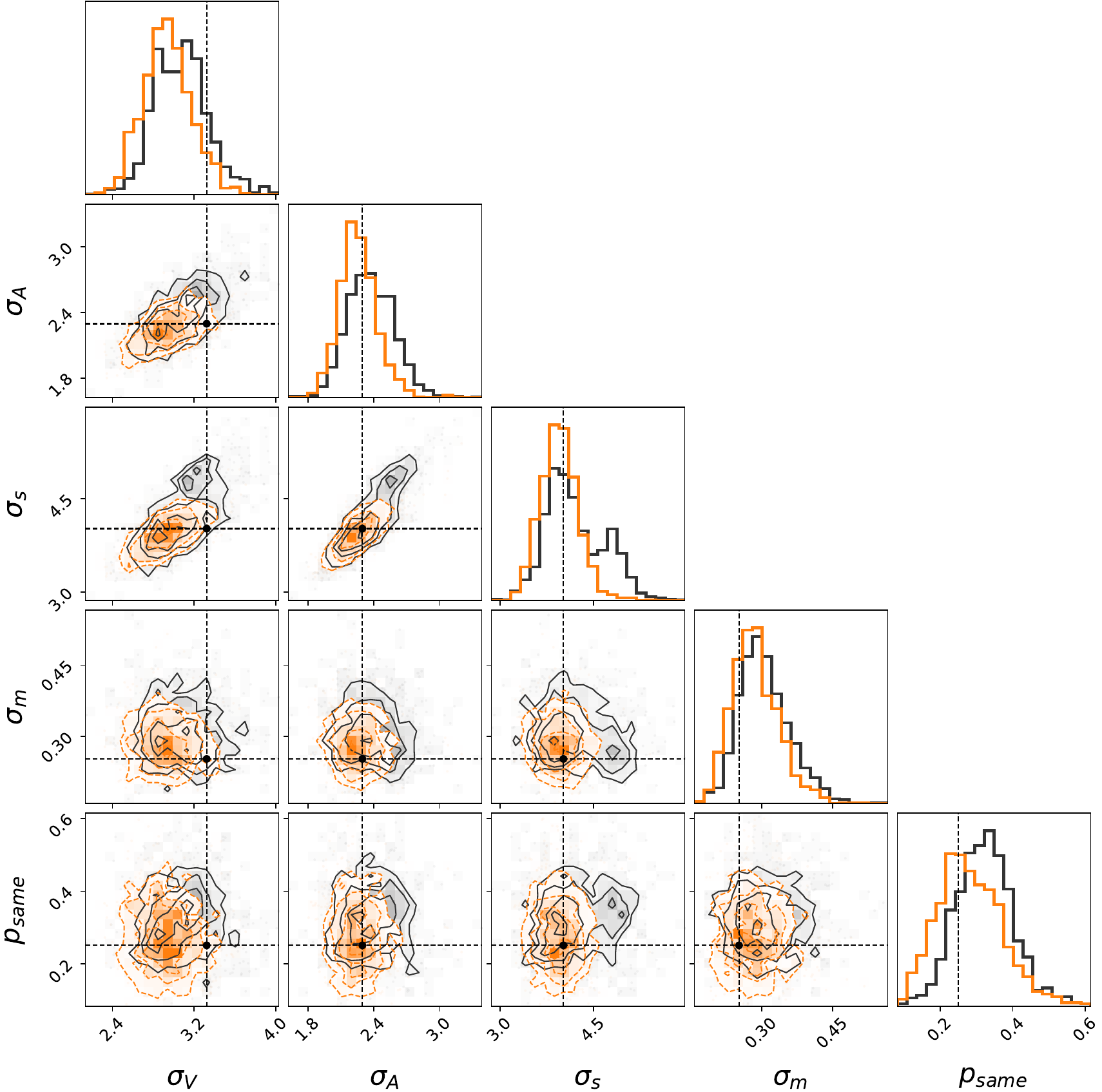}
    \caption{} %
    \label{fig:bav_corner_simformer}
  \end{subfigure}
\caption{Posterior samples for the Bayesian causal inference (BCI) model, for a randomly chosen dataset. Ground-truth samples are shown in \textbf{black}, and the dashed black lines and black dot indicate the true parameter values. \textbf{(a)} Original Simformer samples (\textcolor{darkgreen}{\textbf{green}}) without prior adaptation, and \textbf{(b)} PriorGuide samples (\textcolor{orange}{\textbf{orange}}) under a \emph{mild} prior. Incorporating prior information with PriorGuide steers the inference toward more plausible parameter regions, especially for parameters such as $\sigma_m$ and $p_\text{same}$ which are less constrained by the data alone, resulting in posterior samples that more closely match the ground-truth posterior under the same prior, as shown in (b).
}
\label{fig:bav-posterior-corner}
\end{figure}

\subsection{Training vs. test-time cost analysis} \label{app:inference_time}

We summarize in Table \ref{tab:time} the estimated wall‐clock time required by PriorGuide to produce 1,000 posterior samples under each simulator, measured on a system equipped with an Nvidia Ampere A100 GPU.
For comparison, we also report the cost of retraining a diffusion inference model (Simformer, \citep{gloeckler2024all}) with 10,000 simulations, which includes both the neural network training and the simulator call required to create its training set. 
As Table \ref{tab:time} demonstrates, PriorGuide achieves more than a tenfold speedup across most simulators, with the greatest gains observed in scenarios (\eg,~Turin) where the simulator calls are most computationally expensive. %

\begin{table}[hbt]
  \caption{A comparison of time costs (seconds) between performing PriorGuide posterior inference and retraining Simformer, including neural network training and generating new training data. }
  \label{tab:time}
  \centering
  \begin{tabular}{c c c c c c}
     \toprule
      \multirow{2}{*}{Simulator} & \multirow{2}{*}{PriorGuide testing} & \multicolumn{4}{c}{Simformer retraining and testing} \\
      \cmidrule(lr){3-6}
       &                            & Total & Simulation & Training & Sampling   \\
      \midrule
      Two Moons & 9.7 & 314.4 & 13.5 & 295.2 & 5.7 \\
      OUP & 25.1 & 305.7 & 9.2 & 289.4 & 7.1 \\
      Turin & 159.9 & 9346.6 & 8300.4 & 1037.4 & 8.8 \\
      Gaussian Linear & 16.8 & 266.1 & 3.8 & 255.2 & 7.1 \\
      BCI & 158.8 & 1063.1 & 49.1 & 1005.7 & 8.3 \\
      \bottomrule
  \end{tabular}%
\end{table}

\section{Computational resources and software}
\label{app:miscellanea}

\paragraph{Computational resources}
Most experiments presented in this work are performed on a cluster equipped with AMD MI250X GPUs, while some additional experiments in the appendix are performed on a cluster equipped with Nvidia Ampere A100 GPUs.
The total computational resources consumed for this research, including all development stages and experimental runs, are estimated to be approximately 30,000 GPU hours.

\paragraph{Software}
The core code base is built upon the \texttt{Simformer} repository (\url{https://github.com/mackelab/simformer}, License: MIT), using \texttt{JAX} (\url{https://docs.jax.dev/en/latest/}, License: Apache-2.0) and \texttt{PyTorch} (\url{https://pytorch.org/}, License: modified BSD license). 
The implementations of the Two Moons and Gaussian Linear simulators are based on \texttt{sbibm} (\url{https://github.com/sbi-benchmark/sbibm}, License: MIT).
For the ground-truth posterior generation, we utilize \texttt{sbi} (\url{https://sbi-dev.github.io/sbi/latest/}, License: Apache-2.0) and \texttt{PyVBMC} (\url{https://acerbilab.github.io/pyvbmc/}, License: BSD 3-Clause).
Our implementation of the ACE baseline uses the repository provided by \citep{chang2025amortized} (\url{https://github.com/acerbilab/amortized-conditioning-engine}, License: Apache-2.0).

\section{Use of Large Language Models}
We acknowledge the use of Large Language Models (LLMs) to support various stages of the research. In the initial phase, LLMs were utilized for inspirations, helping to explore methodological approaches and find existing works. In the development phase, LLMs served as programming assistants to aid tasks such as implementing algorithms and debugging. LLMs also supported the writing process, providing assistance with polishing the manuscript for clarity, conciseness, and grammatical correctness.

%% file: references.bib
@article{vehtari2024pareto,
  title={Pareto smoothed importance sampling},
  author={Vehtari, Aki and Simpson, Daniel and Gelman, Andrew and Yao, Yuling and Gabry, Jonah},
  journal={Journal of Machine Learning Research},
  volume={25},
  number={72},
  pages={1--58},
  year={2024}
}

@inproceedings{baoanalytic,
  title={Analytic-DPM: an Analytic Estimate of the Optimal Reverse Variance in Diffusion Probabilistic Models},
  author={Bao, Fan and Li, Chongxuan and Zhu, Jun and Zhang, Bo},
  booktitle={International Conference on Learning Representations}
}

@inproceedings{manorposterior,
  title={On the Posterior Distribution in Denoising: Application to Uncertainty Quantification},
  author={Manor, Hila and Michaeli, Tomer},
  booktitle={The Twelfth International Conference on Learning Representations}
}

@article{rozet2024learning,
  title={Learning diffusion priors from observations by expectation maximization},
  author={Rozet, Fran{\c{c}}ois and Andry, G{\'e}r{\^o}me and Lanusse, Fran{\c{c}}ois and Louppe, Gilles},
  journal={Advances in Neural Information Processing Systems},
  volume={37},
  pages={87647--87682},
  year={2024}
}

@article{song2019generative,
  title={Generative modeling by estimating gradients of the data distribution},
  author={Song, Yang and Ermon, Stefano},
  journal={Advances in neural information processing systems},
  volume={32},
  year={2019}
}

@article{linhart2024diffusion,
  title={Diffusion posterior sampling for simulation-based inference in tall data settings},
  author={Linhart, Julia and Cardoso, Gabriel Victorino and Gramfort, Alexandre and Corff, Sylvain Le and Rodrigues, Pedro LC},
  journal={arXiv preprint arXiv:2404.07593},
  year={2024}
}

@inproceedings{geffner2023compositional,
  title={Compositional score modeling for simulation-based inference},
  author={Geffner, Tomas and Papamakarios, George and Mnih, Andriy},
  booktitle={International Conference on Machine Learning},
  pages={11098--11116},
  year={2023},
  organization={PMLR}
}

@inproceedings{du2023reduce,
  title={Reduce, reuse, recycle: Compositional generation with energy-based diffusion models and mcmc},
  author={Du, Yilun and Durkan, Conor and Strudel, Robin and Tenenbaum, Joshua B and Dieleman, Sander and Fergus, Rob and Sohl-Dickstein, Jascha and Doucet, Arnaud and Grathwohl, Will Sussman},
  booktitle={International conference on machine learning},
  pages={8489--8510},
  year={2023},
  organization={PMLR}
}

@inproceedings{skretafeynman,
  title={Feynman-Kac Correctors in Diffusion: Annealing, Guidance, and Product of Experts},
  author={Skreta, Marta and Akhound-Sadegh, Tara and Ohanesian, Viktor and Bondesan, Roberto and Aspuru-Guzik, Alan and Doucet, Arnaud and Brekelmans, Rob and Tong, Alexander and Neklyudov, Kirill},
  booktitle={Forty-second International Conference on Machine Learning}
}

@article{lee2025debiasing,
  title={Debiasing guidance for discrete diffusion with sequential monte carlo},
  author={Lee, Cheuk Kit and Jeha, Paul and Frellsen, Jes and Lio, Pietro and Albergo, Michael Samuel and Vargas, Francisco},
  journal={arXiv preprint arXiv:2502.06079},
  year={2025}
}

@inproceedings{thorntoncomposition,
  title={Composition and Control with Distilled Energy Diffusion Models and Sequential Monte Carlo},
  author={Thornton, James and B{\'e}thune, Louis and ZHANG, Ruixiang and Bradley, Arwen and Nakkiran, Preetum and Zhai, Shuangfei},
  booktitle={The 28th International Conference on Artificial Intelligence and Statistics}
}

@inproceedings{cardosomonte,
  title={Monte Carlo guided Denoising Diffusion models for Bayesian linear inverse problems.},
  author={Cardoso, Gabriel and Le Corff, Sylvain and Moulines, Eric and others},
  booktitle={The Twelfth International Conference on Learning Representations}
}

@inproceedings{dou2024diffusion,
  title={Diffusion posterior sampling for linear inverse problem solving: A filtering perspective},
  author={Dou, Zehao and Song, Yang},
  booktitle={The Twelfth International Conference on Learning Representations},
  year={2024}
}

@article{wu2023practical,
  title={Practical and asymptotically exact conditional sampling in diffusion models},
  author={Wu, Luhuan and Trippe, Brian and Naesseth, Christian and Blei, David and Cunningham, John P},
  journal={Advances in Neural Information Processing Systems},
  volume={36},
  pages={31372--31403},
  year={2023}
}

@inproceedings{gloeckler2024all,
  title={All-in-one simulation-based inference},
  author={Gloeckler, Manuel and Deistler, Michael and Weilbach, Christian and Wood, Frank and Macke, Jakob H},
  booktitle={Proceedings of the International Conference on Machine Learning (ICML)},
  pages={15735--15766},
  year={2024}
}

@inproceedings{schmitt_consistency_2024,
    title = {Consistency Models for Scalable and Fast Simulation-Based Inference},
    booktitle = {Advances in Neural Information Processing Systems (NeurIPS)},
    author = {Schmitt, Marvin and Pratz, Valentin and Köthe, Ullrich and Bürkner, Paul-Christian and Radev, Stefan T.},
    year = {2024},
}

@article{boystweedie,
  title={Tweedie Moment Projected Diffusions for Inverse Problems},
  author={Boys, Benjamin and Girolami, Mark and Pidstrigach, Jakiw and Reich, Sebastian and Mosca, Alan and Akyildiz, Omer Deniz},
  journal={Transactions on Machine Learning Research}
}

@article{mittal2025amortized,
  title={Amortized In-Context {B}ayesian Posterior Estimation},
  author={Mittal, Sarthak and Bracher, Niels Leif and Lajoie, Guillaume and Jaini, Priyank and Brubaker, Marcus},
  journal={arXiv preprint arXiv:2502.06601},
  year={2025}
}

@inproceedings{greenberg2019automatic,
  title={Automatic posterior transformation for likelihood-free inference},
  author={Greenberg, David and Nonnenmacher, Marcel and Macke, Jakob},
  booktitle={Proceedings of the International Conference on Machine Learning (ICML)},
  pages={2404--2414},
  year={2019},
  organization={PMLR}
}

@inproceedings{snell2025scaling,
  title={Scaling {LLM} test-time compute optimally can be more effective than scaling model parameters},
  author={Snell, Charlie and Lee, Jaehoon and Xu, Kelvin and Kumar, Aviral},
  booktitle={International Conference on Learning Representations (ICLR)},
  year={2025}
}

@article{vetter2025effortless,
  title={Effortless, Simulation-Efficient {B}ayesian Inference using Tabular Foundation Models},
  author={Vetter, Julius and Gloeckler, Manuel and Gedon, Daniel and Macke, Jakob H},
  journal={arXiv preprint arXiv:2504.17660},
  year={2025}
}

@inproceedings{papamakarios2019sequential,
  title={Sequential neural likelihood: Fast likelihood-free inference with autoregressive flows},
  author={Papamakarios, George and Sterratt, David and Murray, Iain},
  booktitle={Proceedings of the 22nd International Conference on Artificial Intelligence and Statistics (AISTATS)},
  pages={837--848},
  year={2019},
  series = {Proceedings of Machine Learning Research},
  organization={PMLR}
}

@article{uhlenbeck1930theory,
  title={On the theory of the Brownian motion},
  author={Uhlenbeck, George E and Ornstein, Leonard S},
  journal={Physical Review},
  volume={36},
  number={5},
  pages={823},
  year={1930},
  publisher={APS}
}

@article{turin1972statistical,
  title={A statistical model of urban multipath propagation},
  author={Turin, George L and Clapp, Fred D and Johnston, Tom L and Fine, Stephen B and Lavry, Dan},
  journal={IEEE Transactions on Vehicular Technology},
  volume={21},
  number={1},
  pages={1--9},
  year={1972},
  publisher={IEEE}
}

@inproceedings{chang2025amortized,
  title={Amortized Probabilistic Conditioning for Optimization, Simulation and Inference},
  author={Chang, Paul E and Loka, Nasrulloh and Huang, Daolang and Remes, Ulpu and Kaski, Samuel and Acerbi, Luigi},
  booktitle={Proceedings of the International Conference on Artificial Intelligence and Statistics (AISTATS)},
  year={2025},
  organization={PMLR}
}

@article{whittle2025distribution,
  title={Distribution Transformers: Fast Approximate {B}ayesian Inference With On-The-Fly Prior Adaptation},
  author={Whittle, George and Ziomek, Juliusz and Rawling, Jacob and Osborne, Michael A},
  journal={arXiv preprint arXiv:2502.02463},
  year={2025}
}

@article{cranmer2020frontier,
  title={The frontier of simulation-based inference},
  author={Cranmer, Kyle and Brehmer, Johann and Louppe, Gilles},
  journal={Proceedings of the National Academy of Sciences},
  volume={117},
  number={48},
  pages={30055--30062},
  year={2020},
  publisher={National Academy of Sciences}
}

@inproceedings{lueckmann2021benchmarking,
  title={Benchmarking simulation-based inference},
  author={Lueckmann, Jan-Matthis and Boelts, Jan and Greenberg, David and Goncalves, Pedro and Macke, Jakob},
  booktitle={Proceedings of the International Conference on Artificial Intelligence and Statistics (AISTATS)},
  pages={343--351},
  year={2021},
  series = {Proceedings of Machine Learning Research},
  organization={PMLR}
}

@article{silvestrin2025stacking,
  title={Stacking Variational {B}ayesian {M}onte {C}arlo},
  author={Silvestrin, Francesco and Li, Chengkun and Acerbi, Luigi},
  journal={arXiv preprint arXiv:2504.05004},
  year={2025}
}

@article{huggins2023pyvbmc,
  title={{PyVBMC}: Efficient {B}ayesian inference in {P}ython},
  author={Huggins, Bobby and Li, Chengkun and Tobaben, Marlon and Aarnos, Mikko J and Acerbi, Luigi},
  journal={Journal of Open Source Software},
  volume={8},
  number={86},
  pages={5428},
  year={2023}
}

@inproceedings{song2023pseudoinverse,
  title={Pseudoinverse-guided diffusion models for inverse problems},
  author={Song, Jiaming and Vahdat, Arash and Mardani, Morteza and Kautz, Jan},
  booktitle={International Conference on Learning Representations (ICLR)},
  year={2023}
}

@inproceedings{song2023loss,
  title={Loss-guided diffusion models for plug-and-play controllable generation},
  author={Song, Jiaming and Zhang, Qinsheng and Yin, Hongxu and Mardani, Morteza and Liu, Ming-Yu and Kautz, Jan and Chen, Yongxin and Vahdat, Arash},
  booktitle={Proceedings of the International Conference on Machine Learning (ICML)},
  pages={32483--32498},
  year={2023},
  organization={PMLR}
}

@book{robert2007bayesian,
  title={The {B}ayesian Choice: From Decision-theoretic Foundations to Computational Implementation},
  author={Robert, Christian P},
  volume={2nd edition},
  year={2007},
  publisher={Springer}
}

@article{gelman2020bayesian,
  title={Bayesian workflow},
  author={Gelman, Andrew and Vehtari, Aki and Simpson, Daniel and Margossian, Charles C and Carpenter, Bob and Yao, Yuling and Kennedy, Lauren and Gabry, Jonah and B{\"u}rkner, Paul-Christian and Modr{\'a}k, Martin},
  journal={arXiv preprint arXiv:2011.01808},
  year={2020}
}

@book{gelman2013bayesian,
  title={Bayesian data analysis},
  author={Gelman, Andrew and Carlin, John B and Stern, Hal S and Vehtari, Aki and Rubin, Donald B},
  volume={3nd edition},
  year={2013},
  publisher={Chapman and Hall/CRC}
}

@article{radev2020bayesflow,
  title={BayesFlow: {L}earning complex stochastic models with invertible neural networks},
  author={Radev, Stefan T and Mertens, Ulf K and Voss, Andreas and Ardizzone, Lynton and K{\"o}the, Ullrich},
  journal={IEEE Transactions on Neural Networks and Learning Systems},
  volume={33},
  number={4},
  pages={1452--1466},
  year={2020},
  publisher={IEEE}
}

@inproceedings{wildberger2024flow,
  title={Flow matching for scalable simulation-based inference},
  author={Wildberger, Jonas and Dax, Maximilian and Buchholz, Simon and Green, Stephen and Macke, Jakob H and Sch{\"o}lkopf, Bernhard},
  booktitle={Advances in Neural Information Processing Systems (NeurIPS)},
  volume={36},
  year={2024},
  publisher={Curran Associates, Inc.}
}

@inproceedings{sohl2015deep,
  title={Deep unsupervised learning using nonequilibrium thermodynamics},
  author={Sohl-Dickstein, Jascha and Weiss, Eric and Maheswaranathan, Niru and Ganguli, Surya},
  booktitle={Proceedings of the International Conference on Machine Learning (ICML)},
  pages={2256--2265},
  year={2015},
  organization={PMLR}
}

@inproceedings{schmitt2023detecting,
  title={Detecting model misspecification in amortized {B}ayesian inference with neural networks},
  author={Schmitt, Marvin and B{\"u}rkner, Paul-Christian and K{\"o}the, Ullrich and Radev, Stefan T},
  booktitle={DAGM German Conference on Pattern Recognition},
  pages={541--557},
  year={2023},
  organization={Springer}
}

@inproceedings{lee2018simple,
  title={A simple unified framework for detecting out-of-distribution samples and adversarial attacks},
  author={Lee, Kimin and Lee, Kibok and Lee, Honglak and Shin, Jinwoo},
  booktitle={Advances in Neural Information Processing Systems (NeurIPS)},
  volume={31},
  year={2018},
  publisher={Curran Associates, Inc.}
}

@inproceedings{nalisnick2019deep,
  title={Do deep generative models know what they don't know?},
  author={Nalisnick, Eric and Matsukawa, Akihiro and Teh, Yee Whye and Gorur, Dilan and Lakshminarayanan, Balaji},
  booktitle={International Conference on Learning Representations (ICLR)},
  year={2019}
}

@inproceedings{huang2024learning,
  title={Learning Robust Statistics for Simulation-based Inference under Model Misspecification},
  author={Huang, Daolang and Bharti, Ayush and Souza, Amauri and Acerbi, Luigi and Kaski, Samuel},
  booktitle={Advances in Neural Information Processing Systems (NeurIPS)},
  volume={36},
  year={2024},
  publisher={Curran Associates, Inc.}
}

@inproceedings{lueckmann2017flexible,
  title={Flexible statistical inference for mechanistic models of neural dynamics},
  author={Lueckmann, Jan-Matthis and Goncalves, Pedro J and Bassetto, Giacomo and {\"O}cal, Kaan and Nonnenmacher, Marcel and Macke, Jakob H},
  booktitle={Advances in Neural Information Processing Systems (NeurIPS)},
  volume={30},
  year={2017},
  publisher={Curran Associates, Inc.}
}

@inproceedings{
    muller2022transformers,
    title={Transformers Can Do {B}ayesian Inference},
    author={Samuel M{\"u}ller and Noah Hollmann and Sebastian Pineda Arango and Josif Grabocka and Frank Hutter},
    booktitle={International Conference on Learning Representations (ICLR)},
    year={2022}
}

@inproceedings{vaswani2017attention,
  title={Attention is all you need},
  author={Vaswani, Ashish and Shazeer, Noam and Parmar, Niki and Uszkoreit, Jakob and Jones, Llion and Gomez, Aidan N and Kaiser, {\L}ukasz and Polosukhin, Illia},
  booktitle={Advances in Neural Information Processing Systems (NeurIPS)},
  volume={30},
  year={2017},
  publisher={Curran Associates, Inc.}
}

@inproceedings{nguyen2022transformer,
    author    = {Nguyen, Tung and Grover, Aditya},
    title     = {{T}ransformer {N}eural {P}rocesses:
{U}ncertainty-Aware Meta Learning Via Sequence Modeling},
    booktitle = {Proceedings of the International Conference on Machine Learning (ICML)},
    year      = {2022},
    pages     = {123--134},
    publisher = {PMLR},
}

@inproceedings{garnelo2018conditional,
  title={Conditional Neural Processes},
  author={Garnelo, Marta and Rosenbaum, Dan and Maddison, Chris J and Ramalho, Tiago and Saxton, David and Shanahan, Murray and Teh, Yee Whye and Rezende, Danilo J and Eslami, SM Ali},
  booktitle={Proceedings of the International Conference on Machine Learning (ICML)},
  pages={1704--1713},
  year={2018},
}

@inproceedings{bruinsma2023autoregressive,
  title={Autoregressive Conditional Neural Processes},
  author={Bruinsma, Wessel P and Markou, Stratis and Requeima, James and Foong, Andrew YK and Andersson, Tom R and Vaughan, Anna and Buonomo, Anthony and Hosking, J Scott and Turner, Richard E},
  booktitle={International Conference on Learning Representations (ICLR)},
  year={2023}
}

@inproceedings{kingma15,
  author    = {Diederik P Kingma and
               Jimmy Ba},
  title     = {Adam: {A} Method for Stochastic Optimization},
  booktitle = {International Conference on Learning Representations (ICLR)},
  year      = {2015},
}

@inproceedings{ho2020denoising,
  title={Denoising Diffusion Probabilistic Models},
  author={Ho, Jonathan and Jain, Ajay and Abbeel, Pieter},
  booktitle={Advances in Neural Information Processing Systems (NeurIPS)},
  volume={33},
  pages={6840--6851},
  year={2020},
  publisher={Curran Associates, Inc.}
}

@inproceedings{song2021score,
  title={Score-Based Generative Modeling through Stochastic Differential Equations},
  author={Yang Song and Jascha Sohl-Dickstein and Diederik P. Kingma and Abhishek Kumar and Stefano Ermon and Ben Poole},
  booktitle={International Conference on Learning Representations (ICLR)},
  year={2021},
  month={May},
  publisher={ICLR},
}

@inproceedings{lipman2023flow,
  title={Flow Matching for Generative Modeling},
  author={Lipman, Yaron and Chen, Ricky T. Q. and Ben-Hamu, Heli and Nickel, Maximilian and Le, Matthew},
  booktitle={International Conference on Learning Representations (ICLR)},
  year={2023}
}

@book{sarkka2019applied,
  title={Applied Stochastic Differential Equations},
  author={S{\"a}rkk{\"a}, Simo and Solin, Arno},
  year={2019},
  publisher={Cambridge University Press}
}

@inproceedings{rissanen2024hunch,
      title={Free Hunch: Denoiser Covariance Estimation for Diffusion Models Without Extra Costs}, 
      author={Severi Rissanen and Markus Heinonen and Arno Solin},
      year={2025},
      booktitle={International Conference on Learning Representations (ICLR)}
}

@article{SORENSON,
title = {Recursive {B}ayesian estimation using {G}aussian sums},
journal = {Automatica},
volume = {7},
number = {4},
pages = {465-479},
year = {1971},
issn = {0005-1098},
author = {H.W. Sorenson and D.L. Alspach},
}

@inproceedings{dhariwal2021diffusion,
  title={Diffusion models beat {GAN}s on image synthesis},
  author={Dhariwal, Prafulla and Nichol, Alexander},
  booktitle={Advances in Neural Information Processing Systems (NeurIPS)},
  volume={34},
  pages={8780--8794},
  year={2021},
  publisher={Curran Associates, Inc.}
}

@article{ho2022classifier,
  title={Classifier-free diffusion guidance},
  author={Ho, Jonathan and Salimans, Tim},
  journal={NeurIPS 2021 Workshop on Deep Generative Models and Downstream Applications.},
  year={2022}
}

@inproceedings{chung2023diffusion,
  title={Diffusion Posterior Sampling for General Noisy Inverse Problems},
  author={Chung, Hyungjin and Kim, Jeongsol and Mccann, Michael T and Klasky, Marc L and Ye, Jong Chul},
  booktitle={International Conference on Learning Representations (ICLR)},
  year={2023}
}

@article{hyvarinen2005,
  title={Estimation of non-normalized statistical models by score matching.},
  author={Hyv{\"a}rinen, Aapo and Dayan, Peter},
  journal={Journal of Machine Learning Research},
  volume={6},
  number={4},
  year={2005}
}

@article{vincent2011connection,
    title={A connection between score matching and denoising autoencoders},
    author={Vincent, Pascal},
    journal={Neural Computation},
    volume={23},
    number={7},
    pages={1661--1674},
    year={2011},
    publisher={MIT Press}
}

@inproceedings{ho2022video,
    title={Video Diffusion Models},
    author={Ho, Jonathan and Salimans, Tim and Gritsenko, Alexey and Chan, William and Norouzi, Mohammad and Fleet, David J.},
    booktitle={Advances in Neural Information Processing Systems (NeurIPS)},
    volume={35},
    pages={18954--18967},
    year={2022},
    publisher={Curran Associates, Inc.}
}

@inproceedings{peng2024improving,
  title={Improving Diffusion Models for Inverse Problems Using Optimal Posterior Covariance},
booktitle={International Conference on Learning Representations (ICLR)},
  author={Peng, Xinyu and Zheng, Ziyang and Dai, Wenrui and Xiao, Nuoqian and Li, Chenglin and Zou, Junni and Xiong, Hongkai},
  year={2024}
}

@inproceedings{finzi2023user,
  title={User-defined event sampling and uncertainty quantification in diffusion models for physical dynamical systems},
  author={Finzi, Marc Anton and Boral, Anudhyan and Wilson, Andrew Gordon and Sha, Fei and Zepeda-N{\'u}{\~n}ez, Leonardo},
  booktitle={Proceedings of the International Conference on Machine Learning (ICML)},
  pages={10136--10152},
  year={2023},
  organization={PMLR}
}

@inproceedings{bao2022analytic,
    title={Analytic-{DPM}: an Analytic Estimate of the Optimal Reverse Variance in Diffusion Probabilistic Models},
    author={Bao, Fan and Li, Chongxuan and Zhu, Jun and Zhang, Bo},
    booktitle={International Conference on Learning Representations (ICLR)},
    year={2022}
}

@article{kording2007causal,
  title={Causal inference in multisensory perception},
  author={K{\"o}rding, Konrad P and Beierholm, Ulrik and Ma, Wei Ji and Quartz, Steven and Tenenbaum, Joshua B and Shams, Ladan},
  journal={PLoS One},
  volume={2},
  number={9},
  pages={e943},
  year={2007},
  publisher={Public Library of Science San Francisco, USA}
}

@article{elsemuller2024sensitivity,
    title={Sensitivity-Aware Amortized {B}ayesian Inference},
    author={Elsem{\"u}ller, Lasse and Olischl{\"a}ger, Hans and Schmitt, Marvin and B{\"u}rkner, Paul-Christian and K{\"o}the, Ullrich and Radev, Stefan T},
    journal={Transactions on Machine Learning Research (TMLR)},
    year={2024}
}

@article{hollmann2025accurate,
  title={Accurate predictions on small data with a tabular foundation model},
  author={Hollmann, Noah and M{\"u}ller, Samuel and Purucker, Lennart and Krishnakumar, Arjun and K{\"o}rfer, Max and Hoo, Shi Bin and Schirrmeister, Robin Tibor and Hutter, Frank},
  journal={Nature},
  volume={637},
  number={8045},
  pages={319--326},
  year={2025},
  publisher={Nature Publishing Group UK London}
}

@inproceedings{karras2022elucidating,
    title={Elucidating the Design Space of Diffusion-Based Generative Models},
    author={Karras, Tero and Aittala, Miika and Aila, Timo and Laine, Samuli},
    booktitle={Advances in Neural Information Processing Systems (NeurIPS)},
    volume={35},
    pages={26565--26577},
    year={2022},
  publisher={Curran Associates, Inc.}
}

@inproceedings{loconte2024subtractive,
  title={Subtractive Mixture Models via Squaring: Representation and Learning},
  author={Loconte, Lorenzo and Sladek, Aleksanteri M. and Mengel, Stefan and Trapp, Martin and Solin, Arno and Gillis, Nicolas and Vergari, Antonio},
  booktitle={International Conference on Learning Representations (ICLR)},
  year={2024}
}

@inproceedings{
lopez-paz2017revisiting,
title={Revisiting Classifier Two-Sample Tests},
author={David Lopez-Paz and Maxime Oquab},
booktitle={International Conference on Learning Representations (ICLR)},
year={2017},
}

@article{pedersen2019stochastic,
  title={Stochastic multipath model for the in-room radio channel based on room electromagnetics},
  author={Pedersen, Troels},
  journal={IEEE Transactions on Antennas and Propagation},
  volume={67},
  number={4},
  pages={2591--2603},
  year={2019},
  publisher={IEEE}
}

@article{talts2018validating,
  title={Validating {B}ayesian inference algorithms with simulation-based calibration},
  author={Talts, Sean and Betancourt, Michael and Simpson, Daniel and Vehtari, Aki and Gelman, Andrew},
  journal={arXiv preprint arXiv:1804.06788},
  year={2018}
}

@article{yuyan2025robust,
  title={Robust {B}ayesian methods using amortized simulation-based inference},
  author={Yuyan, Wang and Evans, Michael and Nott, David J},
  journal={arXiv preprint arXiv:2504.09475},
  year={2025}
}

@article{papamakarios2016fast,
  title={Fast $\varepsilon$-free inference of simulation models with {B}ayesian conditional density estimation},
  author={Papamakarios, George and Murray, Iain},
  journal={Advances in Neural Information Processing Systems (NeurIPS)},
  volume={29},
  year={2016}
}

@inproceedings{lueckmann2019likelihood,
  title={Likelihood-free inference with emulator networks},
  author={Lueckmann, Jan-Matthis and Bassetto, Giacomo and Karaletsos, Theofanis and Macke, Jakob H},
  booktitle={Symposium on Advances in Approximate Bayesian Inference},
  pages={32--53},
  year={2019},
  organization={PMLR}
}

@inproceedings{hermans2020likelihood,
  title={Likelihood-free mcmc with amortized approximate ratio estimators},
  author={Hermans, Joeri and Begy, Volodimir and Louppe, Gilles},
  booktitle={International conference on machine learning},
  pages={4239--4248},
  year={2020},
  organization={PMLR}
}

@article{thomas2022likelihood,
  title={Likelihood-free inference by ratio estimation},
  author={Thomas, Owen and Dutta, Ritabrata and Corander, Jukka and Kaski, Samuel and Gutmann, Michael U},
  journal={Bayesian Analysis},
  volume={17},
  number={1},
  pages={1--31},
  year={2022},
  publisher={International Society for Bayesian Analysis}
}

@inproceedings{sharrock2024sequential,
  title={Sequential Neural Score Estimation: Likelihood-Free Inference with Conditional Score Based Diffusion Models},
  author={Sharrock, Louis and Simons, Jack and Liu, Song and Beaumont, Mark},
  booktitle={International Conference on Machine Learning},
  pages={44565--44602},
  year={2024},
  organization={PMLR}
}

@article{wilcoxon1945individual,
  title={Individual Comparisons by Ranking Methods},
  author={Wilcoxon, Frank},
  journal={Biometrics Bulletin},
  volume={1},
  number={6},
  pages={80--83},
  year={1945},
  publisher={JSTOR}
}

@article{acerbi2018variational,
  title={Variational {B}ayesian {M}onte {C}arlo},
  author={Acerbi, Luigi},
  journal={Advances in Neural Information Processing Systems (NeurIPS)},
  volume={31},
  year={2018}
}

@article{acerbi2018bayesian,
  title={Bayesian comparison of explicit and implicit causal inference strategies in multisensory heading perception},
  author={Acerbi, Luigi and Dokka, Kalpana and Angelaki, Dora E and Ma, Wei Ji},
  journal={PLoS computational biology},
  volume={14},
  number={7},
  pages={e1006110},
  year={2018},
  publisher={Public Library of Science San Francisco, CA USA}
}

@article{corner,
      doi = {10.21105/joss.00024},
      url = {https://doi.org/10.21105/joss.00024},
      year  = {2016},
      month = {jun},
      publisher = {The Open Journal},
      volume = {1},
      number = {2},
      pages = {24},
      author = {Daniel Foreman-Mackey},
      title = {corner.py: Scatterplot matrices in Python},
      journal = {The Journal of Open Source Software}
}

@book{robert2004monte,
  title={Monte Carlo Statistical Methods},
  author={Robert, Christian P. and Casella, George},
  year={2004},
  edition={2nd},
  publisher={Springer},
  address={New York},
  series={Springer Texts in Statistics},
  isbn={0-387-21239-6}
}

@article{sherwood2020assessment,
  title={An assessment of Earth's climate sensitivity using multiple lines of evidence},
  author={Sherwood, Steven C and Webb, Mark J and Annan, James D and Armour, Kyle C and Forster, Piers M and Hargreaves, Julia C and Hegerl, Gabriele and Klein, Stephen A and Marvel, Kate D and Rohling, Eelco J and others},
  journal={Reviews of Geophysics},
  volume={58},
  number={4},
  pages={e2019RG000678},
  year={2020},
  publisher={Wiley Online Library}
}

@article{del2008forming,
  title={Forming priors for {DSGE} models (and how it affects the assessment of nominal rigidities)},
  author={Del Negro, Marco and Schorfheide, Frank},
  journal={Journal of Monetary Economics},
  volume={55},
  number={7},
  pages={1191--1208},
  year={2008},
  publisher={Elsevier}
}

@article{flaxman2020estimating,
  title={Estimating the effects of non-pharmaceutical interventions on COVID-19 in Europe},
  author={Flaxman, Seth and Mishra, Swapnil and Gandy, Axel and Unwin, H Juliette T and Mellan, Thomas A and Coupland, Helen and Whittaker, Charles and Zhu, Harrison and Berah, Tresnia and Eaton, Jeffrey W and others},
  journal={Nature},
  volume={584},
  number={7820},
  pages={257--261},
  year={2020},
  publisher={Nature Publishing Group UK London}
}
